\documentclass[12pt]{article}
\usepackage[margin=1in]{geometry}

\usepackage{times,upgreek,morenotations,rotating}

\usepackage[utf8]{inputenc} % allow utf-8 input
\usepackage[T1]{fontenc}    % use 8-bit T1 fonts
\usepackage{url}            % simple URL typesetting
\usepackage{booktabs}       % professional-quality tables
\usepackage{amsfonts}       % blackboard math symbols
\usepackage{nicefrac}       % compact symbols for 1/2, etc.
\usepackage{microtype}      % microtypography

\usepackage{morenotations}

\usepackage{times}
\usepackage{graphicx} % more modern
\usepackage{cleveref}
\usepackage{natbib}
\usepackage{bm} % bold maths
\usepackage{fancyhdr}
\usepackage{graphicx}
\usepackage{color}
\usepackage{natbib}
\usepackage{booktabs}

\usepackage{graphicx}
\usepackage{caption}
\usepackage{subcaption}
\usepackage{tcolorbox}% http://ctan.org/pkg/tcolorbox

\newcommand{\thescalei}{0.06}
\newcommand{\imageinc}[1]{\hspace{-0.25cm} \includegraphics[trim=0bp 0bp 0bp
0bp,clip,width=\thescalei\textwidth]{#1} \hspace{-0.25cm}}

\newcommand{\thescalee}{0.18}
\newcommand{\imageincc}[2]{\hspace{-0.25cm} \includegraphics[trim=0bp 0bp 0bp
0bp,clip,height=\thescalee\textwidth,page=#2]{#1} \hspace{-0.25cm}}

\newcommand{\realinc}[1]{\hspace{-0.25cm} #1 \hspace{-0.25cm}}
\newcommand{\toyfn}{Figs/toy_pdf/epsilon_50-crop}

\usepackage{tikz}
\usetikzlibrary{arrows,decorations.pathmorphing,backgrounds,positioning,fit,petri}

\title{\papertitle\footnote{Appears in ICML 2019. Authors in alphabetical order. Work
    started while AKM was with the ANU and ZS was visiting Data61.}}

\author{
Zac Cranko$^{\spadesuit,\dagger}$ $\:\:\:$ Aditya Krishna Menon$^\heartsuit$ $\:\:\:$ Richard Nock$^{\dagger,\spadesuit,\clubsuit}$
\\
Cheng Soon Ong$^{\dagger,\spadesuit}$ $\:\:\:$ Zhan Shi$^{\diamondsuit}$ $\:\:\:$ Christian Walder$^{\dagger,\spadesuit}$\\\\
{\small $^\dagger$Data61, $^\spadesuit$the Australian National
  University, $^\heartsuit$Google Research
 }\\
{\small $^\clubsuit$the University of Sydney, $^\diamondsuit$University of Illinois at Chicago}\\
  {\small \texttt{firstname.lastname@$\{$data61.csiro.au,anu.edu.au$\}$; zshi22@uic.edu}}
}
\date{}

\begin{document}
\thispagestyle{empty}
\maketitle



\begin{abstract}

% !TEX root=../nips18-adversarial-mf-1.tex

The last few years have seen a staggering number of 
empirical studies of the
robustness of neural networks in a model of adversarial
perturbations of their inputs. Most
rely on an adversary which carries out local
modifications within prescribed balls. None however has so far
questioned the broader picture: how to frame a \textit{resource-bounded} adversary so
that it can be \textit{severely detrimental} to learning,
a non-trivial problem which entails at a minimum the choice of loss and classifiers.

We suggest a formal answer for losses that satisfy the minimal statistical
requirement of being \textit{proper}. We pin down a simple sufficient property for any given class of adversaries to be detrimental to learning,
involving a central measure of ``harmfulness''
which generalizes
the well-known class of integral probability metrics.
A key feature of our result is that it holds for \textit{all} proper losses,
and for a popular subset of these, the optimisation of this central measure appears to be
\textit{independent of the loss}. When classifiers
are Lipschitz -- a now popular approach in adversarial training --, this
optimisation resorts to \textit{optimal transport} to make a
low-budget compression of class marginals. Toy experiments reveal a
finding recently separately observed: training against a sufficiently
budgeted adversary of this kind \textit{improves} generalization.
\end{abstract}

% !TEX root=../nips18-adversarial-mf-1.tex

\section{Introduction}\label{sec-int}

%% INSIST: VERTICAL PARTITION
%%

Starting from the observation that deep nets are sensitive
to imperceptible perturbations of
% AKM: edit
%the input observations \citep{szsbegfIP},
their inputs~\citep{szsbegfIP},
a surge
of recent work
% AKM: edit
%has been focusing
has focussed
on new \textit{adversarial training} approaches to supervised learning
\citep{acwOG,aeikSR,bilvncMN,brrgTE,bprAE,cdlsCA,dalbkkaSA,fffAV,gmfsrwgAS,grasvUR,grcvCA,iealAA,kgbAM,mlwewsshbCA,mmstvTD,skcDG,sknekPL,tkpgbmEA,uokvAR,wjcAT,wzPD}
(and references within). In the now popular model of 
\citet{mmstvTD}, we want to learn a
classifier from a set $\mathcal{H}$, given a distribution of clean examples $D$ and loss
$\ell$. Adversarial training then seeks to find
\begin{eqnarray}
\arg \min_{h\in \mathcal{H}} \E_{(\X, \Y)\sim D}
  \left[\max_{\ve{\delta}: \|\ve{\delta}\|\leq \delta^*} \ell(\Y, h(\X +
  \ve{\delta}))\right], \label{eqORI}
\end{eqnarray}
where $\|.\|$ is a norm and $\delta^*$ is the budget of the
adversary. It has recently been observed that adversarial training
damages standard accuracy as data size and adversary's budget
($\delta^*$) increases \cite{tsetmRM}. A Bayesian explanation is given for a particular
$\{D,\|.\|,\ell\}$ in \citet{tsetmRM}, and the authors conclude their
findings questioning the interplay between adversarial robustness and
standard accuracy. 

In this paper, we dig into this relationship (i) by casting the standard
accuracy and loss in \eqref{eqORI} in the broad context of Bayesian decision theory \cite{gdGT}
and (ii) by considering a general form of adversaries, not restricted to the ones
used in \eqref{eqORI}. In particular, we assume that the
loss is \textit{proper}, which
is just a general form of statistical unbiasedness that many
popular choices meet \citep{hbPS,rwCB}. The minimization of a
proper loss
  gives guarantees on the accuracy (for example,
  \citet{kmOT}), so it directly connects to the setting of \citet{tsetmRM}. Regarding the adversaries, 
instead of relying on the
local adversarial modification $\X \rightarrow \X +
  \ve{\delta}$ for some $\|\ve{\delta}\|\leq \delta^*$, we consider a
  set of possible local modifications $\X \rightarrow a(\X)$ for some
  $a \in \mathcal{A} \subseteq \mathcal{X}^{\mathcal{X}}$
  ($\mathcal{A}$ fixed). We then analyze the conditions on $\mathcal{A}$ under which, for some $\epsilon > 0$,
\begin{eqnarray}
\boxed{\min_{h\in \mathcal{H}} \E_{(\X, \Y)\sim D}  \left[\max_{a\in
    \mathcal{A}} \ell(\Y, h\circ a(\X))\right] \geq (1-\epsilon)\properloss_0 \:\:,} \label{eqNOVO}
\end{eqnarray}
where $\properloss_0$ is the loss of the "blunt" predictor which
predicts nothing. If $h$ has range $\mathbb{R}$, this blunt predictor
is in general 0 (for the log loss, square loss, etc), which translates
into a class probability estimate of $1/2$ for all observations and
global accuracy of $50\%$ for two classes, \textit{i.e.} that of an
unbiased coin. We see the connection of \eqref{eqNOVO} to the
accuracy: as $\epsilon \searrow$, the learner will be tricked into
converging to an extremely poorly accurate predictor. How one can
design such provably efficient adversaries, furthermore under tight budget
constraints, is the starting point of our paper.

\textbf{Our first contribution} (Section \ref{sec-imp}) analyzes budgeted adversaries that can
enforce \eqref{eqNOVO}. Our main finding shows that
\eqref{eqNOVO} is implied by a very simple condition involving a central quantity $\upgamma$ generalizing the
celebrated integral probability metrics \citep{sfgslOI}. Furthermore, under some
additional condition on the loss, satisfied by the log, square and
Matsushita losses, the adversarial optimization of
$\upgamma$ \textit{does not
depend on the loss}. In other words, 
\begin{tcolorbox}[colframe=blue,boxrule=0.5pt,arc=4pt,left=6pt,right=6pt,top=6pt,bottom=6pt,boxsep=0pt]
    \begin{center}
      the adversary can attack the learner disregarding its loss.
    \end{center}
  \end{tcolorbox}

\begin{table}[t]
\begin{center}
{\small
\begin{tabular}{ccc}\\ \hline \hline
\realinc{clean class marginals} & \realinc{adversarial class marginals} &
                                                                     \realinc{OT
                                                                     plan} \\ 
\imageincc{{\toyfn}}{1} &
\imageincc{{\toyfn}}{3}  &
\imageincc{{\toyfn}}{8}  \\ \hline \hline
\end{tabular}

\begin{tabular}{cccccc||cccccc}\\ \hline\hline
&\multicolumn{5}{c||}{transformations from \texttt{digit-1}} &
                                               & \multicolumn{5}{c}{transformations from
                                               \texttt{digit-3}}\\
&\realinc{0} & \realinc{0.15} & \realinc{0.3} & \realinc{0.45} &
                                                                \realinc{0.6}
  & &  \realinc{0} & \realinc{0.15} & \realinc{0.3} & \realinc{0.45} &
                                                                \realinc{0.6}\\
\realinc{$\#1$} &\imageinc{{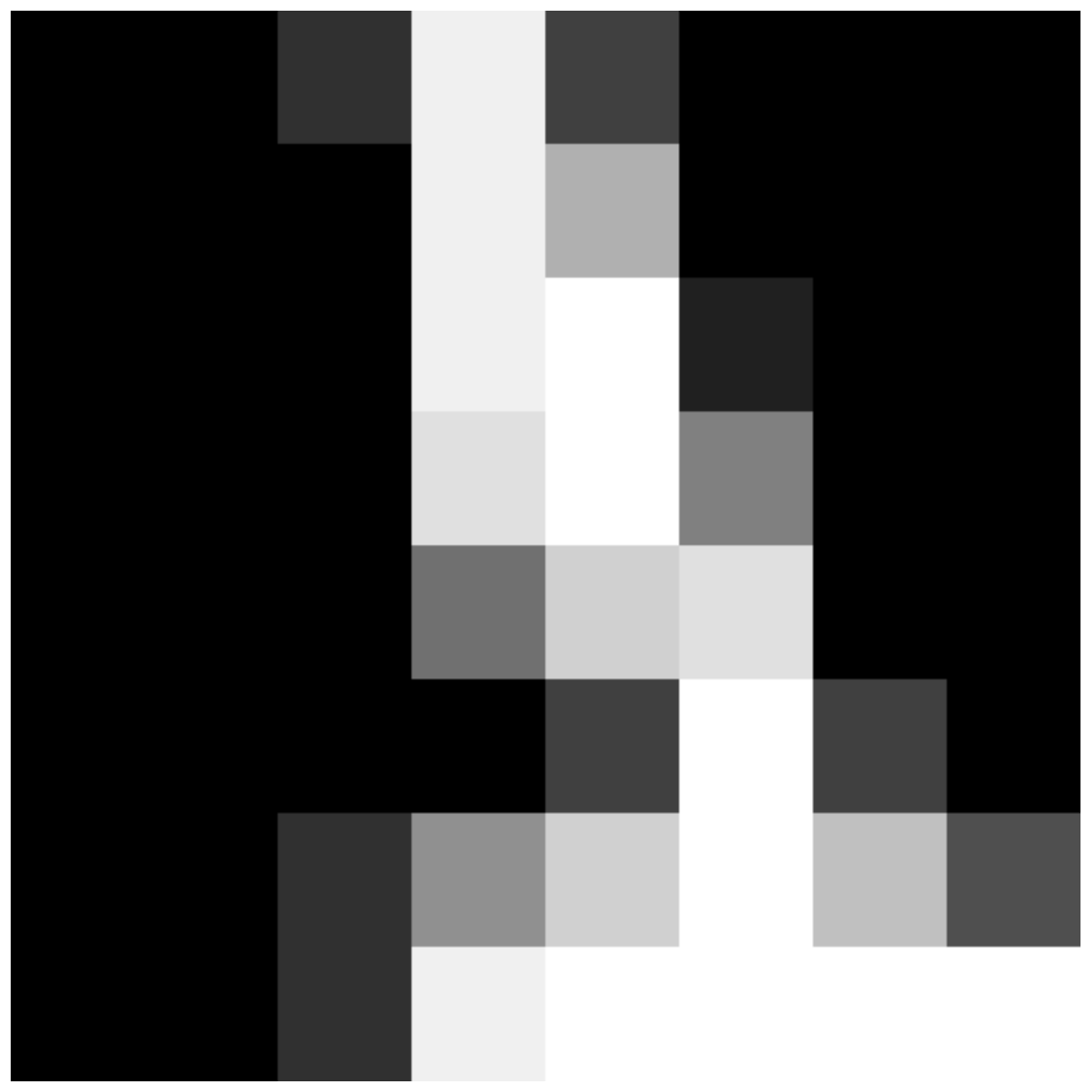}} &
\imageinc{{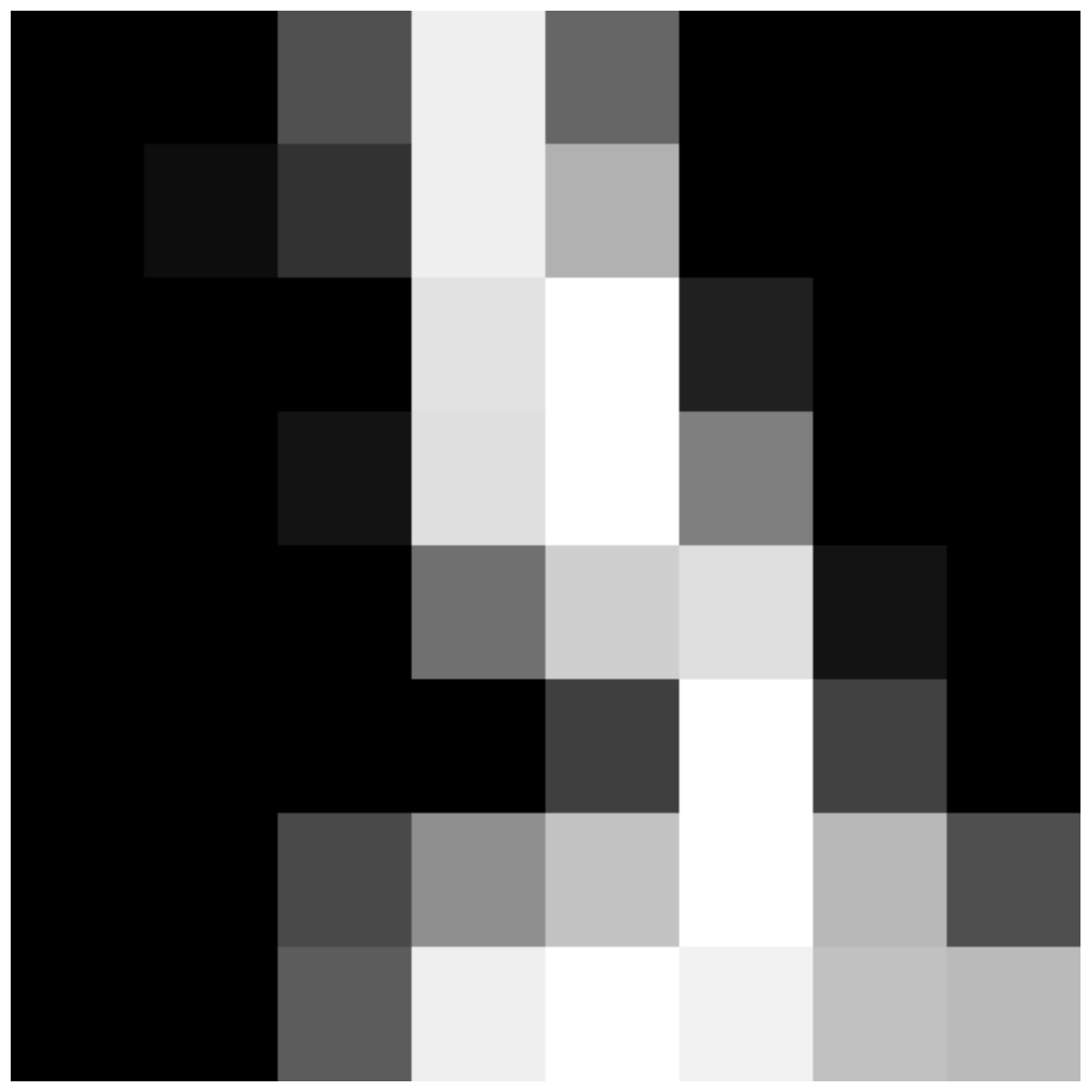}} &
\imageinc{{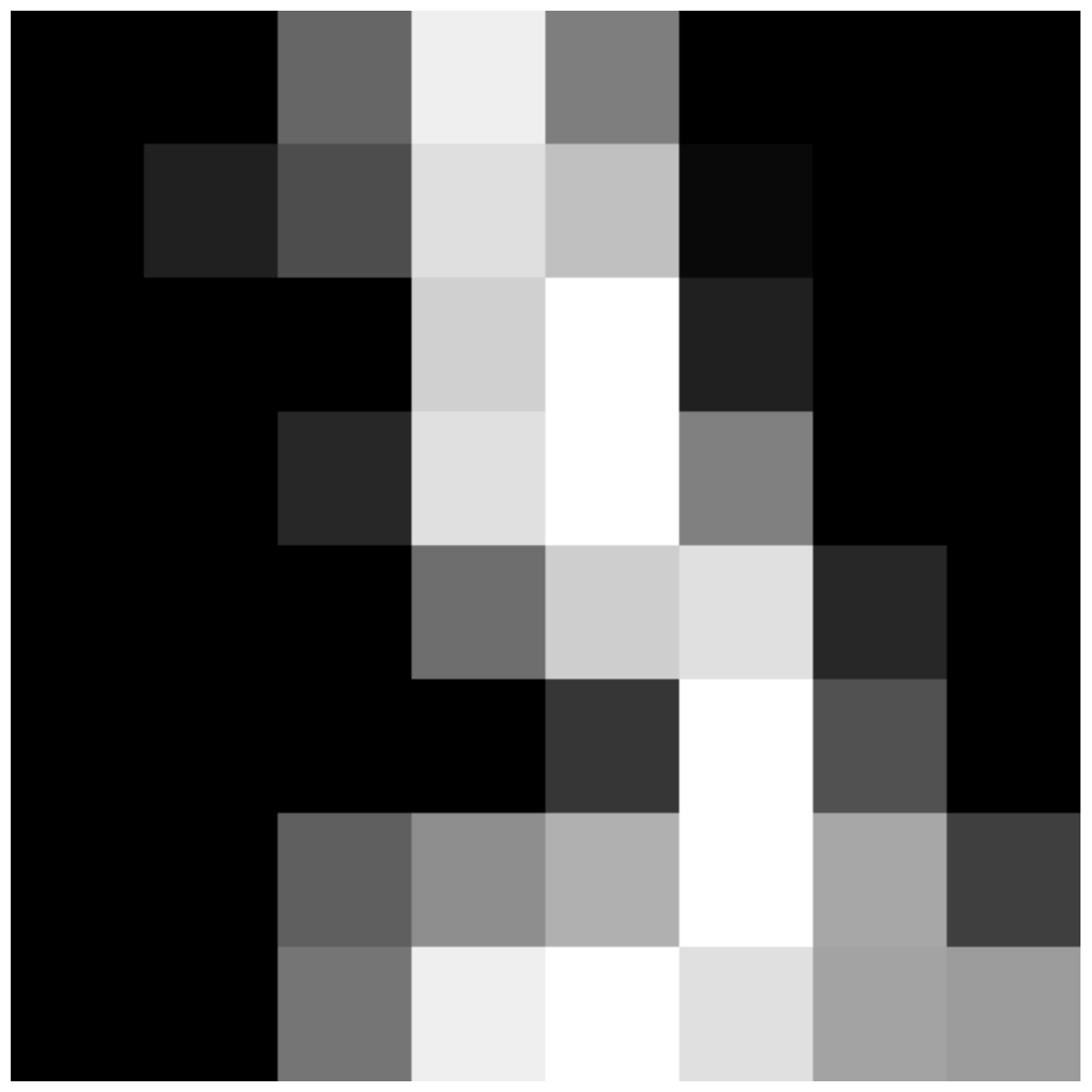}} &
\imageinc{{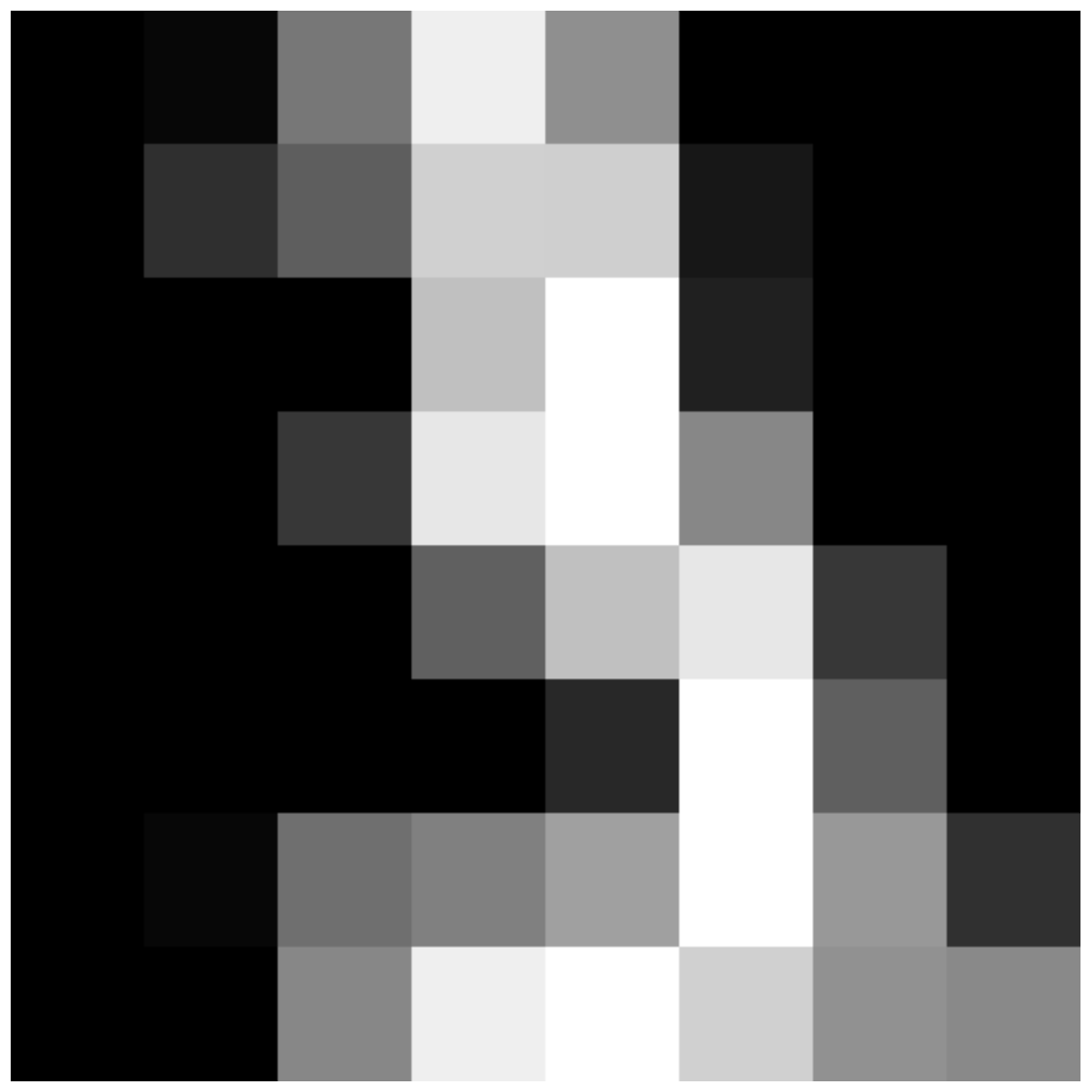}} &
\imageinc{{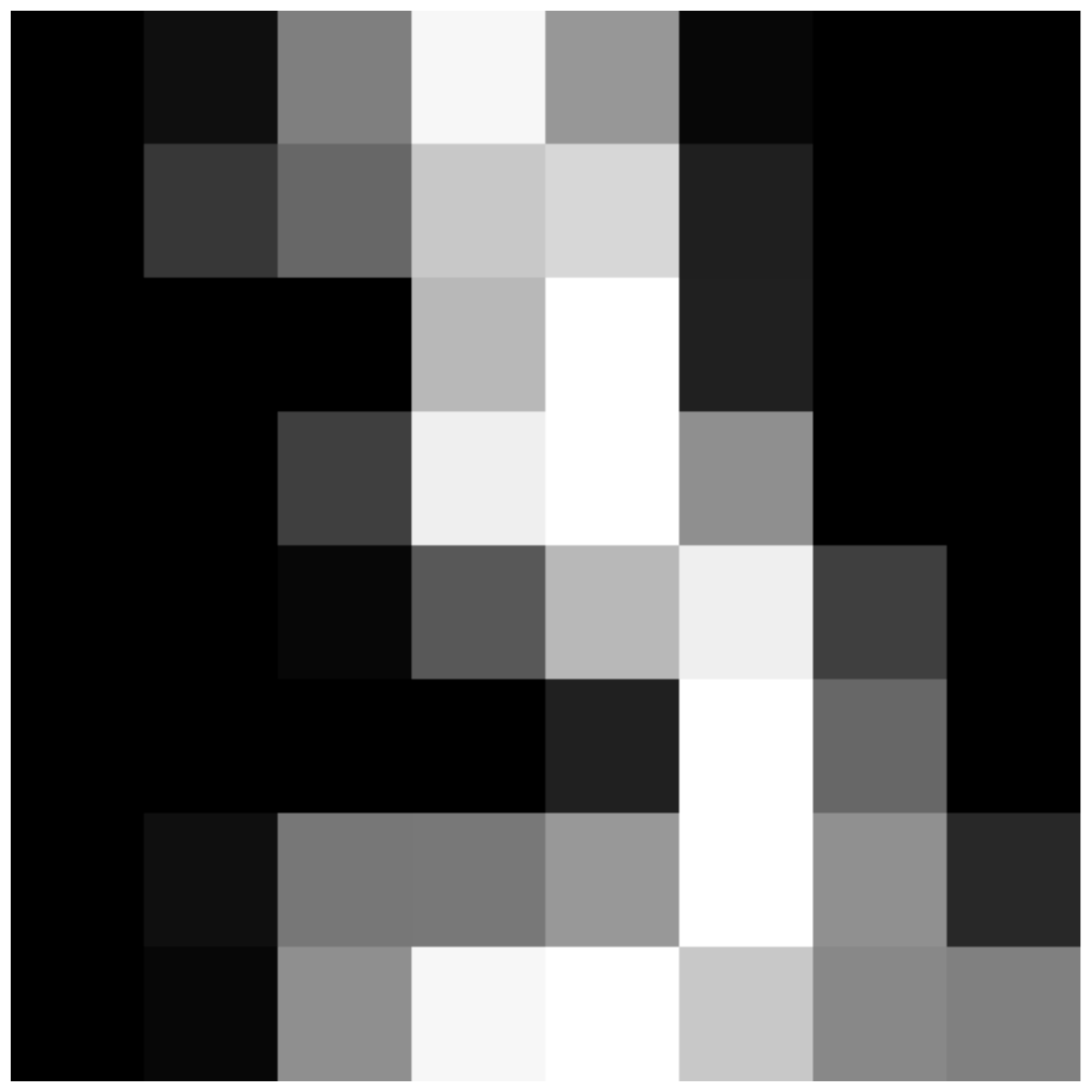}} &
\realinc{$\#3$} &\imageinc{{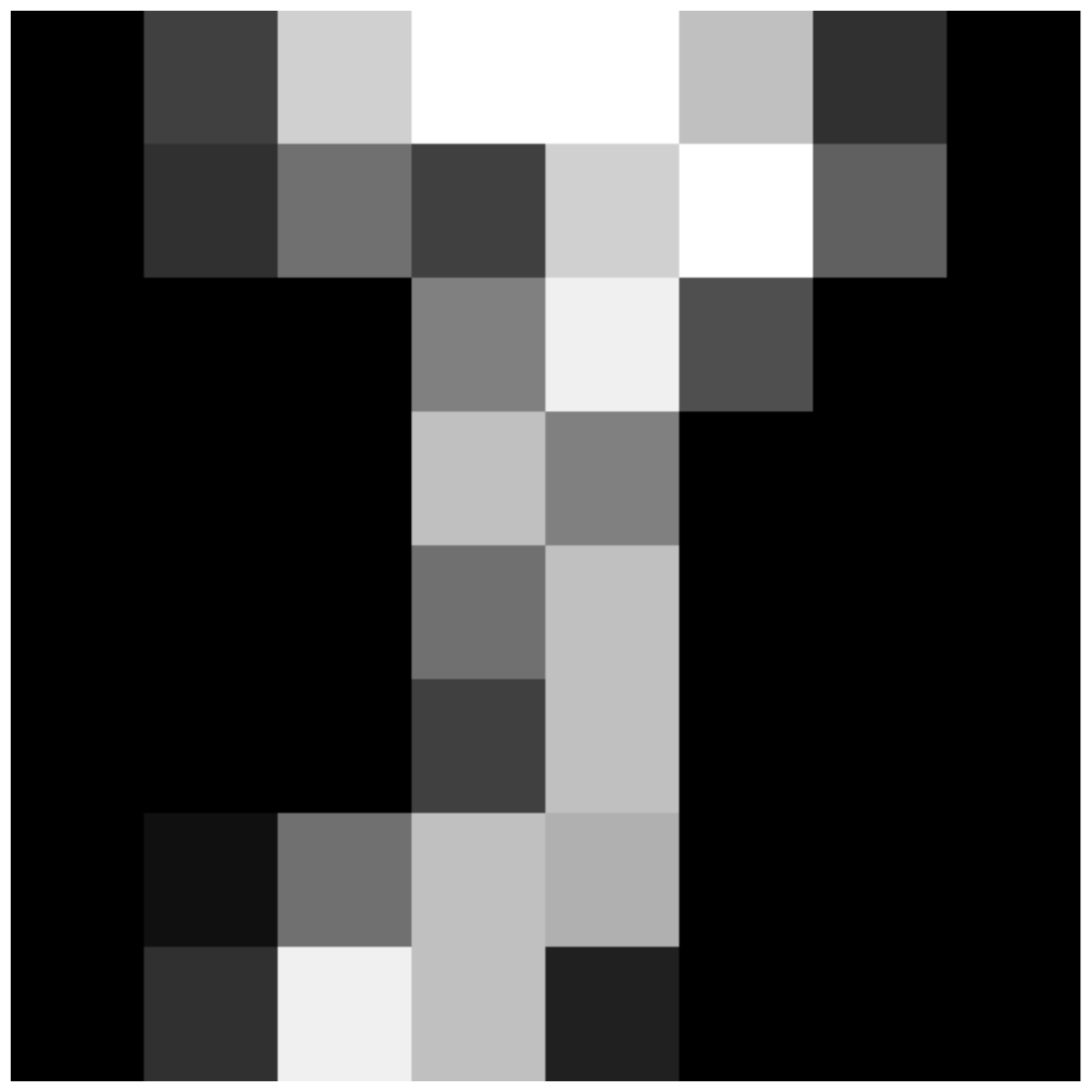}} &
\imageinc{{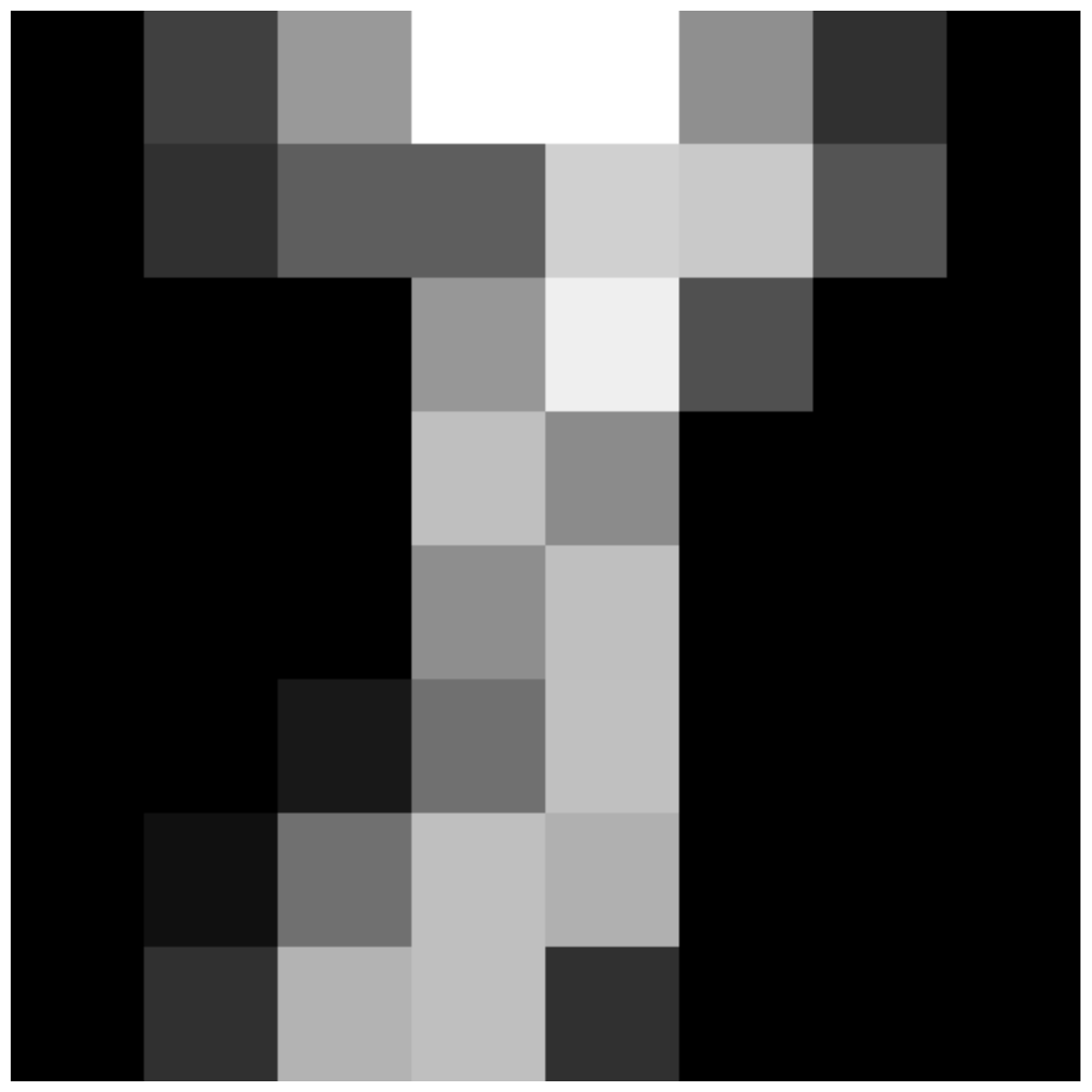}} &
\imageinc{{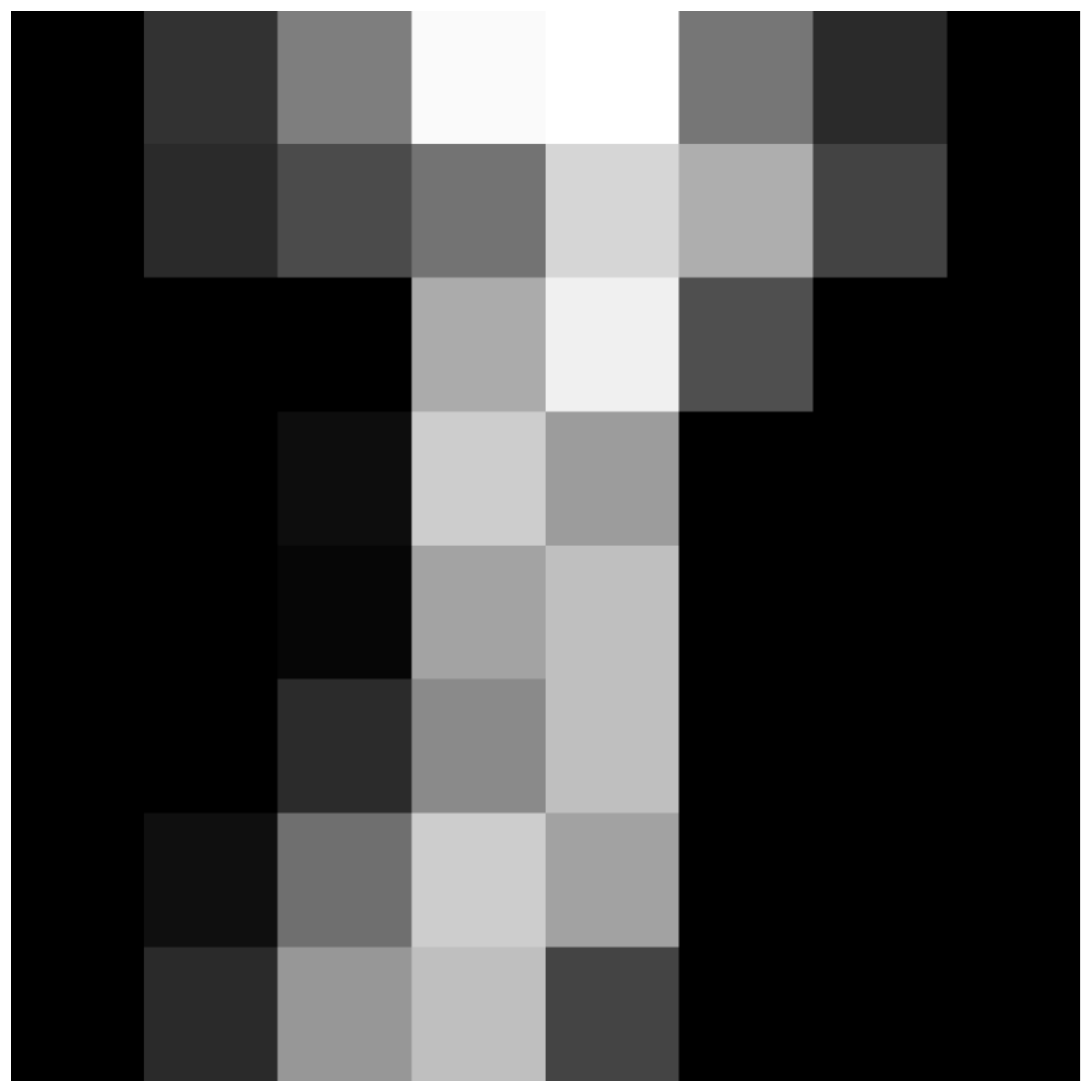}} &
\imageinc{{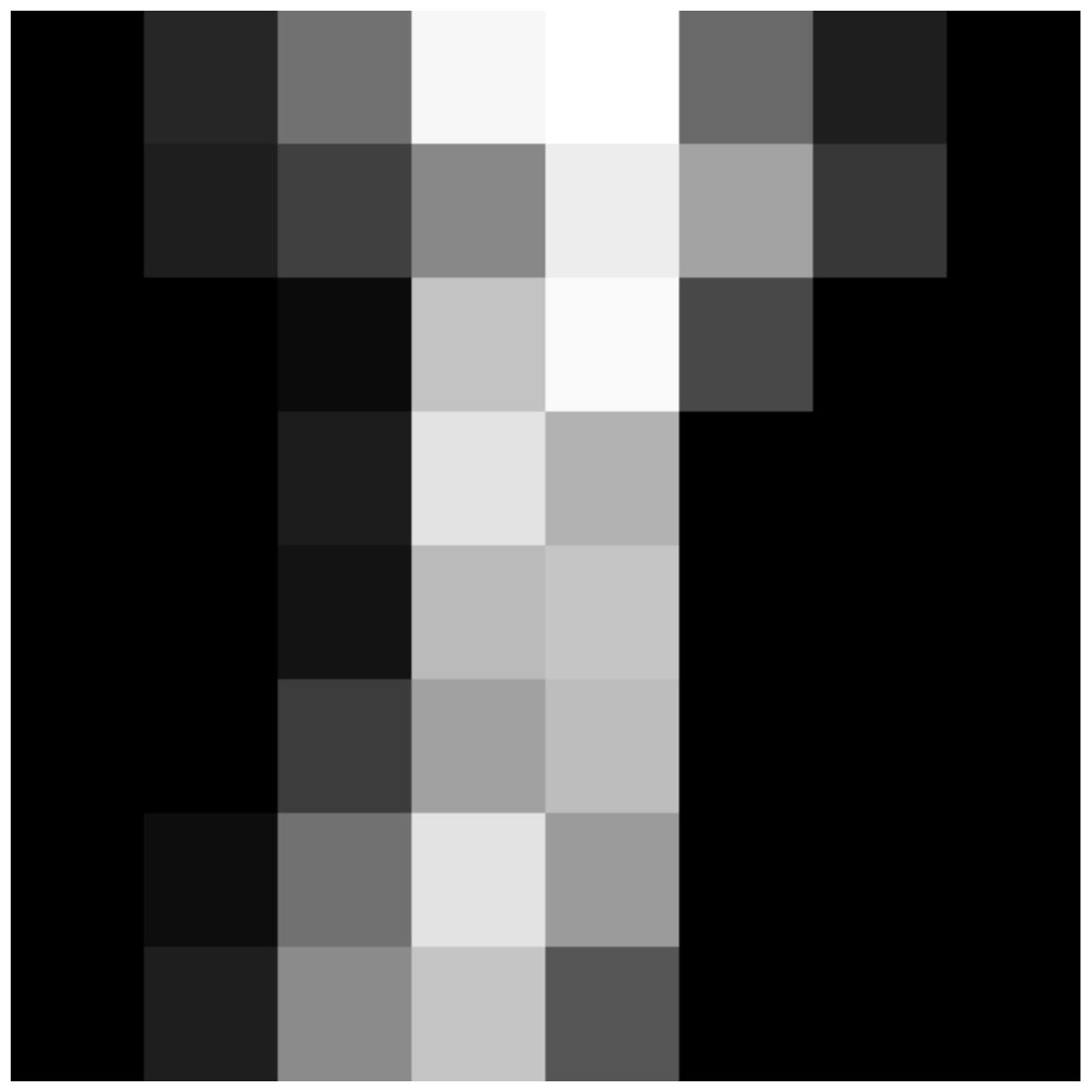}} &
\imageinc{{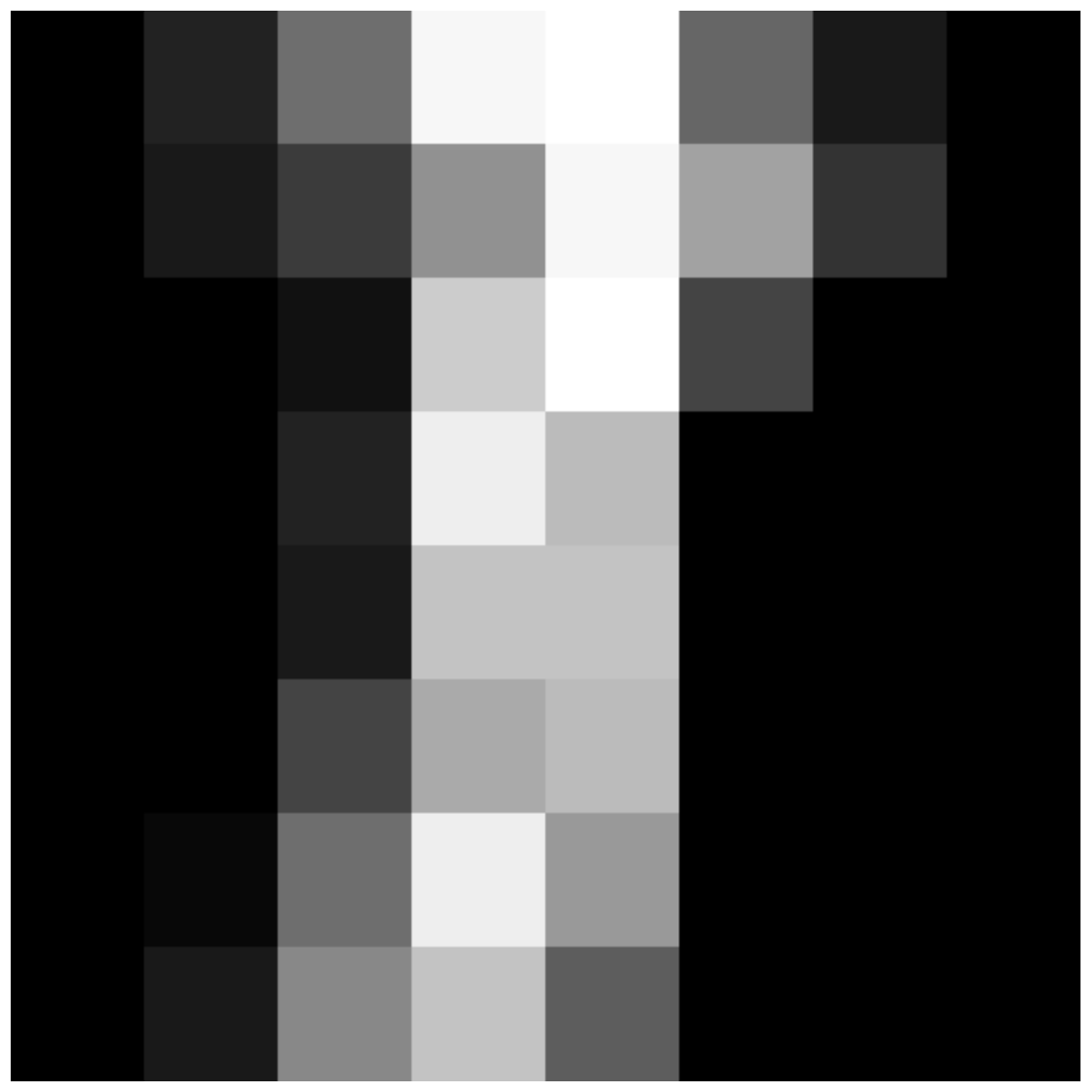}} \\
\realinc{$\#2$} & \imageinc{{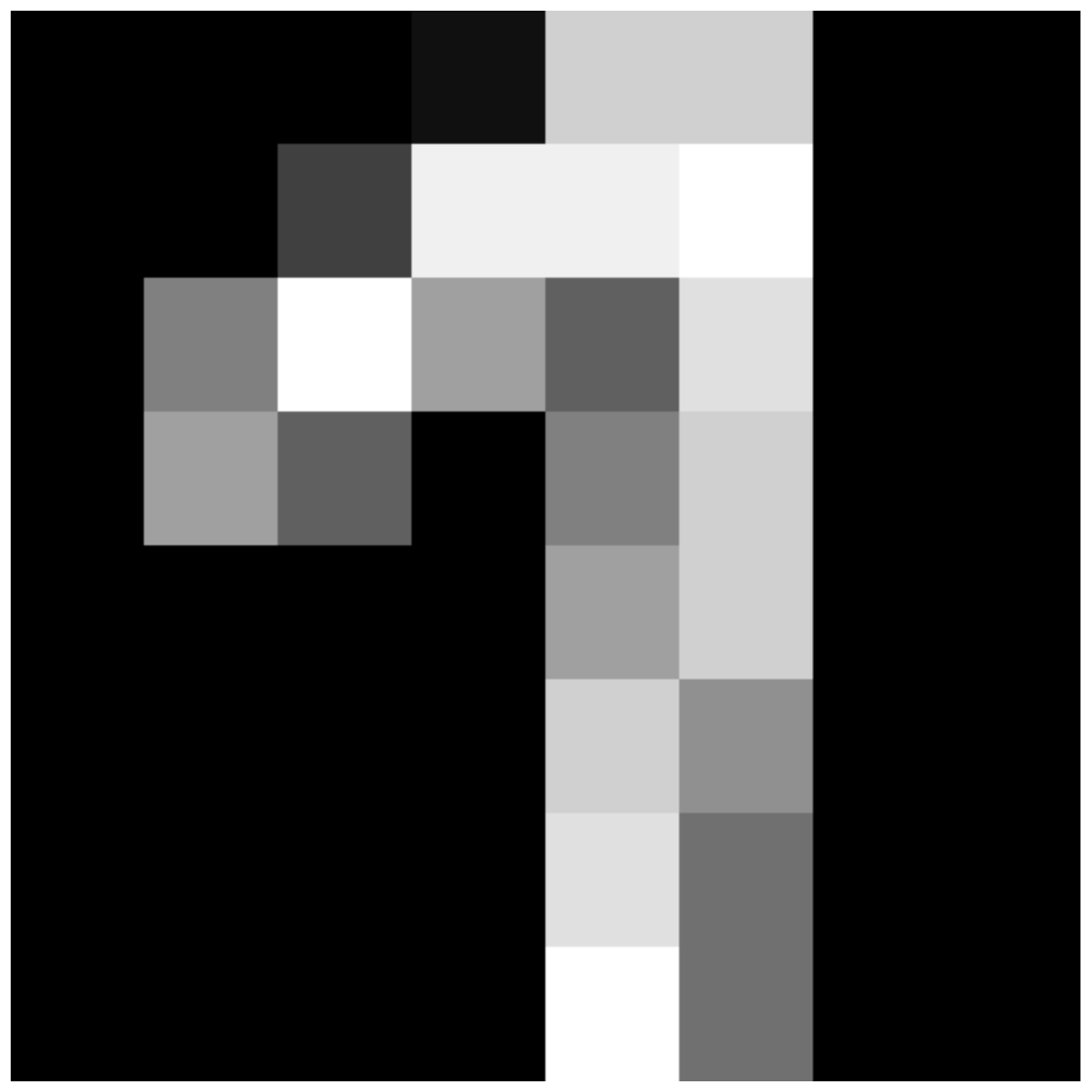}} &
\imageinc{{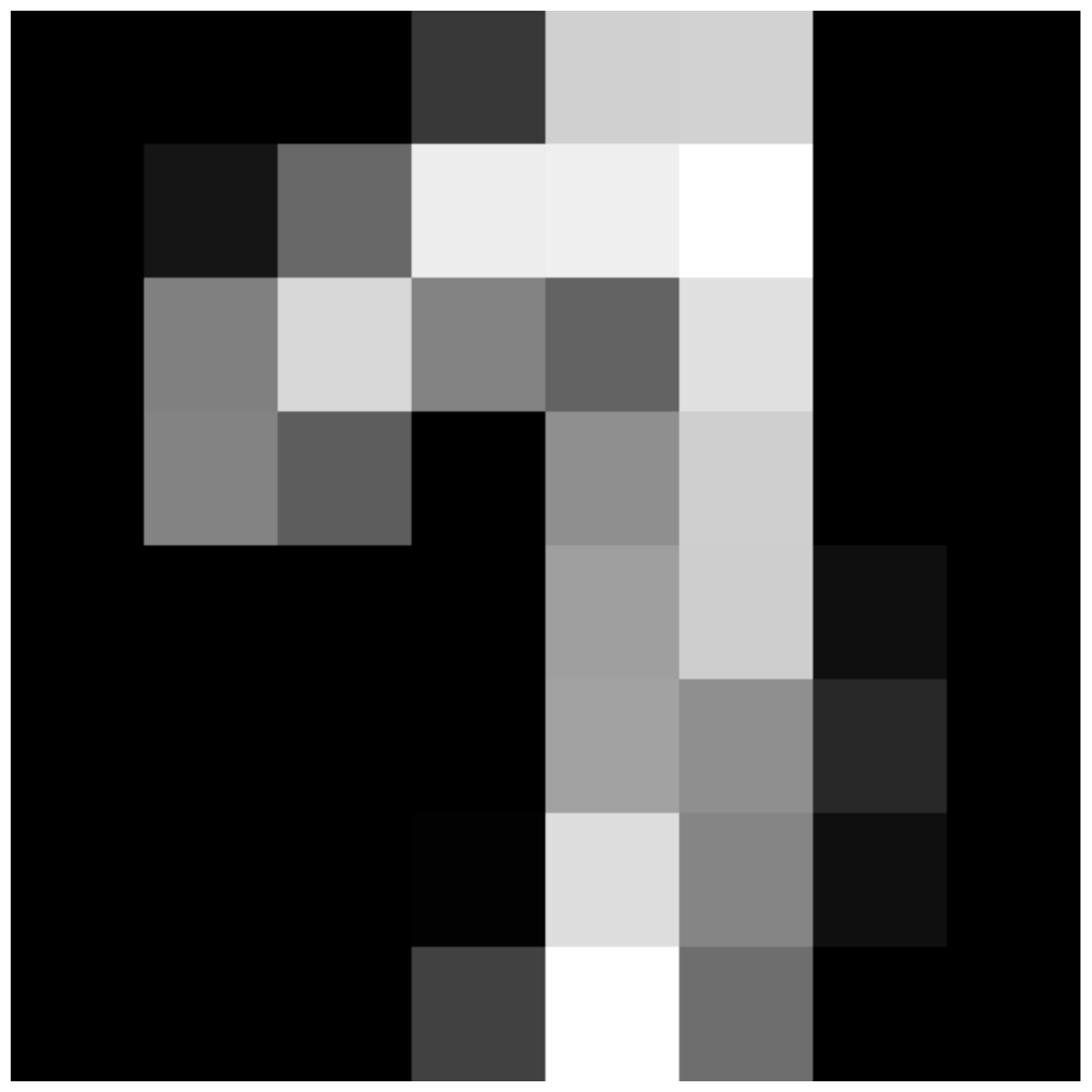}} &
\imageinc{{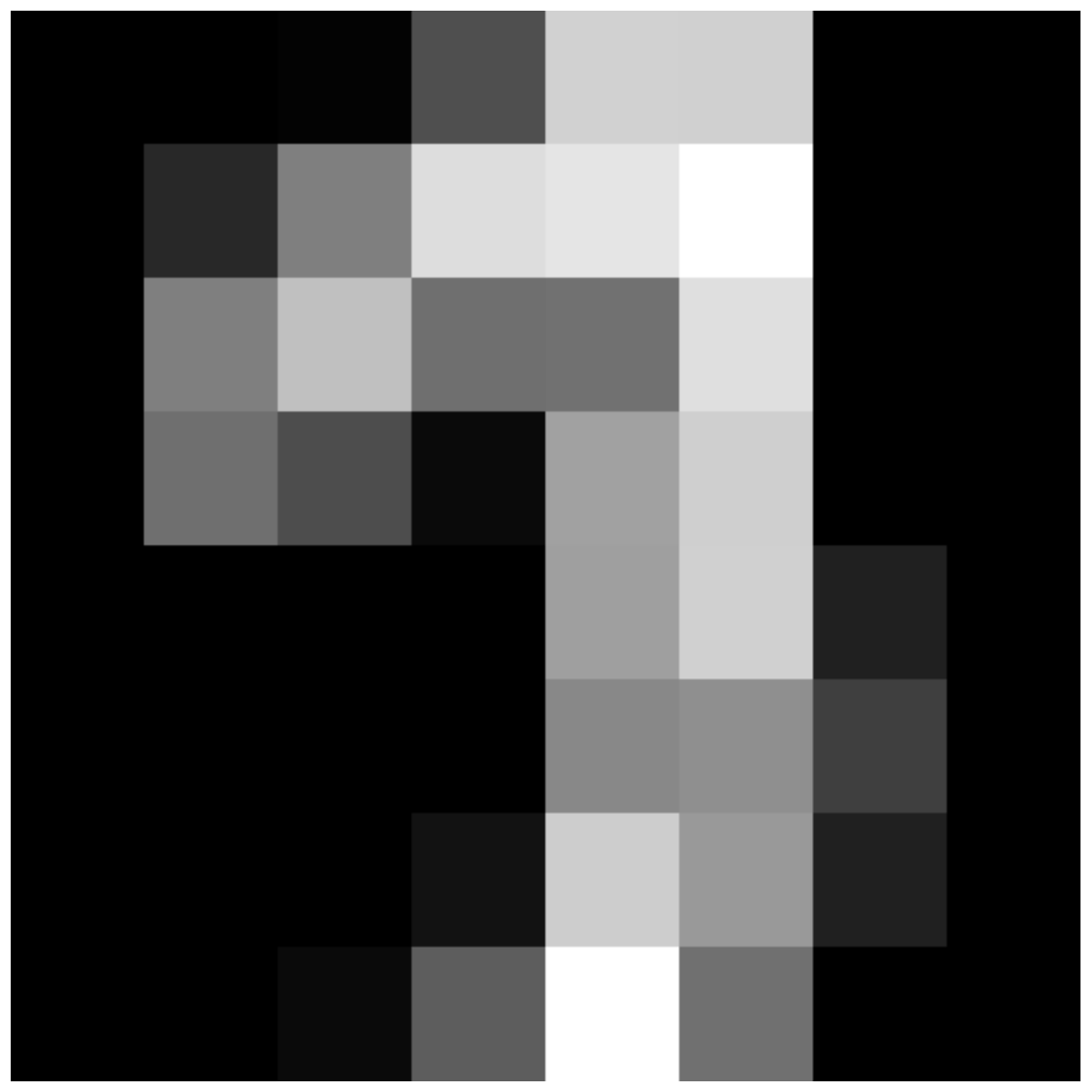}} &
\imageinc{{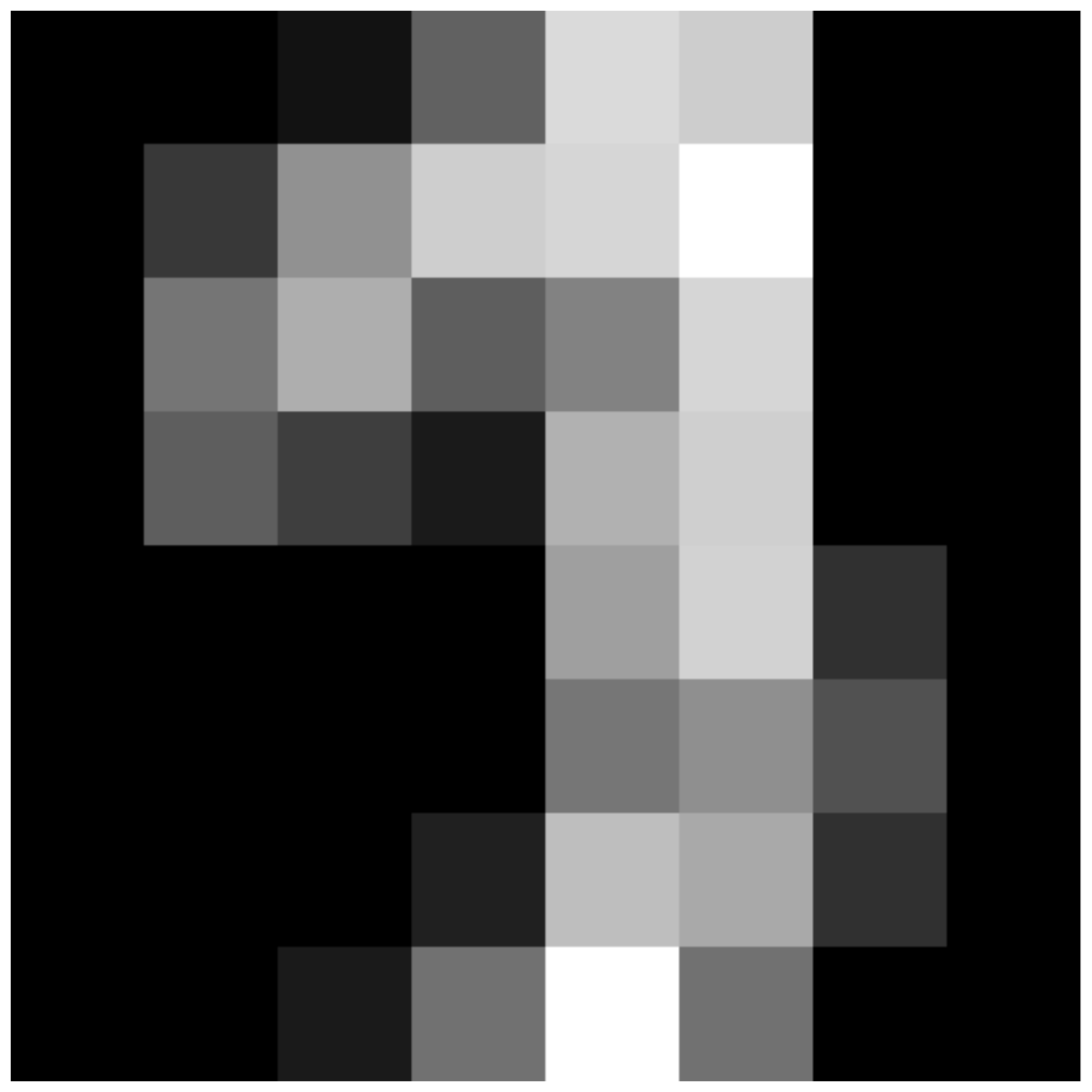}} &
\imageinc{{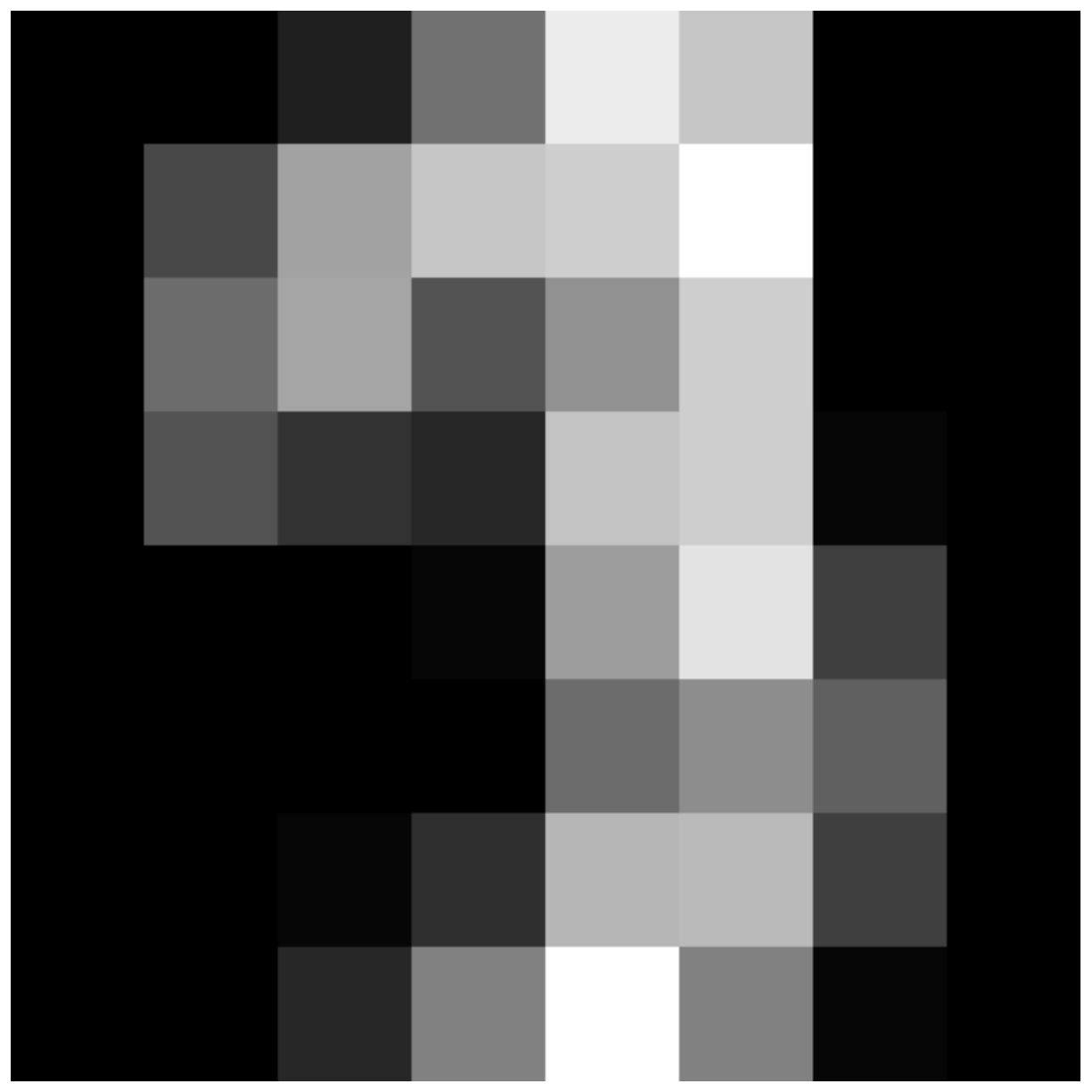}} &
\realinc{$\#4$} &\imageinc{{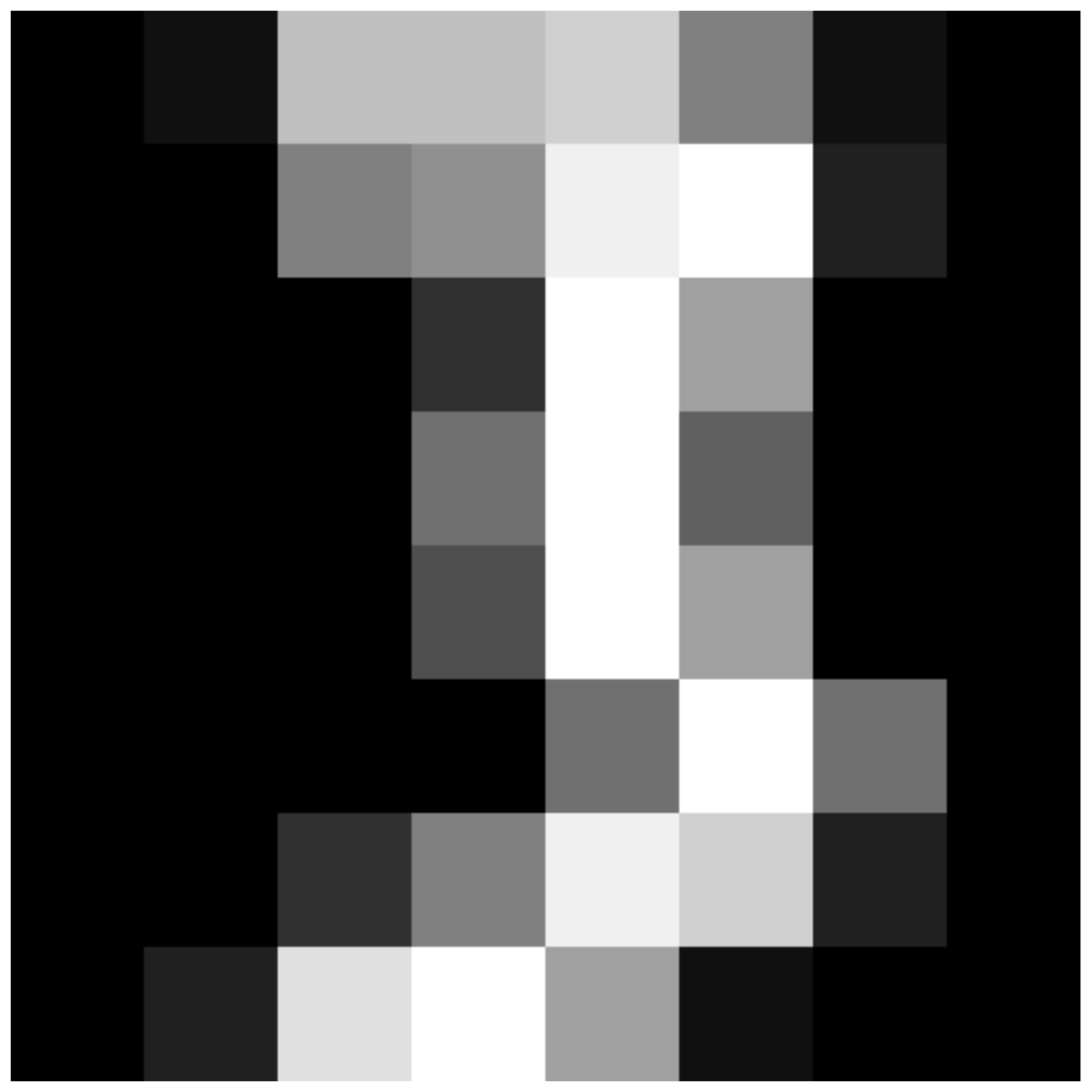}} &
\imageinc{{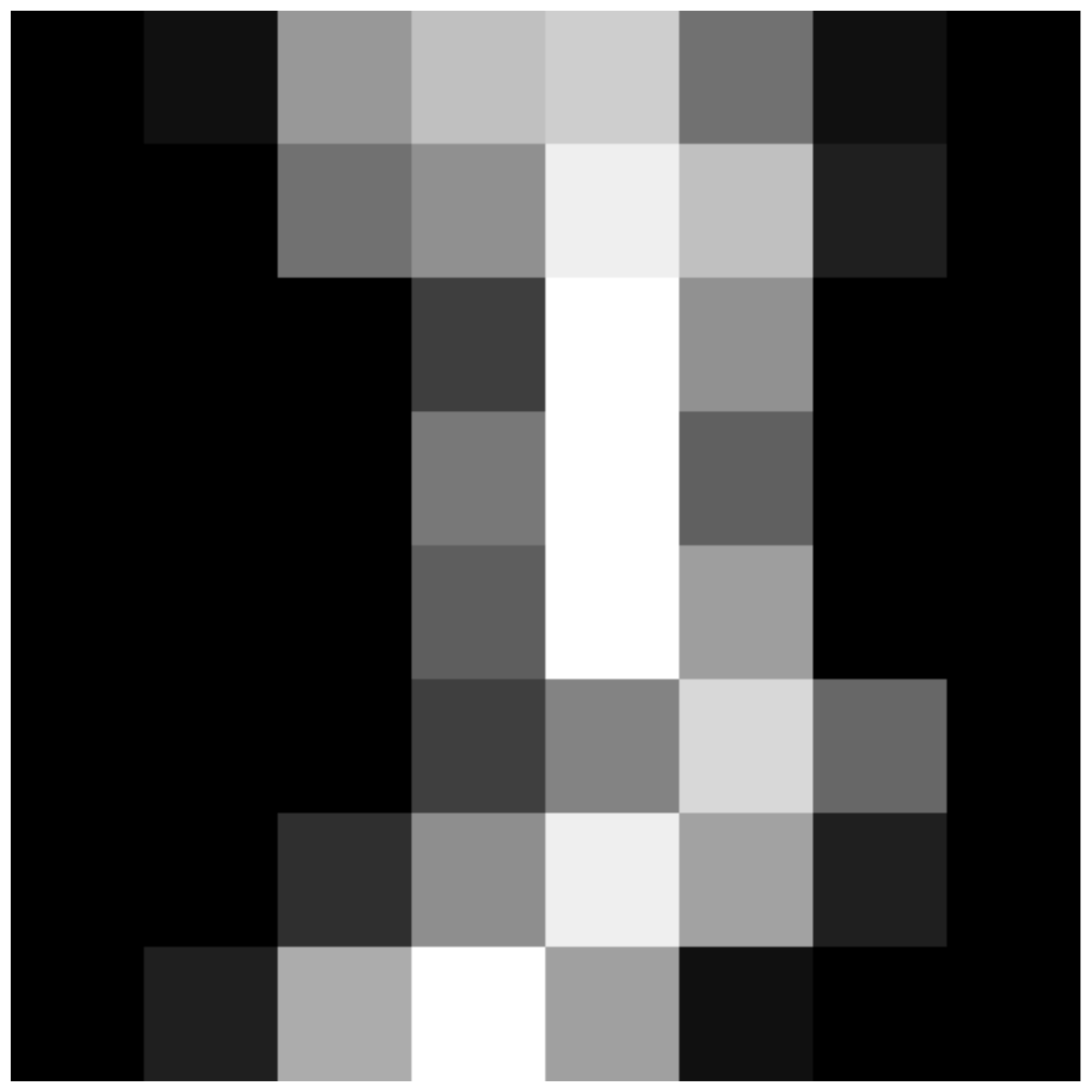}} &
\imageinc{{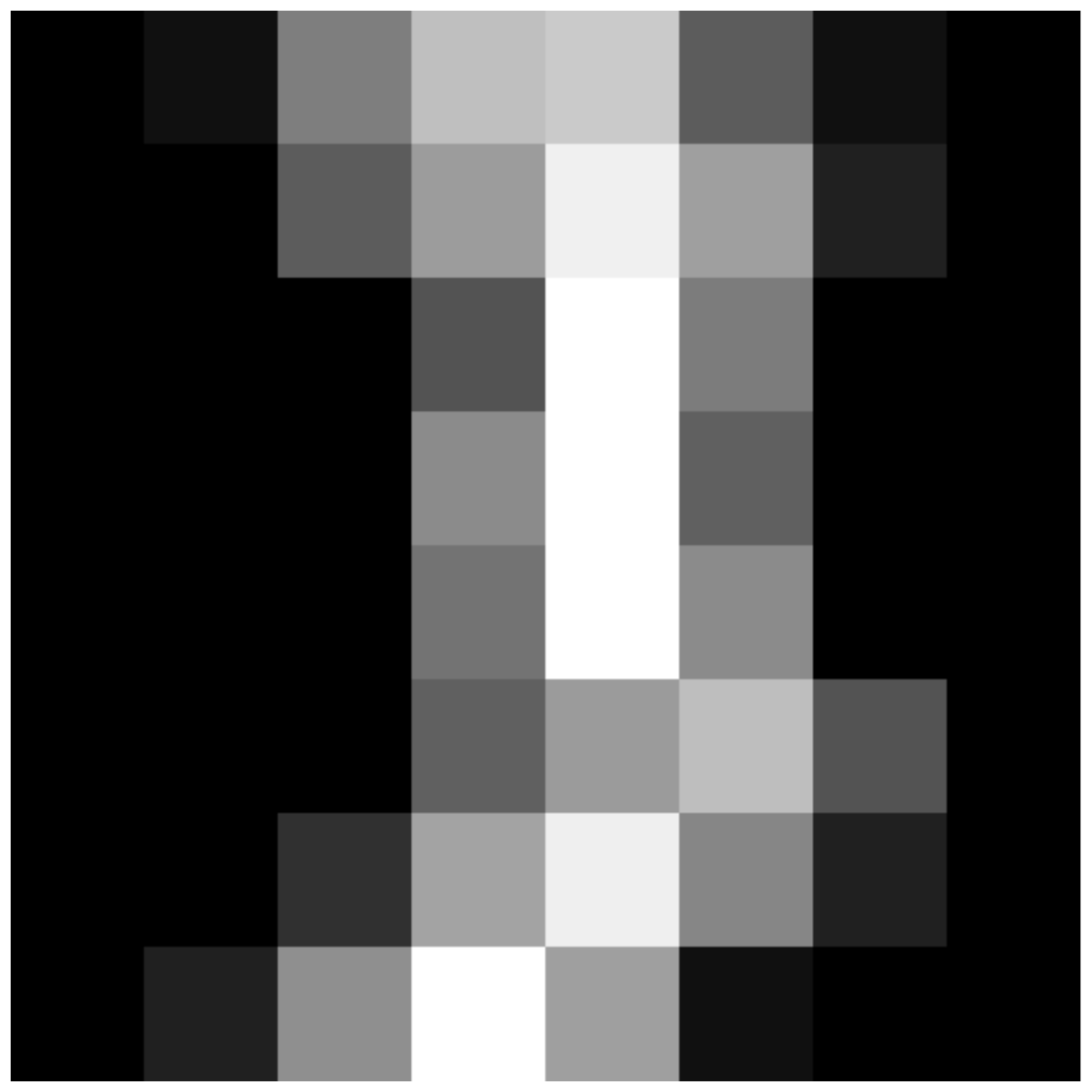}} &
\imageinc{{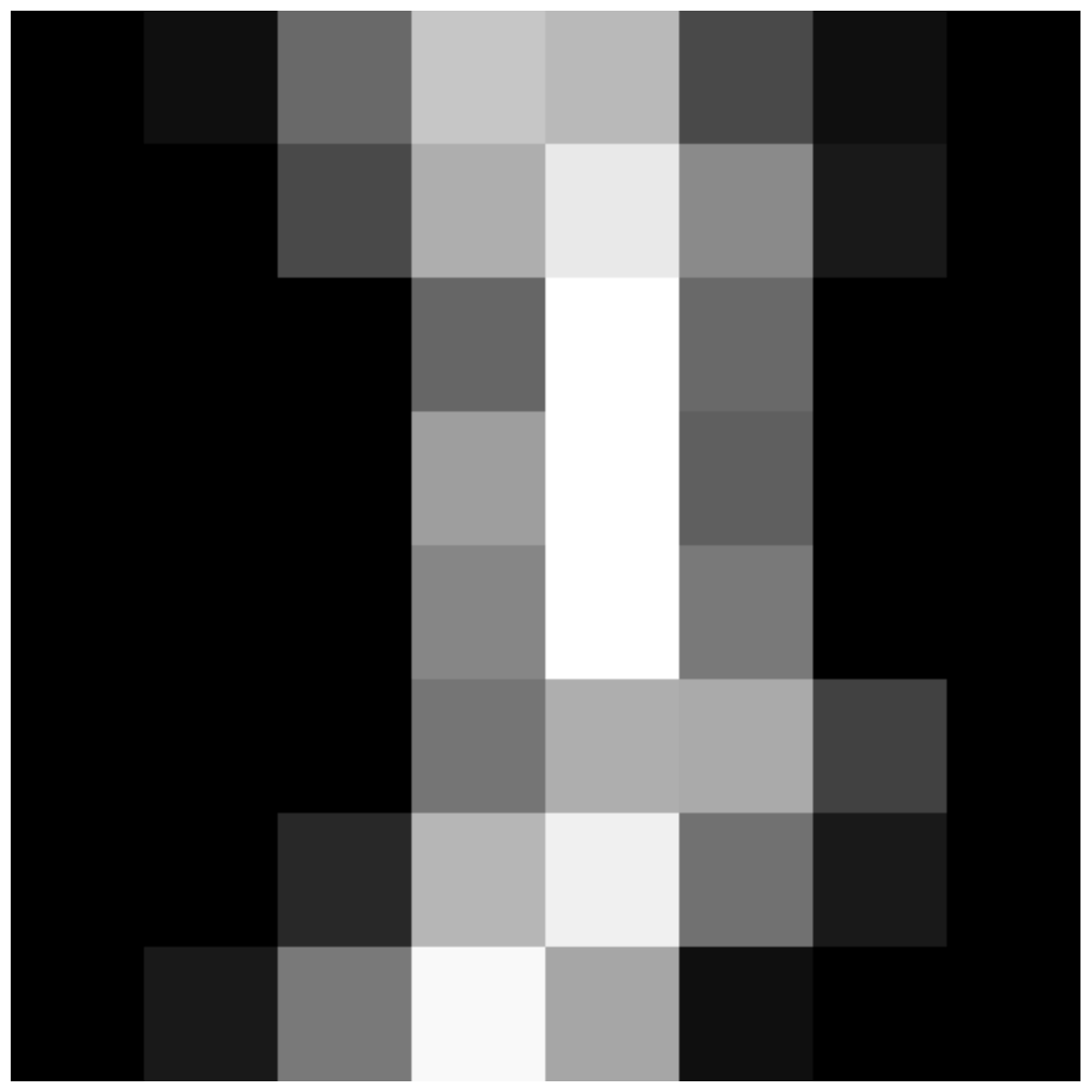}} &
\imageinc{{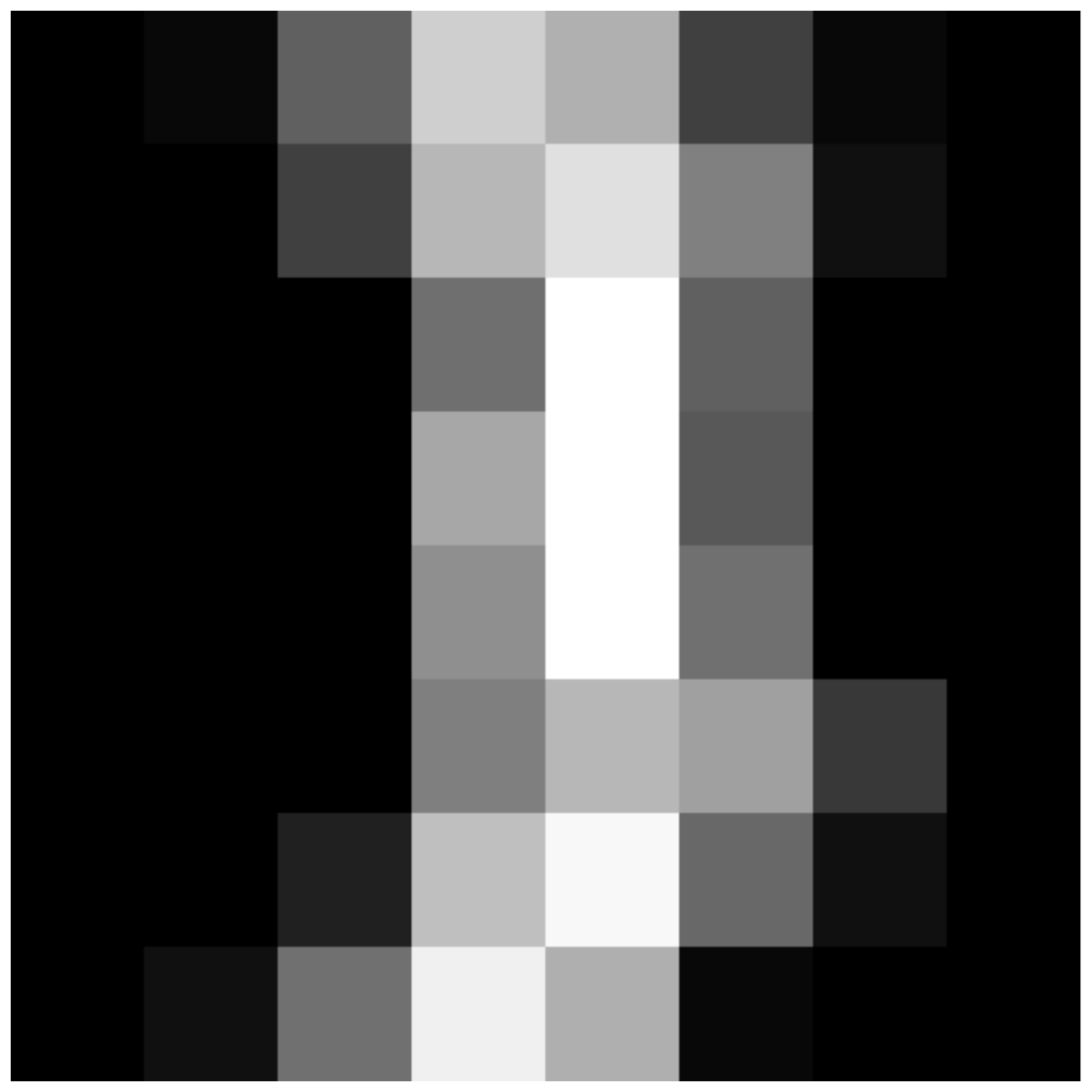}} \\ \hline\hline
\end{tabular}
}
\end{center}
\caption{\textit{Top table}: compression of the optimal transport (OT) plan for a
  Mixup adversary on a toy 1D data. \textit{Bottom table}:
  transformations performed by a Monge adversary for the \texttt{digit-1} vs \texttt{digit-3}
  classification problem on four USPS digits (noted $\#1$ to $\#4$), for various adversarial budgets (0
  = clean data, see
  Section \ref{sec:toy} for details).\label{tab-int-e}\vspace{-0.5cm}}
\end{table}

\textbf{Our second contribution} (Section \ref{sec-mong}) considers the adversarial
optimisation of $\upgamma$ when the classifiers in $\mathcal{H}$ satisfy a
generalized form of Lipschitz continuity. Controlling Lipschitz
continuity has recently emerged as a solution to limit the impact of
adversarial examples \citep{cbgduPN}. In this context, efficient
budgeted adversaries take a particular form: we show that, for an
adversary to minimize $\upgamma$,
\begin{tcolorbox}[colframe=blue,boxrule=0.5pt,arc=4pt,left=6pt,right=6pt,top=6pt,bottom=6pt,boxsep=0pt]
    \begin{center}
      it is
      sufficient to compress the optimal
      transport plan between class marginals using the Lipschitz
      function as transportation cost, disregarding the learner's $\mathcal{H}$.
    \end{center}
  \end{tcolorbox}
This result puts the machinery of optimal transport (OT) to the
table of adversarial design \cite{vOT}, with a new purpose (the compression of
OT plans). These two findings turn out to be very useful from an
experimental standpoint: we have implemented two kinds of
adversaries inspired by our theory (called Mixup and
Monge for their respective links with \citet{zcdlMB,vOT}); Table \ref{tab-int-e}
displays their behaviour on two simple problems. We have
observed that training a learner against a
”weak” (severely budgeted) adversary 
improves \textit{generalization} on clean data, a phenomenon
also observed elsewhere \cite{tsetmRM,zcdlMB}. 
The \texttt{digit} experiment displays how
our adversaries progressively
transform observations of one class into credible observations of the
other (See Experiments in Section \ref{sec:toy}, and the Supplement, \supplement). 

\textbf{Our third contribution} (Section \ref{sec-wts}) is an adversarial boosting result: it
answers the question as to whether one can efficiently craft an
arbitrarily \textit{strong} adversary from the sole access to a black box
\textit{weak} adversary. In the
framework of reproducing kernel Hilbert
spaces (RKHS), we show that
\begin{tcolorbox}[colframe=blue,boxrule=0.5pt,arc=4pt,left=6pt,right=6pt,top=6pt,bottom=6pt,boxsep=0pt]
    \begin{center}
      this "weak adversary" $\Rightarrow$ "strong adversary" design
      does exist, and our proof is constructive: we build one.
    \end{center}
  \end{tcolorbox}
Our proof revolves around a standard concept
of fixed point theory: contractive
mappings.  We insist on the computational efficiency of this design,
linear in the coding size
of the Wasserstein distance between class marginals. 
It shows that, on some adversarial training problems, the
existence of the weakest forms of adversaries implies that much
stronger ones may be available at cheap (computational) cost.

% !TEX root=../nips18-adversarial-mf-1.tex

\section{Related work}\label{sec-rel}

%% INSIST: VERTICAL PARTITION
%%

Formal approaches to the problem of adversarial training are sparse
compared to the growing literature on the arms race of experimental
results. The formal trend has started on adversarial
changes to a loss to
be optimized \citep{sndCS} or more directly on a classifier's output \citep{haFG,rslCD}. 
For example, \citep{sndCS} add a Wasserstein penalty to
a loss, computing the distance between the true and adversarial
distributions. 
They provide smoothness guarantees for the loss and robustness in
generalization for its minimization. 
\citep{rslCD} directly penalize
the classifier's output (not a loss per se), in the context of shallow
networks, and compute adversarial perturbations in a bounded $L_\infty$
ball. 
A similar approach (but in $L_p$-norm) is taken in \citep{haFG} for kernel methods
and shallow networks. 
Recent ones also focus on introducing general robustness
constraints \citep{bilvncMN}. 

More recently, a handful of work have
started to investigate the \textit{limits} of learning in an
adversarial training setting, but they are limited in that 
they address particular simulated
domains with a particular loss to be optimized, and consider
particular adversaries \citep{bprAE,fffAV,gmfsrwgAS}. The distribution
can involve Gaussians of mixtures \citep{bprAE,fffAV} or the data
lies on concentric spheres \citep{gmfsrwgAS}. The loss involves a distance based
on a norm for all, and the adversary makes local shifts to data of bounded
radius. In the case of \cite{bprAE}, the access to the data is
restricted to statistical queries. The essential results are either
that robustness requires too much information compared to not
requiring robustness \citep{bprAE}, or the "safety" radius of
inoffensive modifications is in fact small relative to some of the
problem's parameters, meaning even "cheap" adversaries can sometimes
generate damaging
adversarial examples \citep{fffAV,gmfsrwgAS}. This depicts a pretty
negative picture of adversarial training --- negative but \textit{local}:
all these results share the same common design pattern of relying on
particular choices 
for all key components of the problem: domain,
loss and adversaries (and eventually classifiers). There is no approach to
date that would relax any of these choices, even less so one that
would simultaneously relax all.
% !TEX root=../nips18-adversarial-mf-1.tex

\section{Definitions and notations}
\label{sec:defs}

We present some important definitions and notations.

\noindent $\triangleright$ \textit{Proper losses}.
Many of our notations follow \cite{rwCB}.
Suppose we have a prediction problem with binary labels.
We let $\properloss : \{-1,1\}
\times [0,1] \rightarrow \overline{\mathbb{R}}$ denote a general loss
function
% AKM: edit
% AKM: moved later
%\footnote{Such losses, for which \textit{properness} makes particular sense, are in fact called class probability estimation losses \cite{rwCB}.}
to be
minimized, where the left argument is a class $\Y \in \{-1, 1\}$ and
the right argument is a class probability estimate
($\overline{\mathbb{R}}$ is the closure of $\mathbb{R}$). Its \emph{conditional Bayes risk} function
is
the best achievable loss
when labels are drawn with a particular positive base-rate,
\begin{eqnarray}
\cbr(\pi) & \defeq &
\inf_c \E_{\Y\sim \pi} \properloss(\Y, c),
\end{eqnarray} 
where $\pi \in [0,1]$, 
% AKM: edit
%$\Pr[\Y = 1] \defeq \pi$ and so $\Pr[\Y = -1] \defeq 1-\pi$.
so that $\Pr[\Y = 1] = \pi$ and $\Pr[\Y = -1] \defeq 1-\pi$.
We call the loss \textit{\textbf{proper}} iff Bayes prediction
locally achieves the minimum everywhere%
% AKM: edit
% AKM: added here
\footnote{Losses for which \textit{properness} makes particular sense are called class probability estimation losses \citep{rwCB}.}%
: $ \cbr(\pi) = \E_{\Y} \properloss(\Y, \pi), \forall \pi
\in [0,1]$. One value of $\cbr$ is interesting in our context, the one
which corresponds to Bayes rule returning maximal "uncertainty",
\textit{i.e.} for $\pi = 1/2$,
\begin{eqnarray}
\properloss^\circ & \defeq & \cbr\left(\frac{1}{2}\right).\label{defmaxunc}
\end{eqnarray}
Without further ado, we give the key definition which makes more
precise the framework sketched in \eqref{eqNOVO}.
\begin{definition}\label{defADVLOSSS}
For any proper loss $\properloss$ and $(\mathcal{H}, \mathcal{A})$ integrable with respect to some
distribution $D$, the \textbf{adversarial loss} $\properloss(\mathcal{H}, \mathcal{A},
D)$ is defined as
\begin{eqnarray}
\label{eqn:adversarial-loss}
\properloss(\mathcal{H}, \mathcal{A},
D) & \defeq & \min_{h\in \mathcal{H}} \E_{(\X, \Y) \sim D} \left[ \max_{a \in \mathcal{A}} \properloss(\Y, h\circ
  a(\X))\right].
\end{eqnarray}
For any $\epsilon \in [0,1]$, we say that $\mathcal{H}$ is
\textbf{$\epsilon$-defeated} by $\mathcal{A}$ on $\properloss$
iff
\begin{eqnarray}
\properloss(\mathcal{H}, \mathcal{A},
D) & \geq & (1-\epsilon) \cdot \properloss^\circ.
\end{eqnarray}
\end{definition}
Intuitively,
% AKM: edit
if the adversary can modify instances such that
%if
the learner does not do much better than a trivial blunt constant predictor,
the adversary can declare victory. 
The additional quantities (such as
the integrability condition) are given later in this section. To
finish up with general proper losses, as an example, the log-loss given by $\ell( +1, c ) = -\log c$ and $\ell( -1, c ) = -\log (1 - c)$ is proper,
with conditional Bayes risk given by the Shannon entropy $\cbr(\pi) =
-\pi \cdot \log \pi - (1 - \pi) \cdot \log(1 - \pi)$. 

\noindent $\triangleright$ \textit{Composite, canonical proper losses}.
We let $\mathcal{H} \subseteq \mathbb{R}^{\mathcal{X}}$ denote a
set of classifiers.
To convert real valued predictions
into class probability estimates \citep{mnGL},
one traditionally uses an invertible \textit{link}
function $\psi : [0,1] \rightarrow \mathbb{R}$,
forming a \textit{composite} loss $\ell_\psi( y, v ) \defeq \ell( y,
\psi^{-1}( v ) )$ \citep{rwCB}. We shall leave hereafter the adjective
composite for simplicity, and the link implicit from context whenever appropriate.
The unique (up to multiplication or addition by a scalar
\citep{bssLF}) \textit{canonical link} for a proper loss $\properloss : \{-1,1\}
\times [0,1] \rightarrow \mathbb{R}$ is defined from the conditional Bayes
risk as $\psi \defeq - \cbr'$ \citep[Section 6.1]{rwCB},
\citep{bssLF}.
As an example, for log-loss we find
% AKM: edit
the canonical link
$\psi( u ) = \log \frac{u}{1 - u}$,
with inverse the well-known sigmoid $\psi^{-1}( v ) = ({1 +
  e^{-v}})^{-1}$. A proper loss will also be assumed to be twice differentiable. Twice
differentiability is a technical convenience to simplify
derivations. It can be removed \citep[Footnote 6]{rwCB}. A canonical
proper loss is a proper loss using the canonical link.

\noindent $\triangleright$ \textit{Adversaries}.
Let $\mathcal{A} \subseteq
{\mathcal{X}}^{\mathcal{X}}$ denote a set of adversaries,
so that any $a \in \mathcal{A}$ is allowed to transform instances in some way
(e.g., change pixel values on an image).
Suppose $D$ (fixed) denotes a
distribution over $\mathcal{X} \times \{-1, 1\}$ and $P$ (resp.
$N$)
is the
corresponding distribution
conditioned on $\Y = 1$ (resp.
$\Y = -1$). The only assumption we make about adversaries is a
measurability one.
We assume that $\forall h
\in \mathcal{H}, \forall a \in \mathcal{A}$, $h\circ a$ is integrable
with respect to $P$ and $N$:
$h\circ a \in L^1(\mathcal{X}, \mathrm{d}P) \cap
L^1(\mathcal{X}, \mathrm{d}N)$.
For the sake of simplicity, we shall denote the tuple
$(\mathcal{H}, \mathcal{A})$ \textit{integrable} with respect to
$D$.  Assuming loss $\properloss$ is proper composite with link $\psi$, there is one interesting
constant $h^\circ \in \mathbb{R}$:
\begin{eqnarray}
h^\circ & \defeq & \psi \left(\frac{1}{2}\right),
\end{eqnarray}
because this value delivers the real valued prediction corresponding to maximal
uncertainty in \eqref{defmaxunc}.
For example, when the loss is proper canonical and
furthermore required to be \textit{symmetric}, \textit{i.e.} there is
no class-dependent misclassification cost, we have \citep{nnOT}
\begin{eqnarray}
h^\circ & = & 0,
\end{eqnarray}
which corresponds to a classifier always abstaining and indeed
delivering maximal uncertainty on prediction.
It is not hard to check that $\properloss^\circ
= \E_{(\X, \Y) \sim D} [ \properloss(\Y, h^\circ)]$ is the loss of
constant $h^\circ$. So we can now see that in Definition
\ref{defADVLOSSS}, as $\epsilon \searrow 0$, training against the adversarial loss essentially produces a
classifier no better than predicting nothing. We do not assume that $h^\circ\in \mathcal{H}$,
but keep in mind that such prediction with maximal uncertainty is the
baseline against which a learner has to compete to "learn" something.

%% NEW TRIANGLE
\noindent $\triangleright$ \textit{The adversarial distortion parameter
  $\upgamma$}. We now unveil the key parameter used earlier in the
Introduction. 
For any $f \in L^1(\mathcal{X}, \mathrm{d}Q)$, $u, v \in \mathbb{R}$, we let:
\begin{eqnarray}
%\begin{eqnarray}
\phi(Q, f, u, v) & \defeq & \int_{\mathcal{X}} u\cdot(f(\ve{x}) + v)\mathrm{d}Q(\ve{x}).
\end{eqnarray}
For any
$g : \mathbb{R} \rightarrow
\mathbb{R}$, the adversarial distortion $\upgamma$ is:

\begin{tcolorbox}[colframe=blue,boxrule=0.5pt,arc=4pt,left=6pt,right=6pt,top=6pt,bottom=6pt,boxsep=0pt]
\vspace{-0.4cm}
{\small
\begin{eqnarray}
\upgamma^{g}_{\mathcal{H}, a} (P, N, \pi,
  b, c) & \defeq & \max_{h \in \mathcal{H}} \{\phi(P, g \circ h\circ
  a, \pi, b) - \phi(N, g \circ h\circ
  a, 1-\pi,-c)\}. \label{defgammag}
\end{eqnarray}
}
  \end{tcolorbox}
Finally, $\upgamma_{\mathcal{H}, a} \defeq
\upgamma^{\mathrm{Id}}_{\mathcal{H}, a}$.
While abstract, we shall shortly see that quantities
$\upgamma_{\mathcal{H}, a}, \upgamma^{g}_{\mathcal{H}, a}$ relate to a
well-known object in the study of distances between probability
distributions. Let
\begin{eqnarray}
\upgammaellpi & \defeq & \pi\cbr(1) + (1-\pi)\cbr(0).\label{defdelta}
\end{eqnarray}
As an example, we have for the the log-loss $\upgammaellpi = 0, \forall \pi$, with the convention $0\log 0 =
0$.
We remark that
$\upgammaellpi$ in \eqref{defdelta} is related to $\upgamma^{g}_{\mathcal{H},
  a}$ in \eqref{defgammag}:
\begin{eqnarray}
\upgammaellpi & = & \upgamma^{g^*}_{\mathcal{H}^*,
  a}(P, N, \pi, 0, 0),
\end{eqnarray}
for $g^* \defeq \Y \cdot \cbr$ and $\mathcal{H}^*$ the singleton classifier
which makes the hard prediction $0$ over $N$ and $1$ over $P$
(Hereafter, we note $\upgamma^{g^*}_{\mathcal{H}^*,
  a}$ instead of $\upgamma^{g^*}_{\mathcal{H}^*,
  a}(P, N, \pi, 0, 0)$ for short). Remark
that such a classifier is not affected by a particular adversary, but
it is not implementable in the general case as it would require to
know the class of an observation. 
% !TEX root=../nips18-adversarial-mf-1.tex

\section{Main result: the hardness theorem}\label{sec-imp} 

We now show a lower bound on the adversarial loss of \eqref{eqn:adversarial-loss}.

\begin{theorem}\label{thPCL}
For any proper loss $\properloss$, link $\psi$ 
and any $(\mathcal{H}, \mathcal{A})$
integrable with respect to $D$, the following holds true:
\begin{eqnarray}
\properloss(\mathcal{H}, \mathcal{A},
D) & \geq &\left(\properloss^\circ - \frac{1}{2} \cdot
  \min_{a \in \mathcal{A}} \beta_a \right)_+,
\end{eqnarray}
where:
\begin{eqnarray}
x_+ & \defeq & \max\{0, x\},\nonumber \\
\beta_a & \defeq & \upgamma^g_{\mathcal{H}, a} (P, N, \pi,
  2\cbr(1), 2\cbr(0)),\nonumber\\
g & \defeq & (-\cbr') \circ \psi^{-1} .\label{defgg}
\end{eqnarray}
\end{theorem}
(all other parameters implicit in the definition of $\beta_a$, Proof in \supplement, Section \ref{proof_thPCL}) 
This pins down a simple condition for the adversary to defeat
$\mathcal{H}$.
\begin{corollary}\label{clink}
Under the conditions and with notations of Theorem \ref{thPCL}, if there exists $\epsilon \in [0,1]$
% AKM: edit
% AKM: added here
and $a \in \mathcal{A}$
such that
\begin{eqnarray}
% AKM: edit
%\exists a \in \mathcal{A} : \beta_a  & \leq & 2\epsilon
\beta_a  & \leq & 2\epsilon
\properloss^\circ,\label{bdefeat} 
\end{eqnarray}
then $\mathcal{H}$ is $\epsilon$-defeated by $\mathcal{A}$ on $\properloss$.
\end{corollary}
(Proof in \supplement, Section \ref{proof_thPCL}) We remark that
whenever $\ell$ is canonical, $g = \mathrm{Id}$ and so
\begin{eqnarray}
\beta_a & = & \upgamma_{\mathcal{H}, a} (P, N, \pi,
  2\cbr(1), 2\cbr(0)).
\end{eqnarray}
We also note that constants $\cbr(0), \cbr(1)$ get out of the
maximization problem in \eqref{defgammag} so when $\ell$ is canonical, the \textit{optimisation} of $\upgamma_{\mathcal{H}, a}$ does \textit{not} depend on the loss at
hand --- hence, its optimisation by an adversary could be done without
knowing the loss that the learner is going to minimise. 
We also remark
that the condition for $\mathcal{H}$ to be $\epsilon$-defeated by
$\mathcal{A}$ does not involve an algorithmic component: it means that
\textit{any} learning algorithm minimising loss $\ell$ will end up
with a poor predictor if \eqref{bdefeat} is satisfied, regardless of
its computational resources.

\noindent $\triangleright$ \emph{Relationships with integral probability metrics}.
In a special case, 
the somewhat abstract quantity $\upgamma^g_{\mathcal{H}, a}$ can be related to the more familiar class of integral probability metrics (IPMs)~\citep{sfgslOI}.  
The latter are a class of metrics on probability distributions,
capturing \textit{e.g.} the total variation divergence, Wasserstein
distance, and maximum mean discrepancy. The proof of the following
Corollary is immediate.

\begin{corollary}\label{corIPM}
Suppose $\upgammaellpi = 0$ and
$\mathcal{H}$ is closed by negation. Then
\begin{eqnarray*}
\lefteqn{2 \cdot \beta_a}\nonumber\\
& \hspace{-0.6cm}= & \hspace{-0.5cm}\max_{h \in \mathcal{H}} \left|
    \int_{\mathcal{X}} g\circ h\circ a(\ve{x}) \mathrm{d}P(\ve{x})-\int_{\mathcal{X}} g\circ h\circ a(\ve{x}) \mathrm{d}N(\ve{x})\right|,
\end{eqnarray*}
which is the integral probability metric for the class $\{g\circ h\circ
a : h \in \mathcal{H}\}$ on $P$ and $N$. Here, $g$ is defined in \eqref{defgg}.
\end{corollary}
% AKM: edit
%Returning to Theorem \ref{thPCL}, we can interpret the result as saying:
We may now interpret Theorem \ref{thPCL} as saying:
for an adversary to defeat a learner minimising a proper loss,
it suffices to make a suitable IPM between the class-conditionals $P, N$ small.
The particular choice of IPM arises from the learner's choice of hypothesis class, $\mathcal{H}$.
Of particular interest is when this comprises kernelized scorers,
as we now detail.

\noindent $\triangleright$ \emph{Relationships with the maximum mean discrepancy}.
The maximum mean discrepancy (MMD)~\citep{gbrssAK} corresponds to an IPM where $\mathcal{H}$ is the unit-ball in an RKHS.
We have the following re-expression of $\gamma_{\mathcal{H}, a}$ for this hypothesis class,
which turns out to involve the MMD.

\begin{figure}[t]
%\vspace{-20px}
\centering
\includegraphics[trim=20bp 350bp 660bp
30bp,clip,width=.70\linewidth]{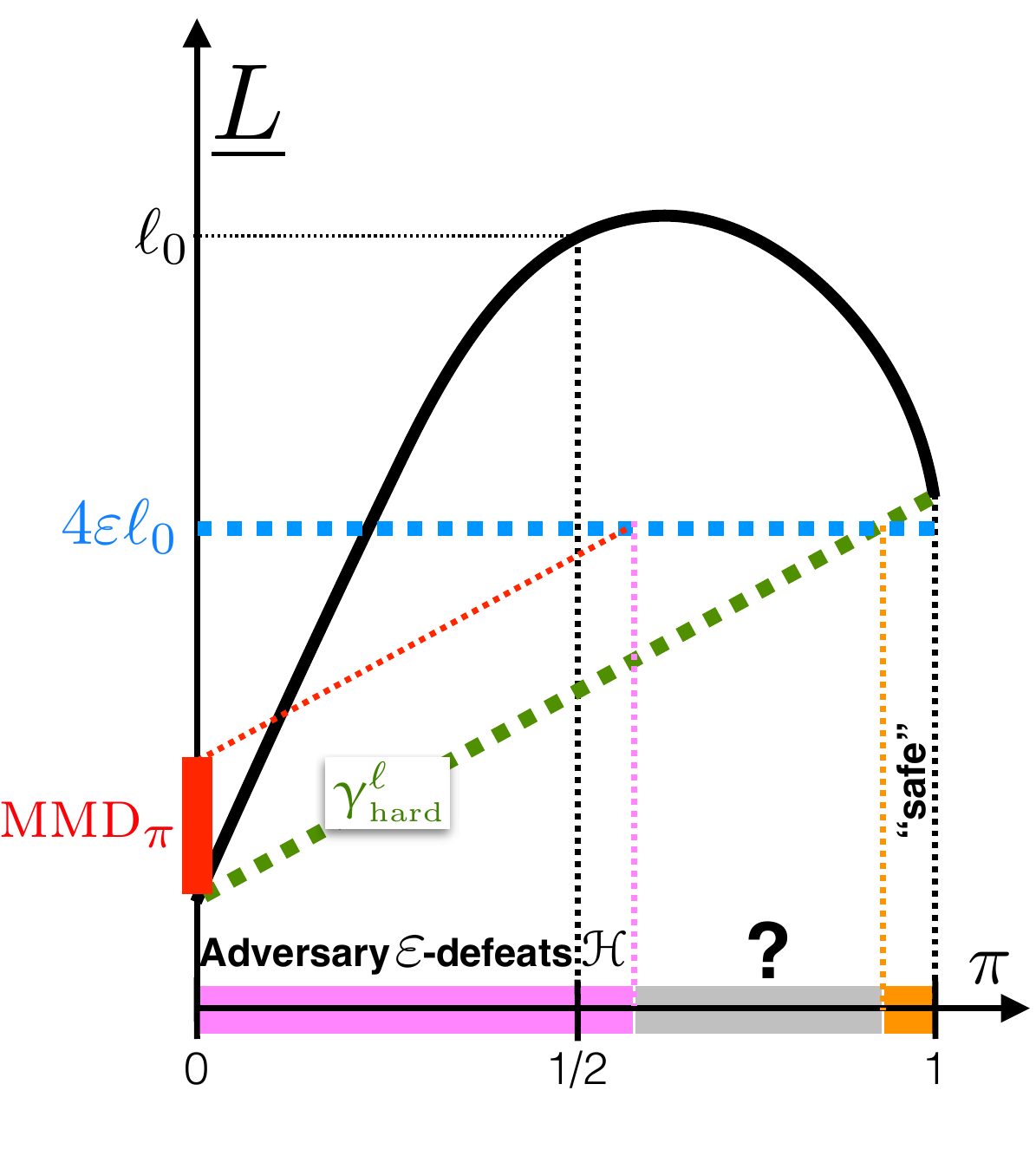} 
\vspace{-0.6cm}
\caption{Suppose an adversary $a$ can guarantee an upperbound on
  $\textsc{mmd}_\pi $ as displayed in thick red. For some
  fixed $\ell (\cbr)$ and $\epsilon$, we display the range of $\pi$ values
  (in pink) for which $a$ $\epsilon$-defeats $\mathcal{H}$. Notice
  that outside this interval, it may not be possible for $a$ to
  $\epsilon$-defeat $\mathcal{H}$ (in grey, tagged "?"), and if $\pi$ is large enough
  (orange, tagged "safe"), then it is not possible for condition \eqref{bdefeat} to
  be satisfied anymore.\label{fig:lbar}}
\vspace{-0.3cm}
\end{figure}
\begin{corollary}\label{corMMD}
Suppose $\ell$ is proper canonical and let $\mathcal{H}$ denote the unit ball of a
reproducing kernel
Hilbert space (RKHS) of functions with reproducing kernel $\kappa$.
Denote 
\begin{eqnarray}
\mu_{a, Q} & \defeq & \int_{\mathcal{X}} \kappa( a(\ve{x}), .)
    \mathrm{d}Q(\ve{x})
\end{eqnarray}
the adversarial mean embedding of $a$ on $Q$. If $\pi = 1/2$ and
$\upgammaellpi = 0$, then
\begin{eqnarray}
2 \cdot \beta_a & = &  \frac{1}{2} \cdot \|\mu_{a, P} - \mu_{a, N}\|_{\mathcal{H}}  . \label{eqMMDpi}
\end{eqnarray}
\end{corollary}
The constraints on $\pi, \upgammaellpi$ are for readability: the proof (in \supplement, Section \ref{proof_corMMD}) shows a more
general result, with unrestricted $\pi, \upgammaellpi$. The right-hand side of \eqref{eqMMDpi} is proportional to the MMD
between $P$ and $N$. In the more general case, the right-hand side of
\eqref{eqMMDpi} is replaced by $\textsc{mmd}_\pi \defeq \|\pi \cdot \mu_{a, P} - (1-\pi) \cdot
  \mu_{a, N}\|_{\mathcal{H}}$. Figure \ref{fig:lbar} displays an
  example picture (for unrestricted $\pi, \upgammaellpi$) for some canonical
proper but asymmetric
loss ($\cbr(0) \neq \cbr(1)$) when an adversary with a given
upperbound guarantee on $\textsc{mmd}_\pi$ can indeed $\epsilon$-defeat
some $\mathcal{H}$. We remark that while this may be possible for a whole
range of $\pi$, this may not be possible for all. The picture would be
different if the loss were symmetric (Corollary \ref{corSYML}
below), since in this case a guarantee to $\epsilon$-defeat
$\mathcal{H}$ for \textit{some} $\pi$ would imply a guarantee for
\textit{all}. Loss asymmetry thus brings
a difficulty for the adversary which, we recall, cannot act on $\pi$.

\noindent $\triangleright$ \emph{Simultaneously defeating
  $\mathcal{H}$ over \textbf{sets} of losses}. Satisfying
\eqref{bdefeat} involves at least the knowledge of one value of the loss, if not
of the loss itself. It turns out that if the loss is canonical and
the adversary has just a
partial knowledge of it, it may in fact still be possible
for him to guess whether \eqref{bdefeat} can be satisfied over this
set, as we now show.

\begin{corollary}\label{corSYML}
Let $\mathcal{L}$ be a set of canonical proper losses satisfying the
following property: $\forall \properloss \in \mathcal{L}, \exists
\cbr^\dagger\in \mathbb{R}$ such that $\cbr(1) = \cbr(0)
\defeq \cbr^\dagger$. Assuming $(\mathcal{H}, \mathcal{A})$
integrable with respect to $D$, if 
\begin{eqnarray}
\exists a \in \mathcal{A} : \upgamma_{\mathcal{H}, a} (P, N, \pi,
  0, 0) & \leq & \epsilon  \cdot \inf_{\ell \in
  \mathcal{L}} \properloss^\circ - \cbr^\dagger,
\end{eqnarray}
then $\mathcal{H}$ is jointly $\epsilon$-defeated by $\mathcal{A}$ on
\textbf{all} losses of $\mathcal{L}$.
\end{corollary}
Notice that all the adversary needs to know is $\mathcal{L}$.
The result easily follows from remarking that we have in this case:
\begin{eqnarray*}
\beta_a & = & 2 \cbr^\dagger + \upgamma_{\mathcal{H}, a} (P, N, \pi,
  0, 0),
\end{eqnarray*}
which we then plug in \eqref{bdefeat} to get the statement of the
Corollary. Corollary \ref{corSYML} is interesting for two reasons.
First, it applies to all
proper symmetric losses \citep{nnOT,rwCB}, which includes popular
losses like the square, logistic and Matsushita losses. Finally, it does not
just offer the adversarial strategy to defeat classifiers that would
be learned on any of such losses, it also applies to more
sophisticated learning strategies that would \textit{tune} the loss at
learning time \citep{nnOT,rwCB} or \textit{tailor} the loss to
specific constraints \citep{bssLF}.
% !TEX root=../nips18-adversarial-mf-1.tex

\section{Monge efficient adversaries}\label{sec-mong}

We now highlight a sufficient condition on adversaries for
\eqref{bdefeat} to be satisfied, which considers classifiers in the
increasingly popular framework of "Lipschitz classification" for
adversarial training \citep{cbgduPN}, and turns out to frame
adversaries in optimal transport (OT) theory \citep{vOT}. We proceed in three
steps, first framing OT adversaries, then Lipschitz classifiers and
finally showing how the former defeats the latter.
\begin{definition}
Given any $c :
\mathcal{X} \times \mathcal{X} \rightarrow \mathbb{R}$ and some $\delta \in \image(c)$, we say that
$\mathcal{A}$ is \textbf{$\delta$-Monge efficient} for cost $c$ on marginals
$P, N$ iff $\exists a \in \mathcal{A} : C(a,P,N) \leq \delta$,
with
\begin{eqnarray*}
C(a,P,N) & \defeq & \inf_{\muup \in \Pi(P, N)} \int
 c(a (\ve{x}), a (\ve{x}')) \mathrm{d}\muup(\ve{x}, \ve{x}'),
\end{eqnarray*}
and $\Pi$ is the set of all joint probability measures whose
marginals are $P$ and $N$. 
\end{definition}
Hence, Monge efficiency relates to an efficient compression of the transport
plan between class marginals. In fact, we should require $c$ to satisfy some mild additional
assumptions for the existence of optimal couplings
\citep[Theorem 4.1]{vOT}, such as lower semicontinuity. We skip
them for the sake of simplicity, but note that infinite costs are
possible without endangering the existence of optimal couplings of
$(P, N)$ \citep{vOT}, which is convenient for the following
generalized notion of Lipschitz continuity.
\begin{definition}\label{defLip}
Let $c : \mathcal{X} \times \mathcal{X} \rightarrow \mathbb{R}$. For
some $K>0$ and $u, v : \mathbb{R} \rightarrow \mathbb{R}$, set
$\mathcal{H}$ is said to be $(u, v, K)$-Lipschitz 
with respect to $c$ iff 
\begin{eqnarray}
u\circ h(\ve{x}) - v\circ h(\ve{y}) \hspace{-0.1cm} \leq \hspace{-0.1cm} K \cdot c(\ve{x},\ve{y}), \forall h \in \mathcal{H}, \forall \ve{x},
\ve{y} \in \mathcal{X}.\label{defGENLIP}
\end{eqnarray}
\end{definition}
We shall also write that $\mathcal{H}$ is $K$-Lipschitz if Definition
\ref{defLip} holds for $u = v = \mathrm{Id}$ ($c$ implicit). Actual Lipschitz
continuity would restrict $c$ to involve a distance, and the
state of the art of adversarial training would restrict further the
distance to be based on a norm \citep{cbgduPN}. Equipped with
this, we obtain the main result of this Section.

\begin{theorem}\label{thOTA}
Fix any $\epsilon > 0$ and proper canonical loss $\ell$. Suppose $\exists c : \mathcal{X} \times \mathcal{X}
\rightarrow \mathbb{R}$ such that:
\begin{enumerate}
\item [(1)] \hspace{-0.2cm} $\mathcal{H}$ is $2K$-Lipschitz with
  respect to $c$;
\item [(2)] \hspace{-0.2cm} $\mathcal{A}$ is $\delta$-Monge efficient for cost $c$ on marginals
$P, N$ for
\begin{eqnarray}
\delta & \leq & \frac{4 \epsilon \properloss^\circ - 2
                \upgammaellpi}{K}.\label{bdelta11}
\end{eqnarray}
\end{enumerate}
Then $\mathcal{H}$ is
$\epsilon$-defeated by $\mathcal{A}$ on $\properloss$.
\end{theorem}
The proof (in \supplement, Section \ref{proof_thOTA}) is given for the
more general case where $\pi$ is not necessarily $1/2$ and 
any proper loss, not necessarily canonical. We also show in the
proof that unless
$\pi = 1/2$, $c$ cannot be a distance in the general case. We take it as a potential difficulty for the adversary which, we recall,
cannot act on $\pi$. 

Theorem \ref{thOTA} is particularly interesting with respect to the
current developing strategies around adversarial training that
"Lipschitzify" classifiers \citep{cbgduPN}. Such strategies assume
that the loss $\ell$ is Lipschitz (remark that we do not make such an
assumption). 
In short, if we rename
$\ell_{\mbox{\tiny{adv}}}$ the inner part (within $[.]$) in
\eqref{eqn:adversarial-loss}, those strategies exploit the fact that
(omitting key parameters for readability)
\begin{eqnarray}
\ell_{\mbox{\tiny{adv}}}(h) \leq \ell_{\mbox{\tiny{clean}}}(h) +
  K_\ell K_h, \forall h \in \mathcal{H},\label{eqLipMin}
\end{eqnarray} 
where
$\ell_{\mbox{\tiny{clean}}}$ is the adversary-free loss and $K_.$ is
the Lipschitz constant of the loss ($\ell$) or classifier learned
($h$). One might think that minimizing \eqref{eqLipMin} is not a good
strategy in the light of Theorem \ref{thOTA} because the
regularization enforces a minimization of $K_h$ ($K$ in Theorem
\ref{thOTA}), so we seemingly
alleviate constraints on the adversary to be $\delta$-Monge
efficient in \eqref{bdelta11} and can end up being more easily defeated. 
This is however a too simplistic conclusion that does not take into account the
other parameters at play, as we now explain in the context of
\citet{cbgduPN}. Consider the logistic loss \citep{cbgduPN}, for which:
\begin{eqnarray}
\properloss^\circ = K_\ell = 1, \upgammaellpi = 0.
\end{eqnarray}
Suppose we can reduce \textit{both} $\ell_{\mbox{\tiny{clean}}}(h)$
and $K_h$ (which is in
fact not hard to ensure for deep architectures \citep[Section
2.1]{mkkySN}, \citep{cksnDR}) so that $ K_h \leq (1-
\ell_{\mbox{\tiny{clean}}}(h)) / 2 =  (\properloss^\circ-
\ell_{\mbox{\tiny{clean}}}(h)) / (K_\ell +
\properloss^\circ)$. Reorganizing, we get $\ell_{\mbox{\tiny{clean}}}(h) +
  K_\ell K_h \leq (1 - K_h) \properloss^\circ$, so for $\mathcal{H}$ to be
  $\varepsilon$-defeated, we in fact get a constraint on $\varepsilon$: $\varepsilon \geq K_h$,
which reframes the constraint on $\delta$ in
\eqref{bdelta11} as (see also \supplement, \eqref{bdelta2}),
\begin{eqnarray}
\delta & \leq & 4\properloss^\circ - \frac{\upgammaellpi}{K_h} = 4,
\end{eqnarray}
which does not depend
anymore on $K_h$. 

The proof of
Theorem \ref{thOTA} is followed in \supplement~by a proof of an
interesting generalization in the light of
those recent results \citep{cbgduPN,cksnDR,mkkySN}: the Monge
efficieny requirement can be weakened under a form of dominance
(similar to a Lipschitz condition) of the
canonical link with respect to the chosen link of the loss. We now
provide a simple family of Monge efficient adversaries.

\noindent $\triangleright$ \emph{Mixup adversaries.} Very recently, it was experimentally demonstrated how a
simple modification of a training sample yields models more likely to
be robust to
adversarial examples and generalize better \citep{zcdlMB}. The process can be
summarized in a simple way: perform random interpolation between two
randomly chosen
training examples to create a new example (repeat as necessary). Since we do not allow the
adversary to tamper with the class, we define as
$\lambda$-\textit{mixup} (for $\lambda \in [0,1]$) the
process which creates for two observations $\ve{x}$ and $\ve{x}'$ having a
different class the following adversarial observation (same
class as $\ve{x}$):
\begin{eqnarray}
a(\ve{x}) & \defeq & \lambda\cdot \ve{x} + (1-\lambda)\cdot
\ve{x}'.\label{defmixup}
\end{eqnarray}
We make the assumption that $\mathcal{X}$ is metric with an associated distance that stems from
this metric. We analyze a very simple case of $\lambda$-mixup, which we call
\textit{$\lambda$-mixup to $\ve{x}^*$}, which replaces $\ve{x}'$ by some
$\ve{x}^*$ in $\mathcal{X}$ in \eqref{defmixup}. Notice that as $\lambda \rightarrow 0$,
we converge to the maximally harmful adversary mentioned in the
introduction. The intuition thus suggests that the set $\mathcal{A}$ of all $\lambda$-mixups
to some $\ve{x}^*$ (where we vary $\lambda$) designs in fact an arbitrarily
Monge efficient adversary, where the optimal transport problem involves
the associated distance of $\mathcal{X}$. This is indeed true and in
fact simple to show.
\begin{lemma}\label{lemMIX1}
For any $\delta > 0$ 
the set of all $\lambda$-mixups to $\ve{x}^*$ is $\delta$-Monge
efficient for $\lambda \leq \delta / W_1$,
where $W_1$ is the
1-Wasserstein distance between the class marginals.
\end{lemma}
(Proof in \supplement, Section \ref{proof_lemMIX1}) 
The mixup methodology as defined in \cite{zcdlMB} can be specialized
in numerous ways: for example, instead of mixing up with a single
observation, we could perform all possible mixups within $\mathcal{S}$
in a spirit closer to \cite{zcdlMB}, or mixups with several distinguished
observations (\textit{e.g.} after clustering), etc. . Many choices like these
would be eligible to be at least Monge efficient, but while they can
be computationally simple to compute, they are just surrogates for
Monge efficiency: tackling directly the compression of the optimal
transport plan is a more direct option to Monge efficiency.
% !TEX root=../nips18-adversarial-mf-1.tex

\section{From weak to strong Monge efficiency}\label{sec-wts}

In Theorem \ref{thOTA}, we
showed how Monge efficiency for adversaries can "take over" Lipschitz
classifiers and defeat them for some
$\epsilon > 0$. Suppose now that the $\mathcal{A}$ we have is
\textit{weak} in that all its elements are Monge efficient but for large
values of $\delta$. In other words, we cannot satisfy condition (2)
in Theorem \ref{thOTA}. Is there another set of adversaries,
$\mathcal{A}^\star$, whose elements would combine the elements of
$\mathcal{A}$ is a computationally savvy way, and which would achieve
any desired level of Monge efficiency? Such a question parallels that
of the boosting framework in supervised learning, in which one combines classifiers just
different from random to achieve a combination arbitrarily accurate \citep{sfBF}.

We now answer our question by the affirmative, in the context of
kernel machines. Let $\mathcal{H}$ denote a RKHS and $\Phi$ a feature
map of the RKHS. $\forall f : \mathcal{X} \rightarrow \mathcal{X}$,
define cost
\begin{eqnarray*}
\lefteqn{C_\Phi(f,P,N)}\nonumber\\
 & \defeq & \inf_{\muup \in \Pi(P, N)}
   \int_{\mathcal{X}} \|\Phi \circ f (\ve{x}) - \Phi \circ f (\ve{x}')\|_{\mathcal{H}} \mathrm{d}\muup(\ve{x}, \ve{x}').
\end{eqnarray*}
\begin{definition}\label{defCONT}
Function $a : \mathcal{X} \rightarrow \mathcal{X}$ is said
$\eta$-contractive for $\Phi$, for some $\eta > 0$  iff $\|\Phi \circ a(\ve{x}) - \Phi \circ a(\ve{x}')\|_{\mathcal{H}}
\leq (1-\eta) \cdot \|\Phi (\ve{x}) - \Phi (\ve{x}')\|_{\mathcal{H}}, \forall
\ve{x}, \ve{x}' \in \mathcal{X}$.
\end{definition}
Set $\mathcal{A}$ is said $\eta$-contractive for $\Phi$ iff it
contains at least one adversary $\eta$-contractive for $\Phi$ (and we make no
assumption on the others). Define now $\mathcal{A}^J \defeq \{a \circ a \circ ... \circ a \mbox{ ($J$
  times)} : a \in \mathcal{A}\}$ for any $J \in \mathbb{N}_*$, and  $W_1^\Phi \defeq \inf_{\muup \in \Pi(P, N)}
   \int_{\mathcal{X}} \|\Phi (\ve{x}) - \Phi (\ve{x}')\|_{\mathcal{H}}
   \mathrm{d}\muup(\ve{x}, \ve{x}')$, the 
1-Wasserstein distance between class marginals in
the feature map. 
\begin{theorem}\label{thmMEF}
Let $\mathcal{H}$ denote a RKHS with feature map $\Phi$ and
$\mathcal{A}$ be $\eta$-contractive for $\Phi$. Then $\mathcal{A}$ is
$\delta$-Monge efficient for $\delta = (1-\eta)\cdot
W_1^\Phi$. Furthermore, $\forall \delta >
0$, $\mathcal{A}^J$ is $\delta$-Monge efficient when $J \geq (1/\eta)\cdot \log (W_1^\Phi/\delta)$.
\end{theorem}
(Proof in \supplement, Section \ref{proof_thmMEF}) To amplify the
difference between $\mathcal{A}$ and $\mathcal{A}^J$, remark that the
worst case of Monge efficiency is $\delta =
W_1^\Phi$, since it is just the Monge efficiency for contracting
nothing. So, as $\eta \rightarrow 0$, there is barely any
guarantee we can get from the $\eta$-contractive $\mathcal{A}$ while
$\mathcal{A}^J$ can still be arbitrarily Monge efficient for a $J$ \textit{linear} in the coding size of the Wasserstein distance
between class marginals. 

% !TEX root=../nips18-adversarial-mf-1.tex

\section{Experiments}
\label{sec:toy}

\newcommand{\toyfnf}{Figs/toy_pdf/summary-crop}

\newcommand{\thescaleS}{0.035}
\newcommand{\imageincS}[1]{\hspace{-0.25cm} \includegraphics[trim=0bp 0bp 0bp
0bp,clip,width=\thescaleS\textwidth]{#1} \hspace{-0.25cm}}

\begin{figure*}[t]
\begin{center}
{\scriptsize
\begin{tabular}{c||c}
1D problem & USPS handwritten digits\\ \hline
\begin{tabular}{cc}
\hspace{-0.5cm} \includegraphics[height=0.16\textwidth,page=2]{\toyfnf} \hspace{-0.3cm} & \hspace{-0.3cm} \includegraphics[height=0.16\textwidth,page=1]{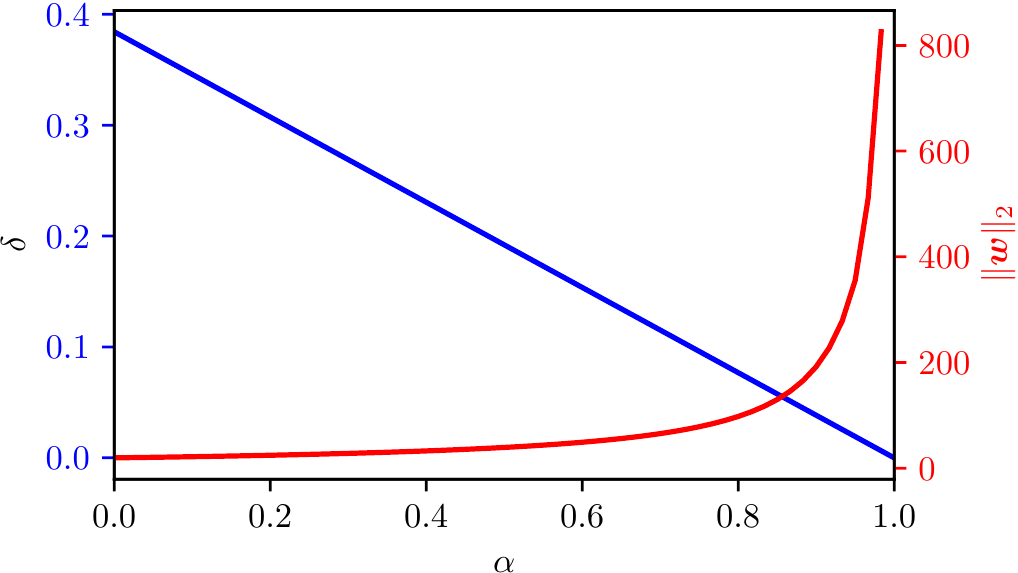} \hspace{-0.5cm} \\
Expected logistic loss& Cost $\delta$ and weight norm $\|\ve w\|_2$
\end{tabular}
&
\begin{tabular}{ccccc||ccccc}\\ 
\multicolumn{5}{c||}{transf. in \texttt{digit-1}} &
                                               \multicolumn{5}{c}{transf. in
                                               \texttt{digit-3}}\\
\realinc{0} & \realinc{0.15} & \realinc{0.3} & \realinc{0.45} &
                                                                \realinc{0.6}
  &  \realinc{0} & \realinc{0.15} & \realinc{0.3} & \realinc{0.45} &
                                                                \realinc{0.6}\\
\imageincS{{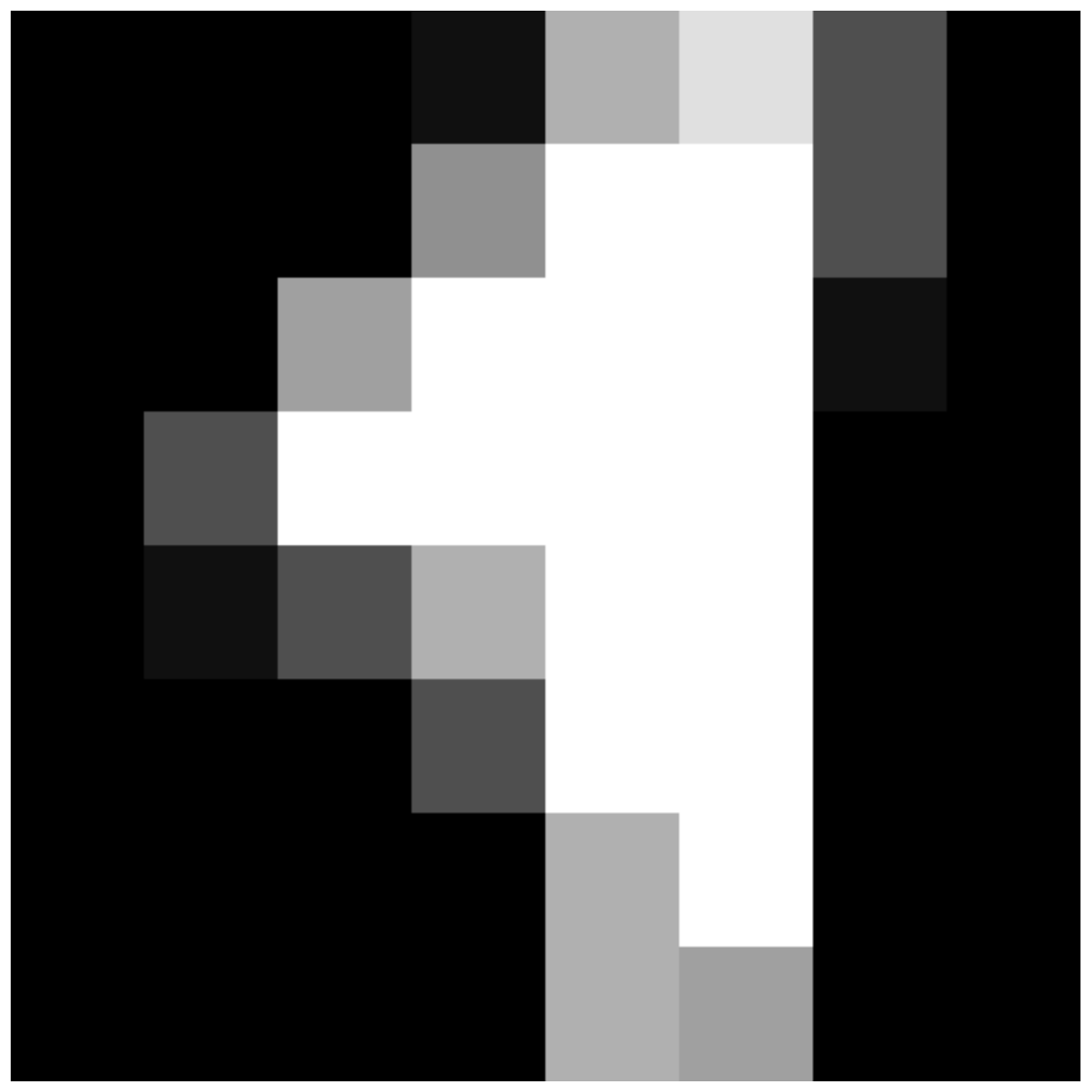}} &
\imageincS{{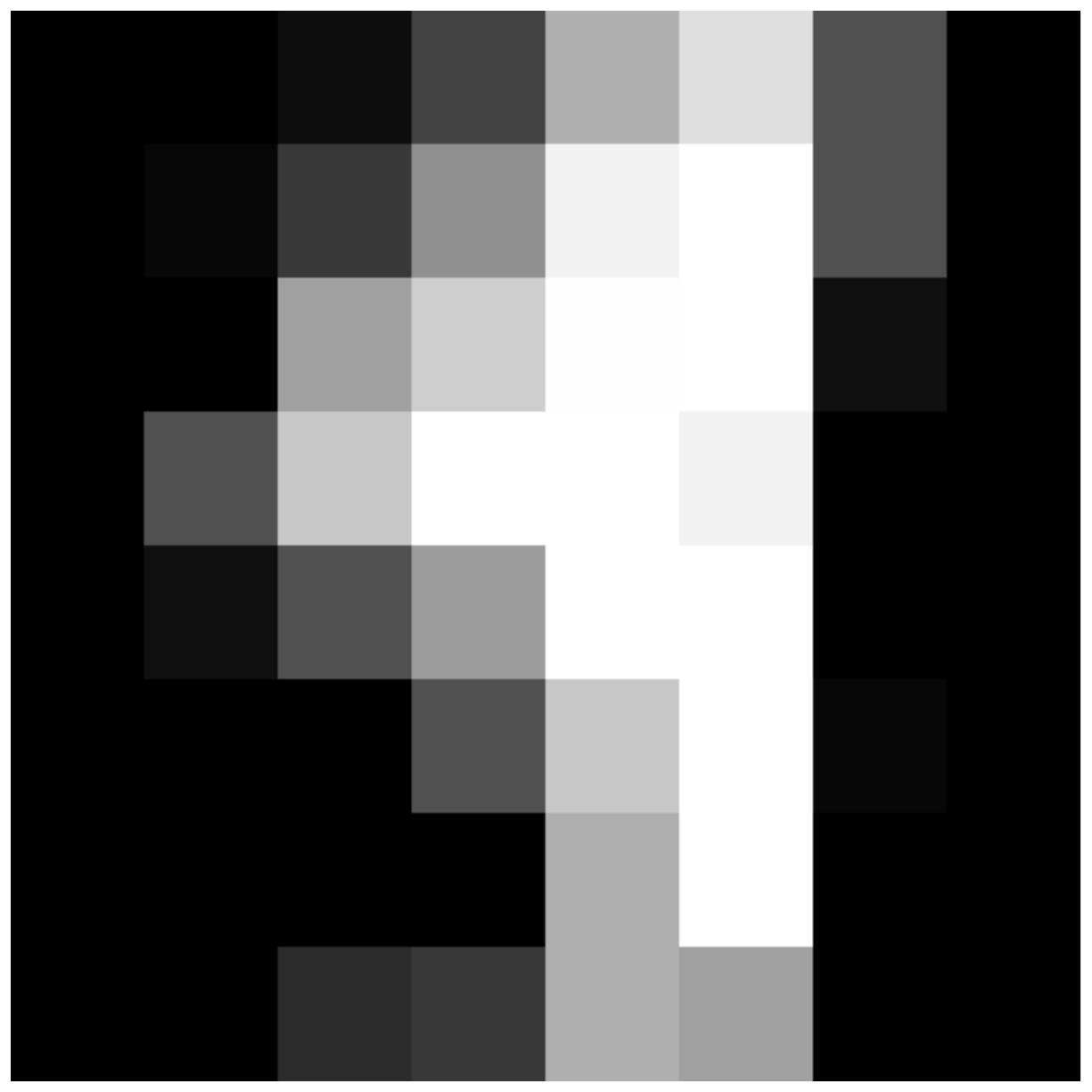}} &
\imageincS{{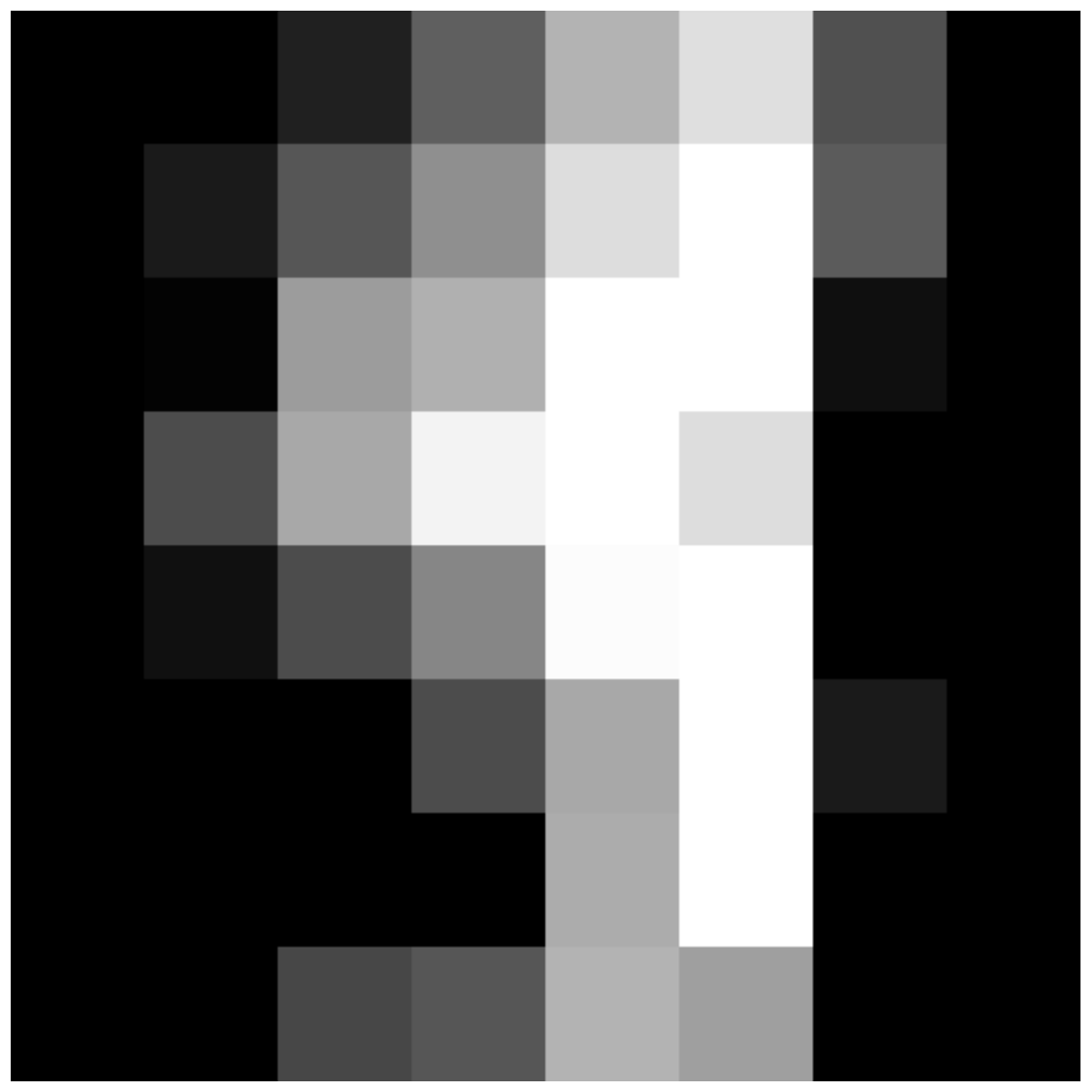}} &
\imageincS{{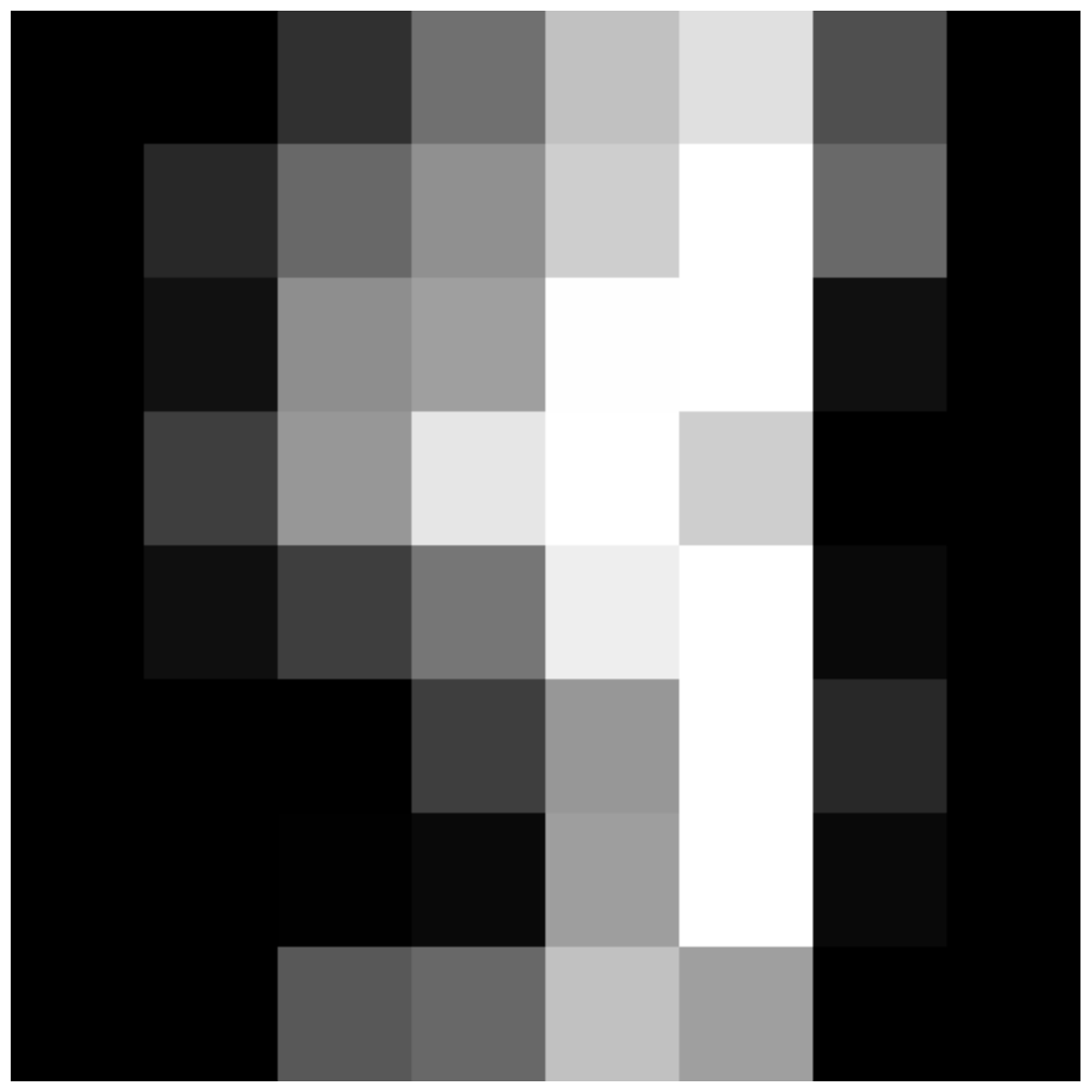}} &
\imageincS{{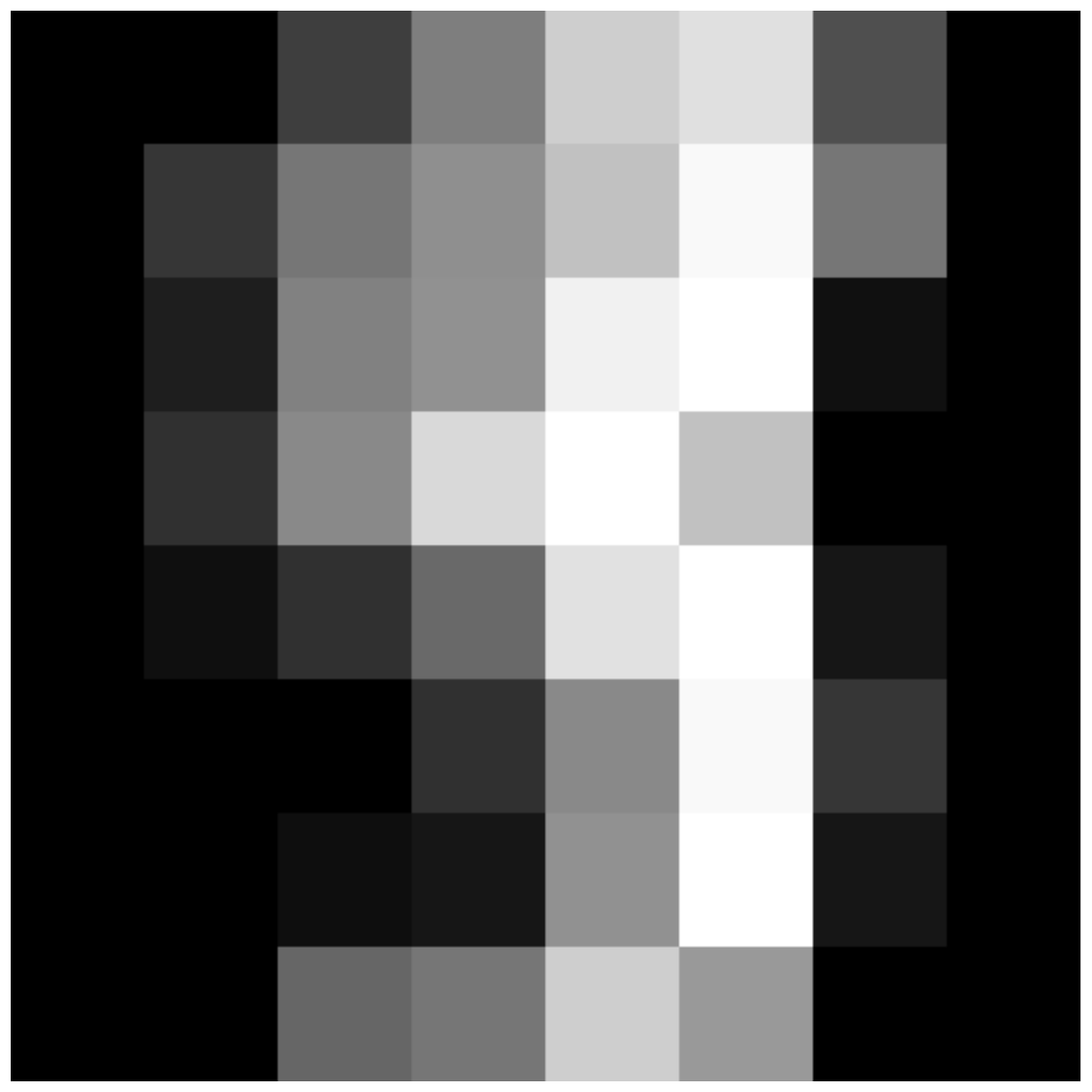}} &
\imageincS{{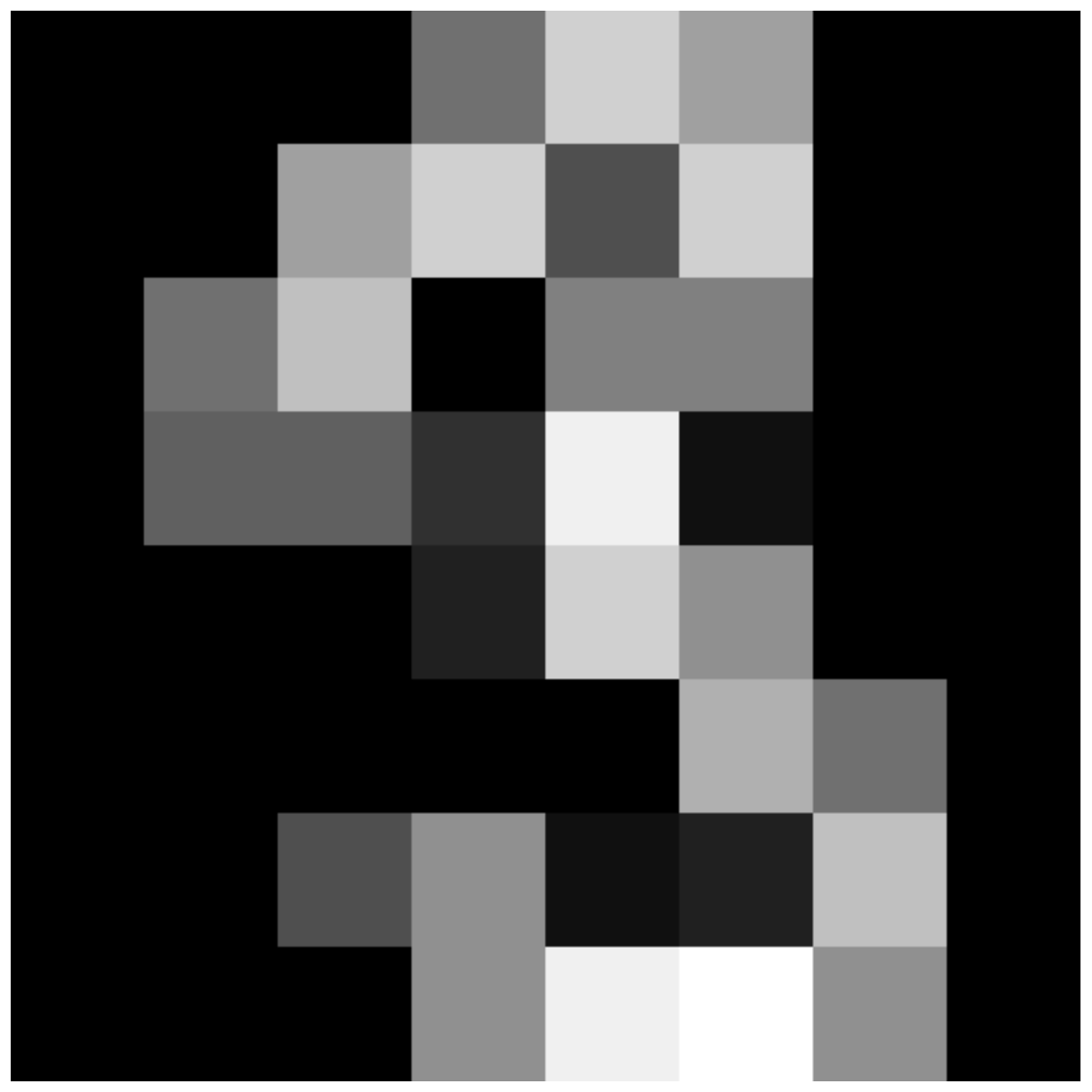}} &
\imageincS{{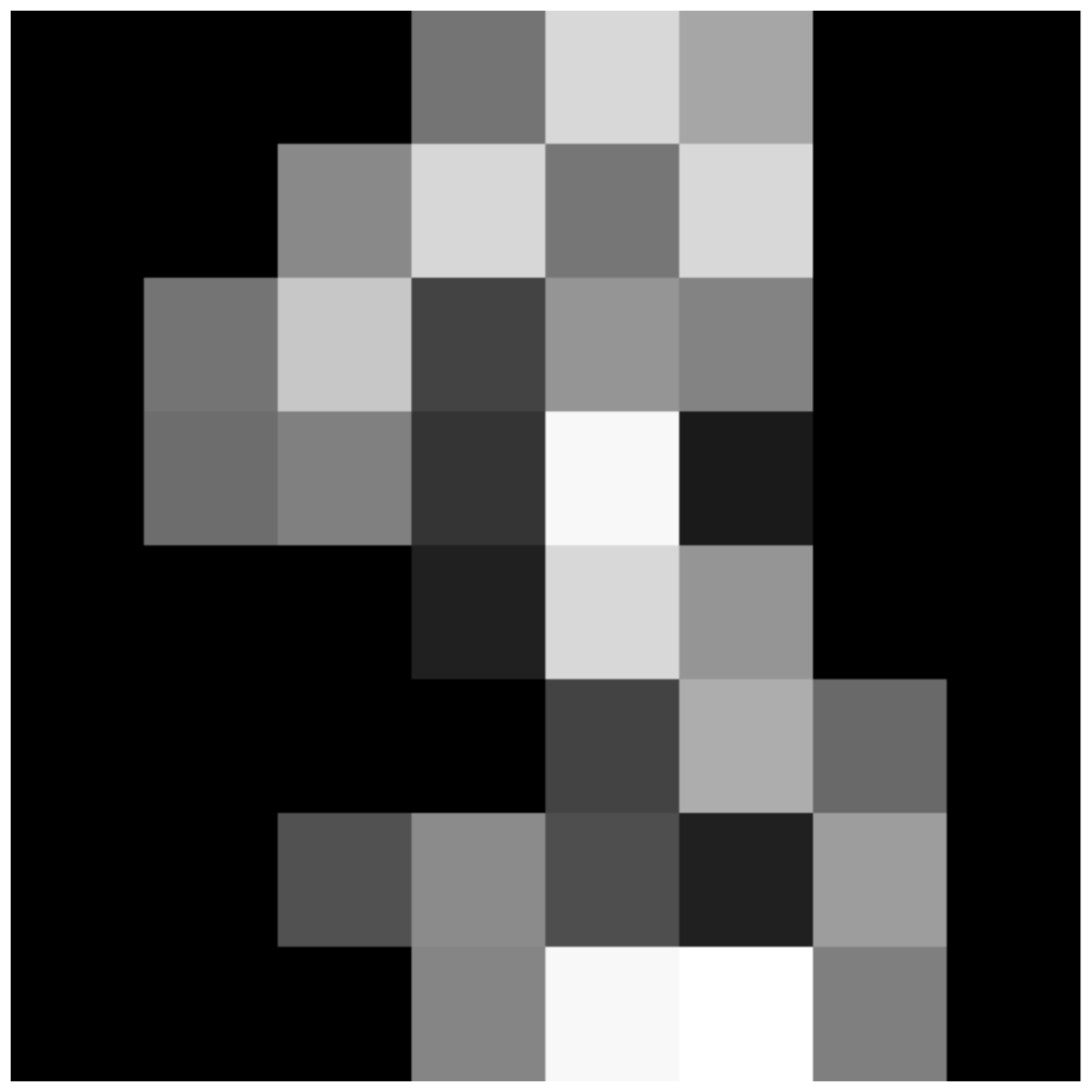}} &
\imageincS{{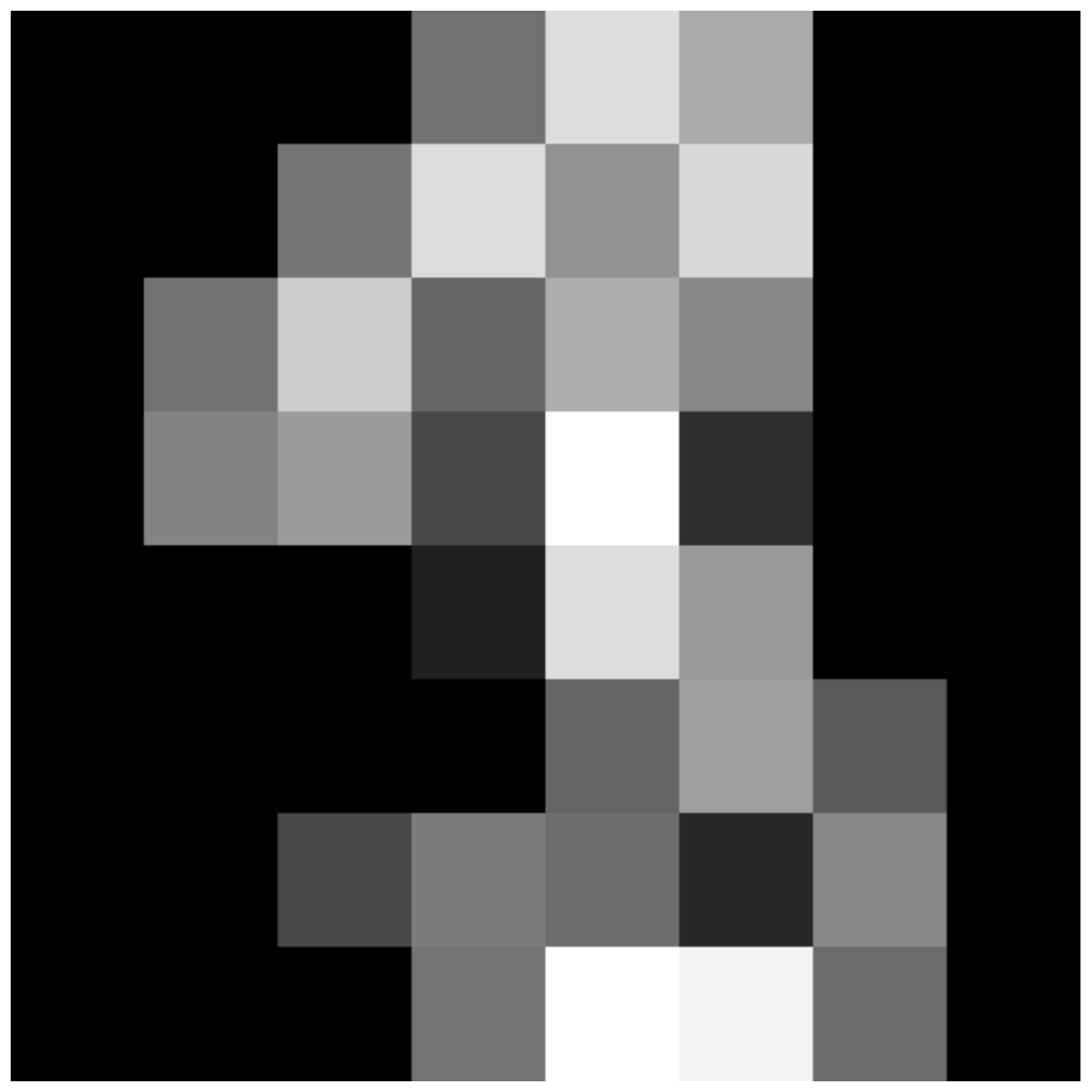}} &
\imageincS{{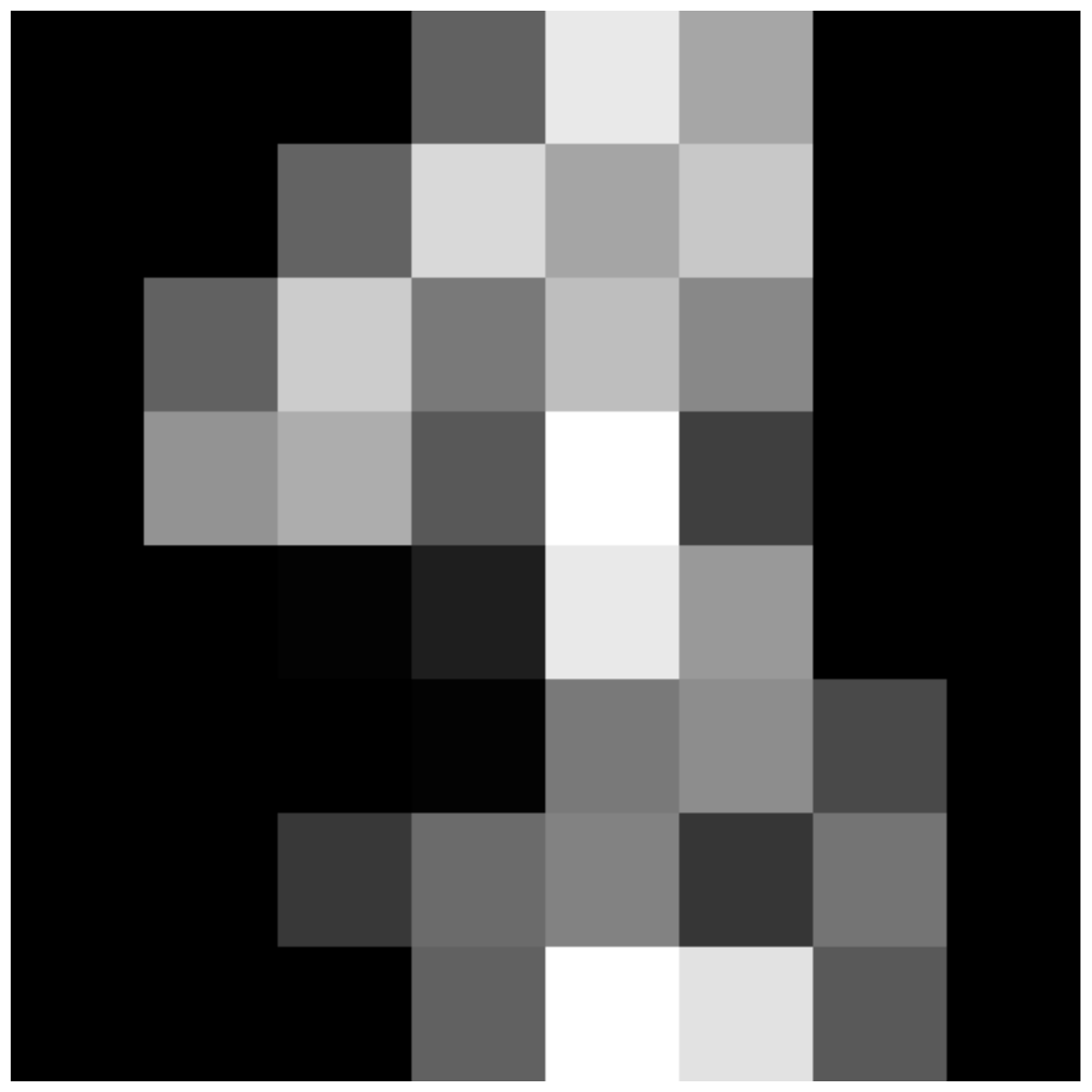}} &
\imageincS{{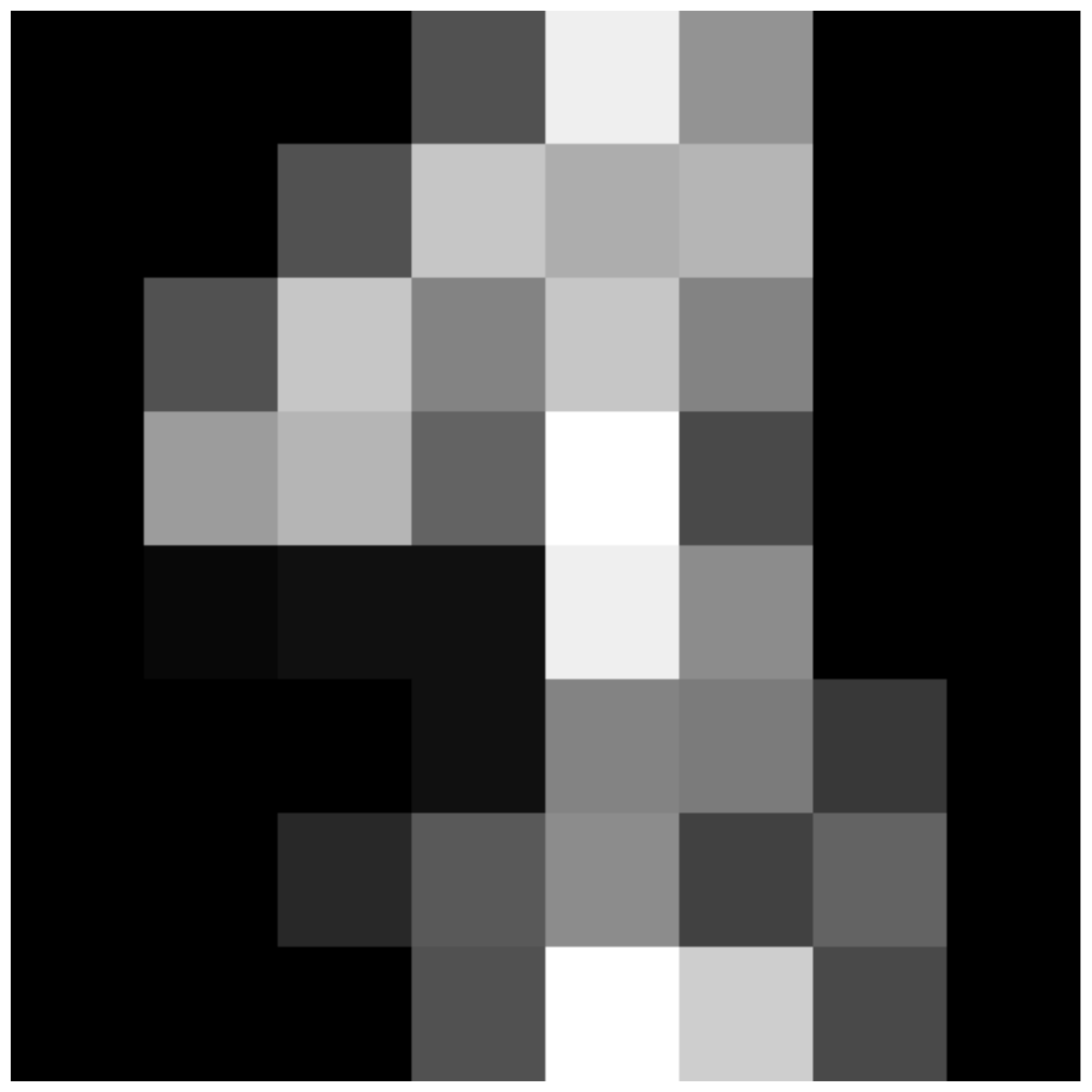}} \\
\imageincS{{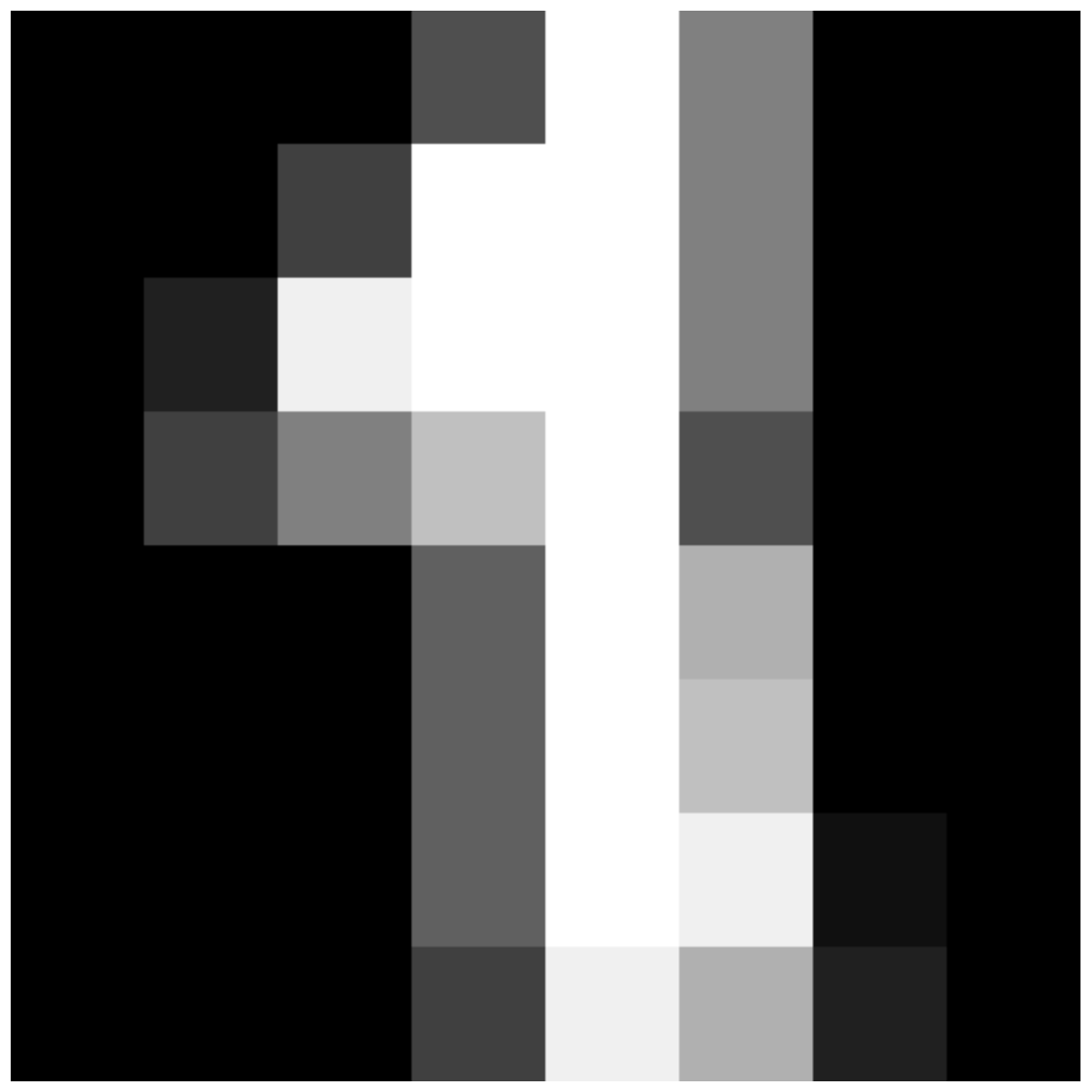}} &
\imageincS{{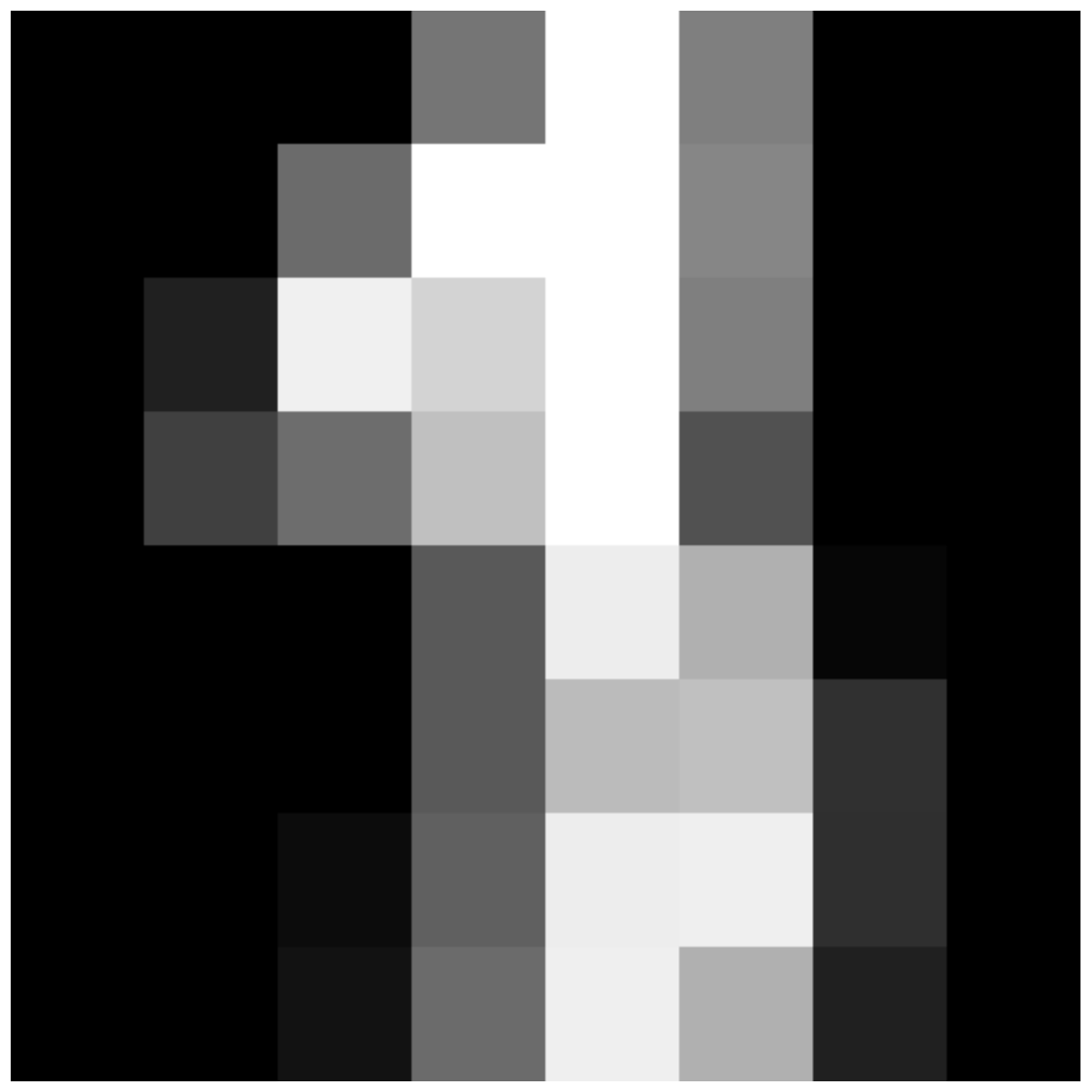}} &
\imageincS{{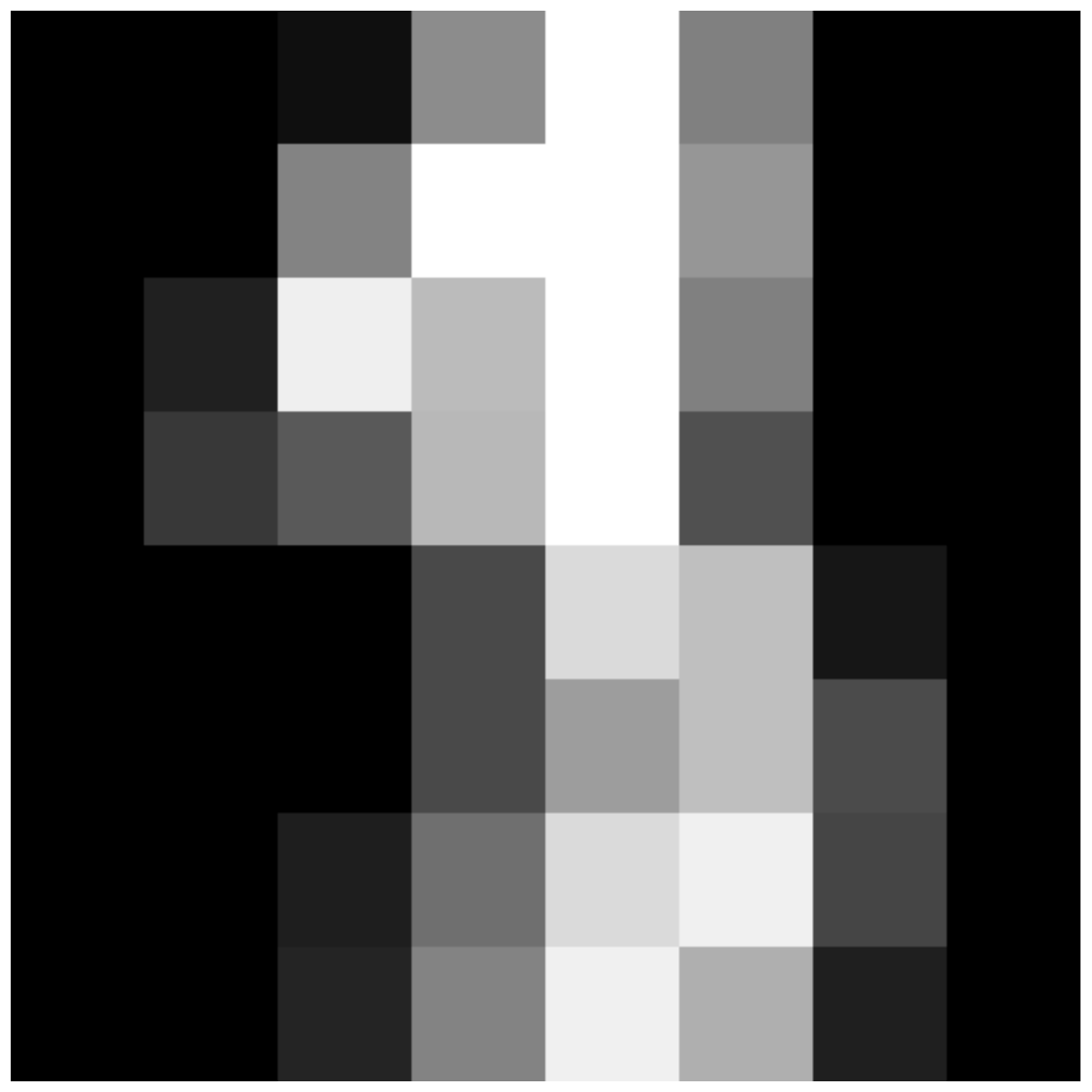}} &
\imageincS{{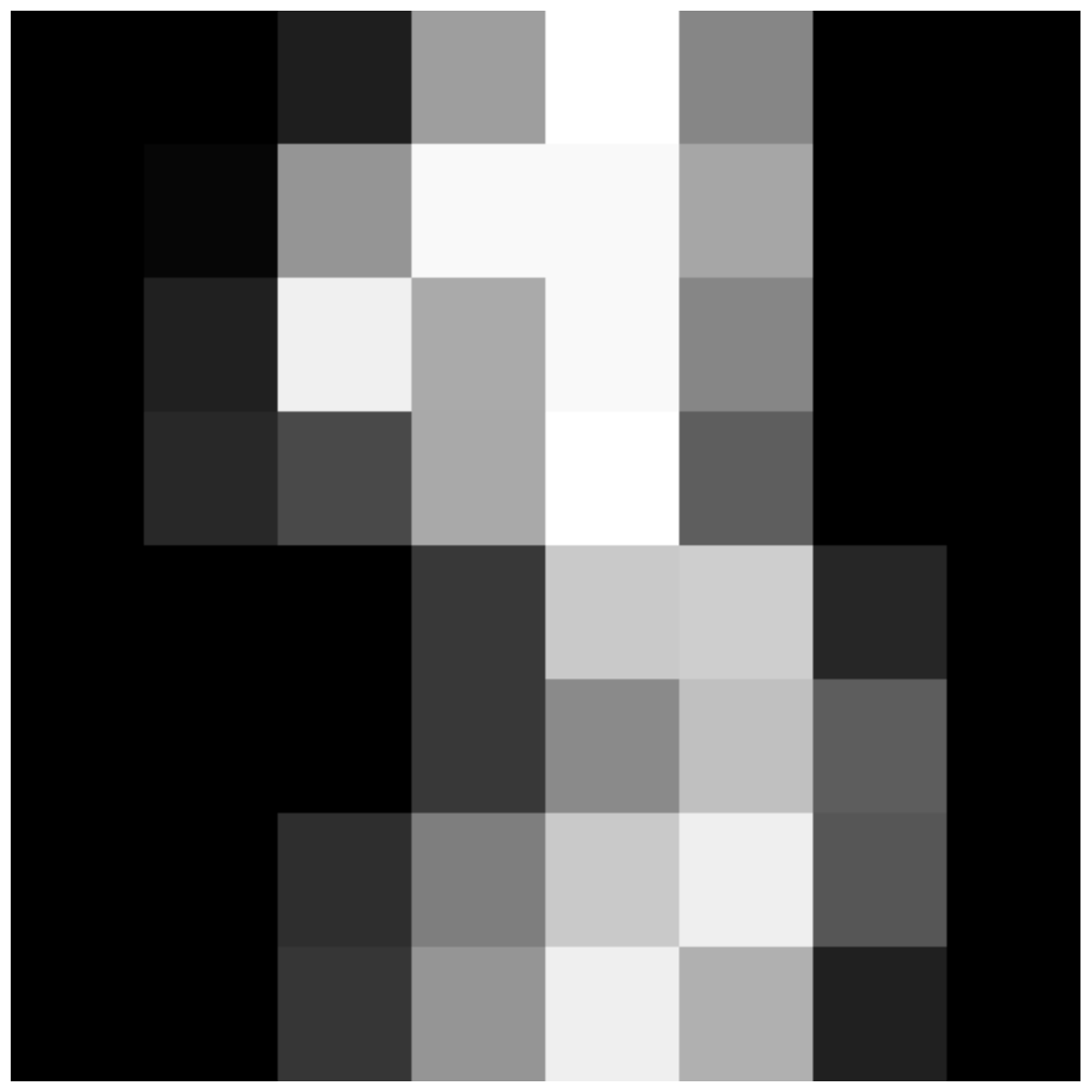}} &
\imageincS{{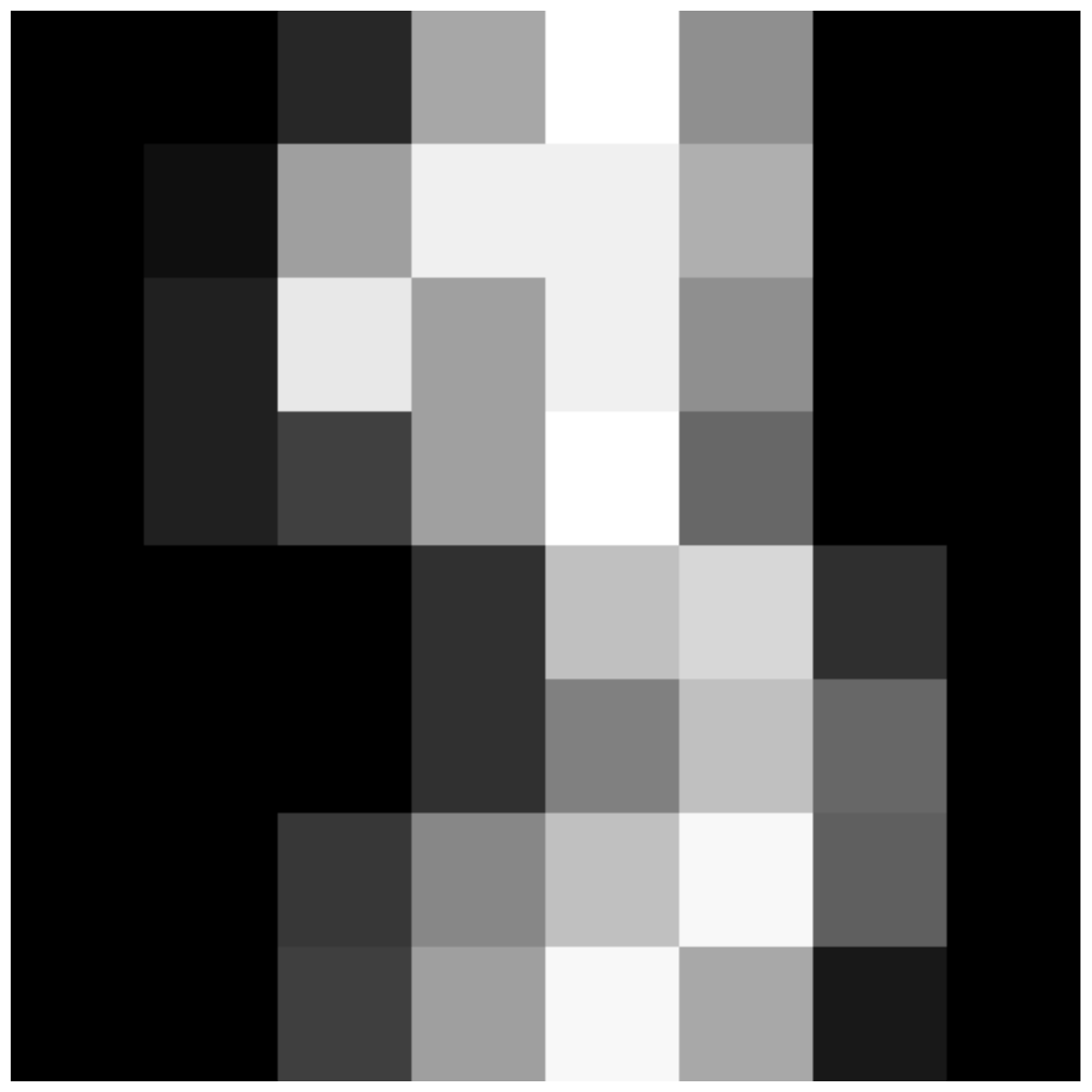}} &
\imageincS{{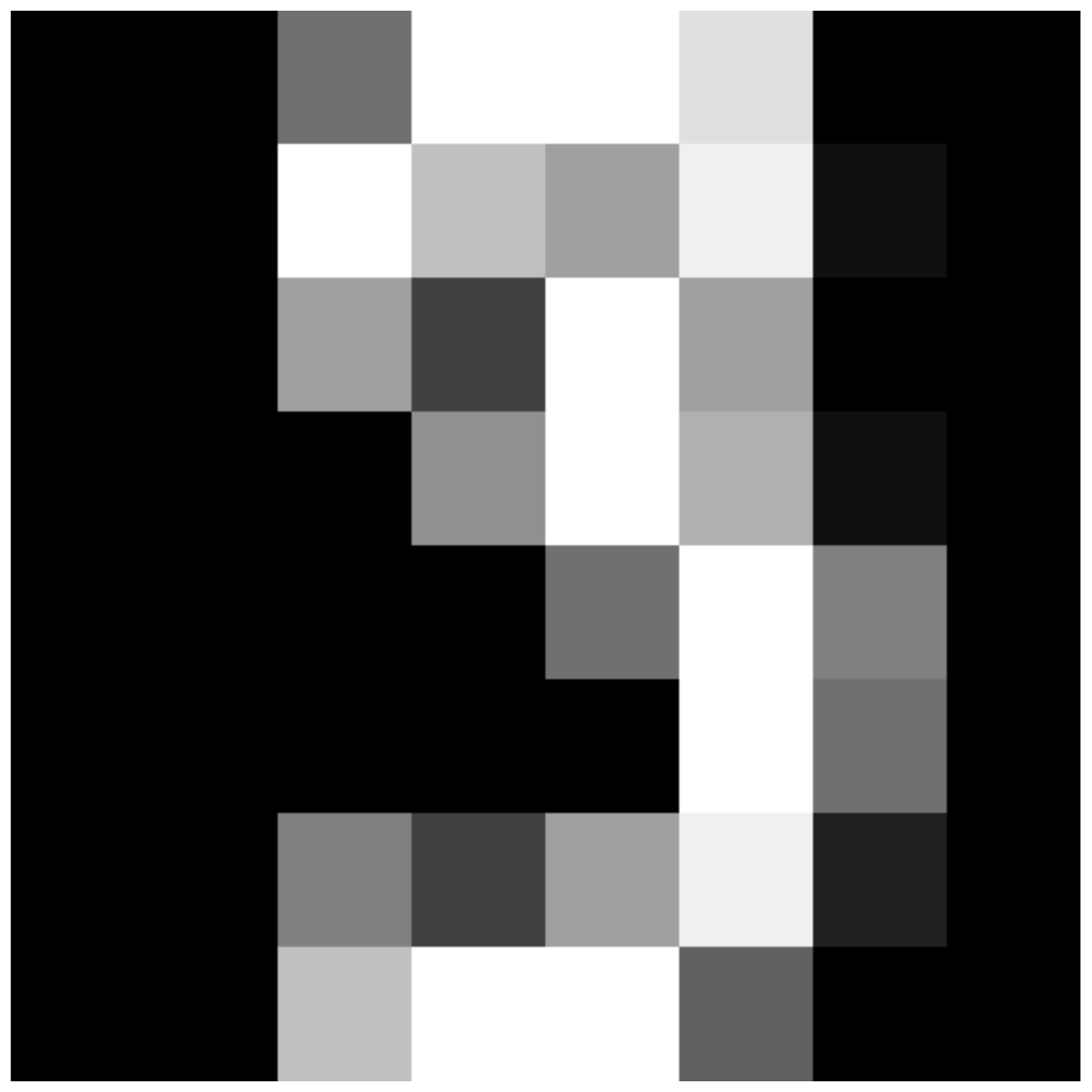}} &
\imageincS{{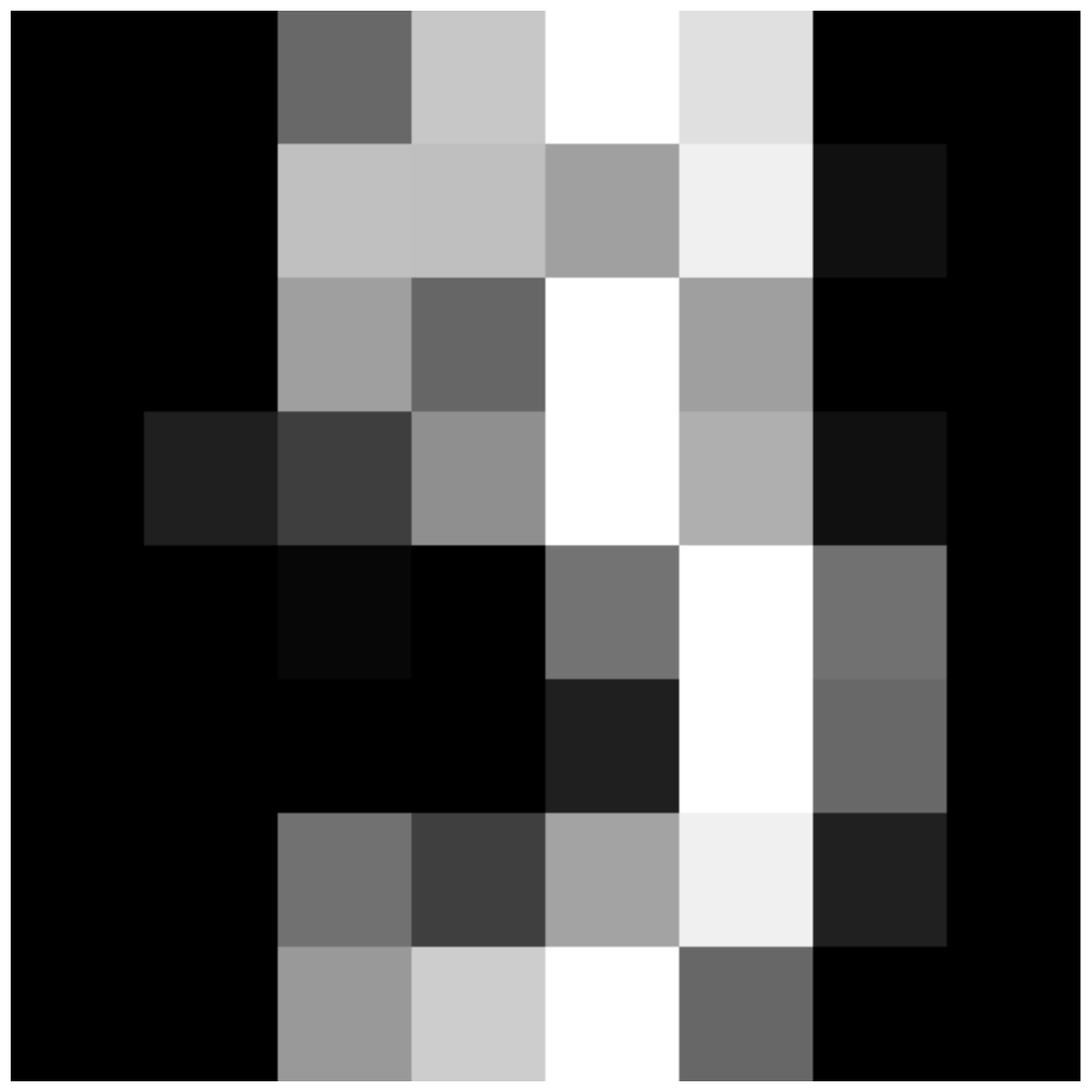}} &
\imageincS{{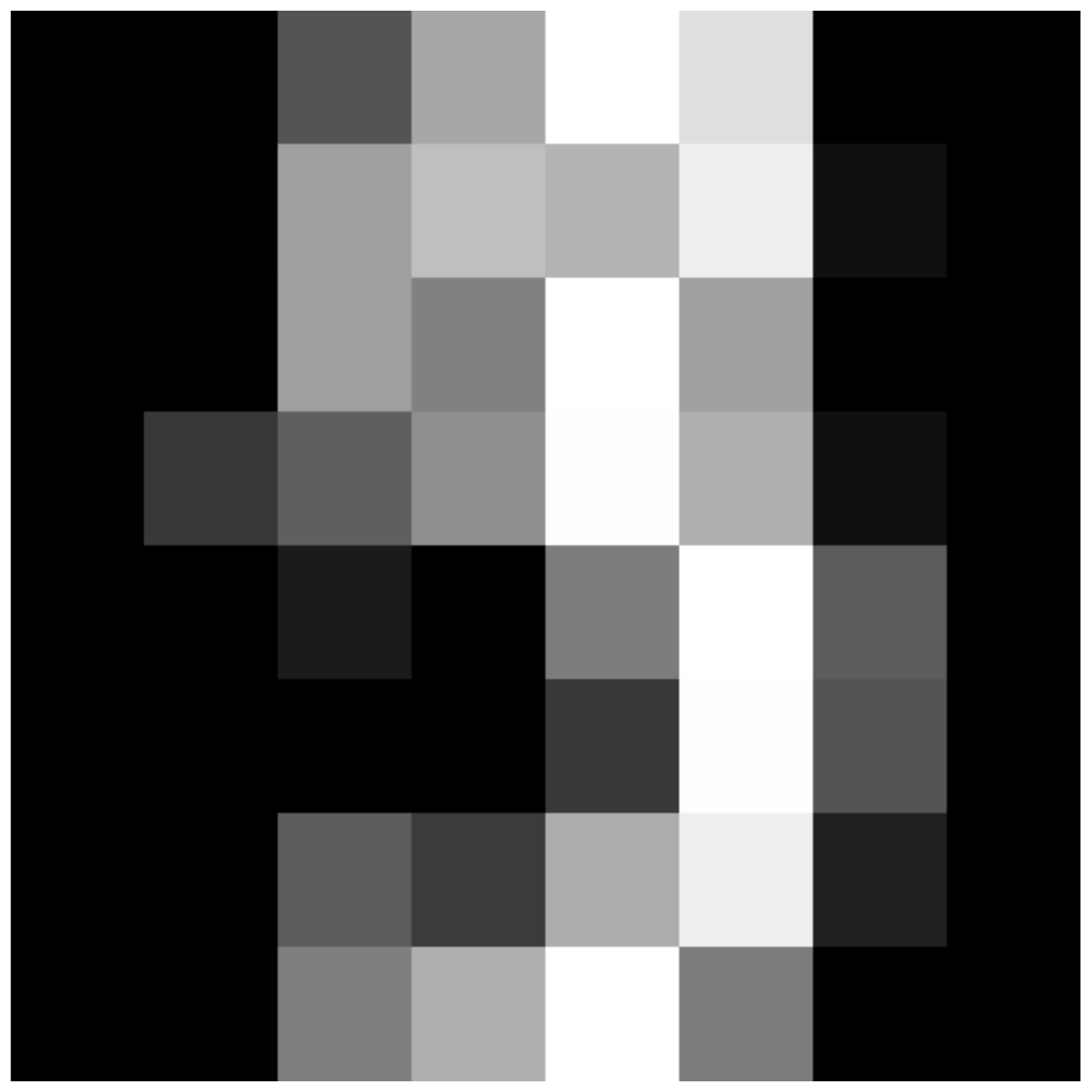}} &
\imageincS{{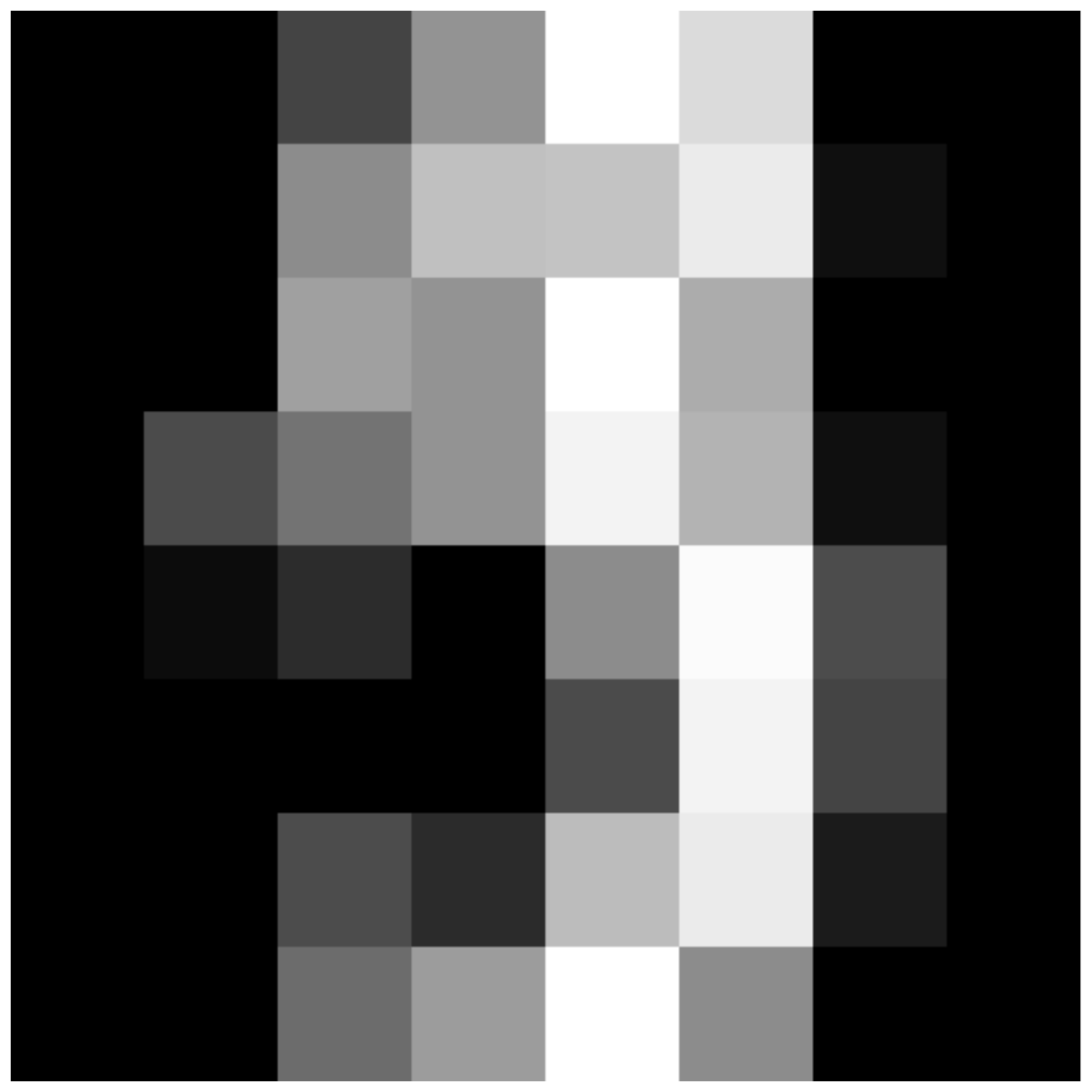}} &
\imageincS{{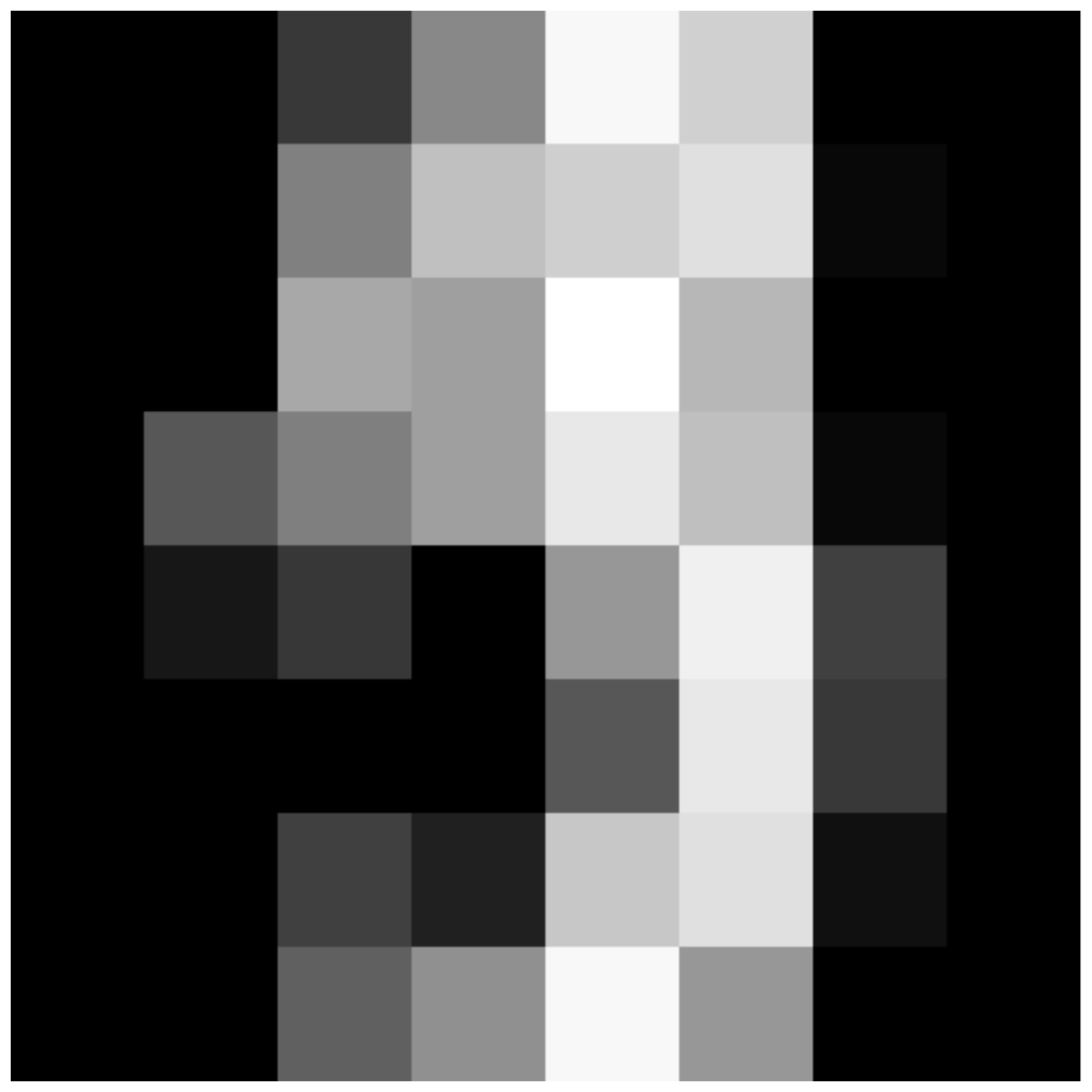}} \\
\imageincS{{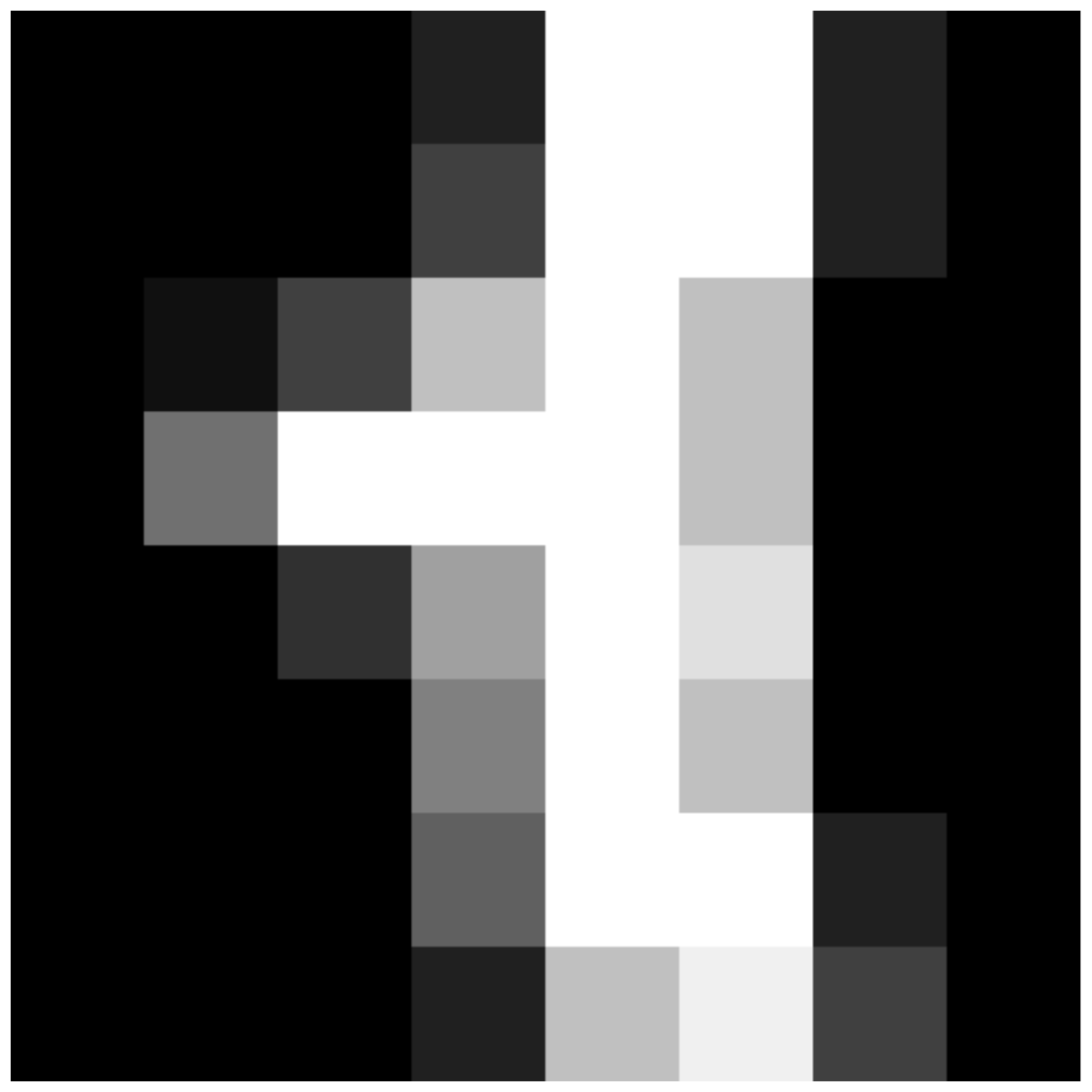}} &
\imageincS{{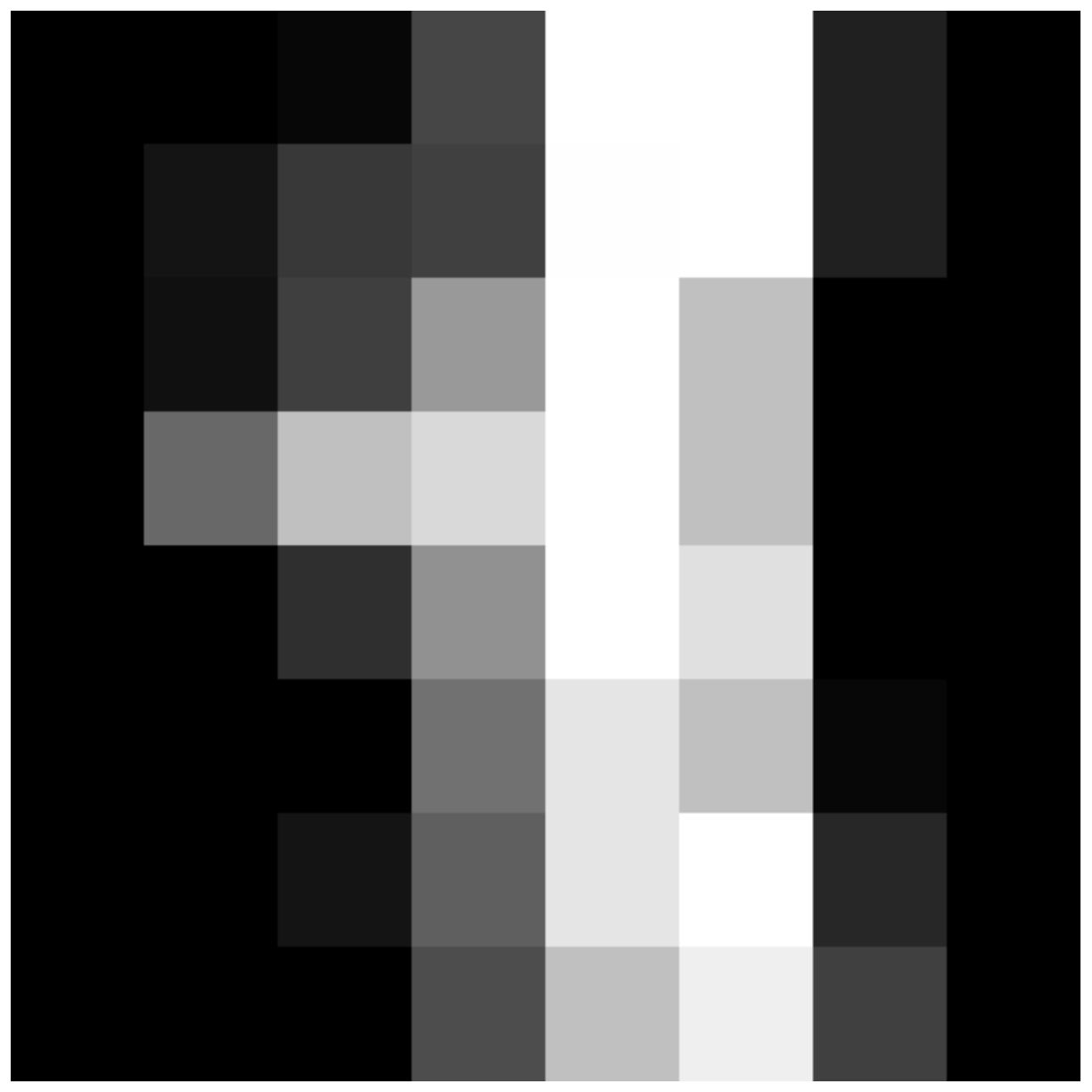}} &
\imageincS{{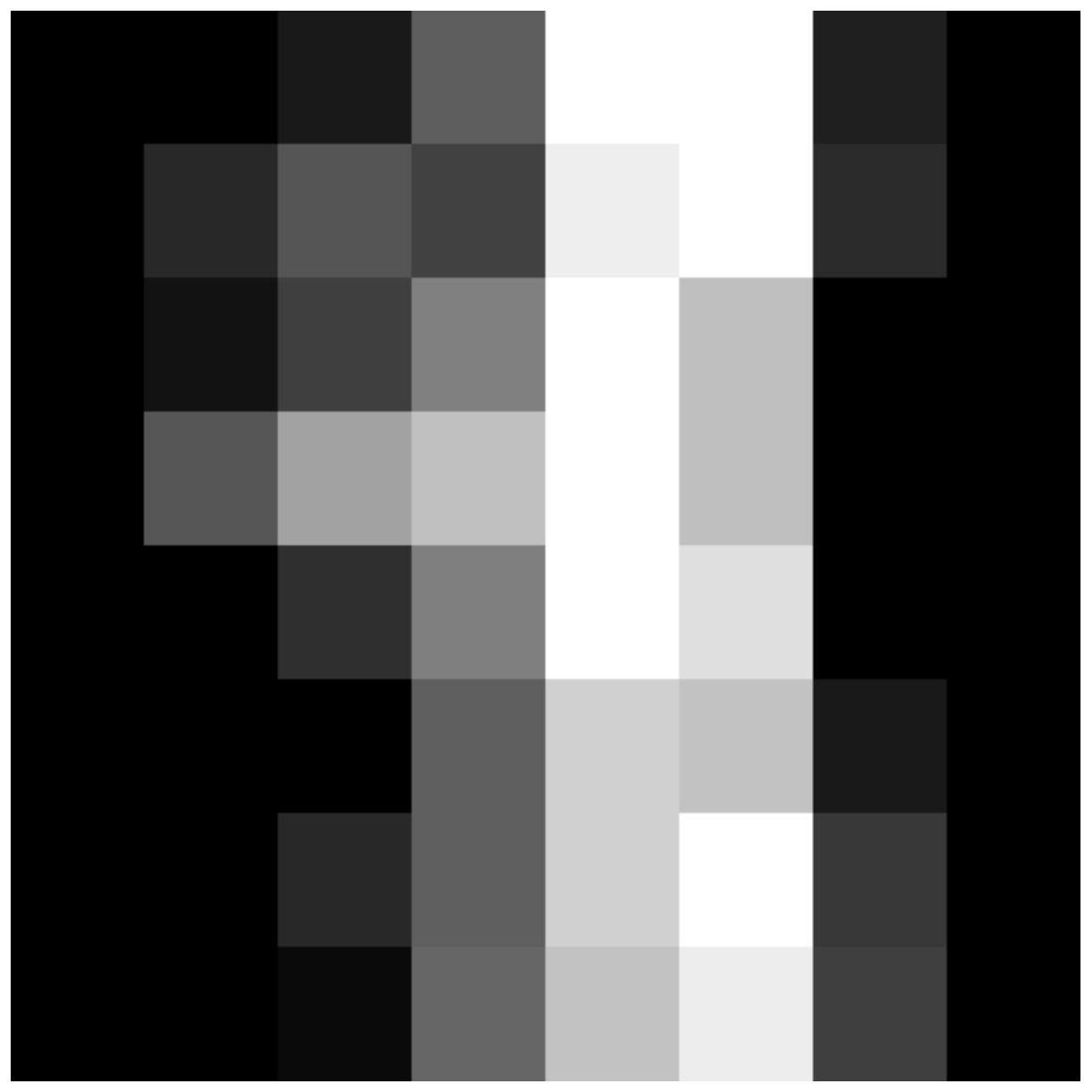}} &
\imageincS{{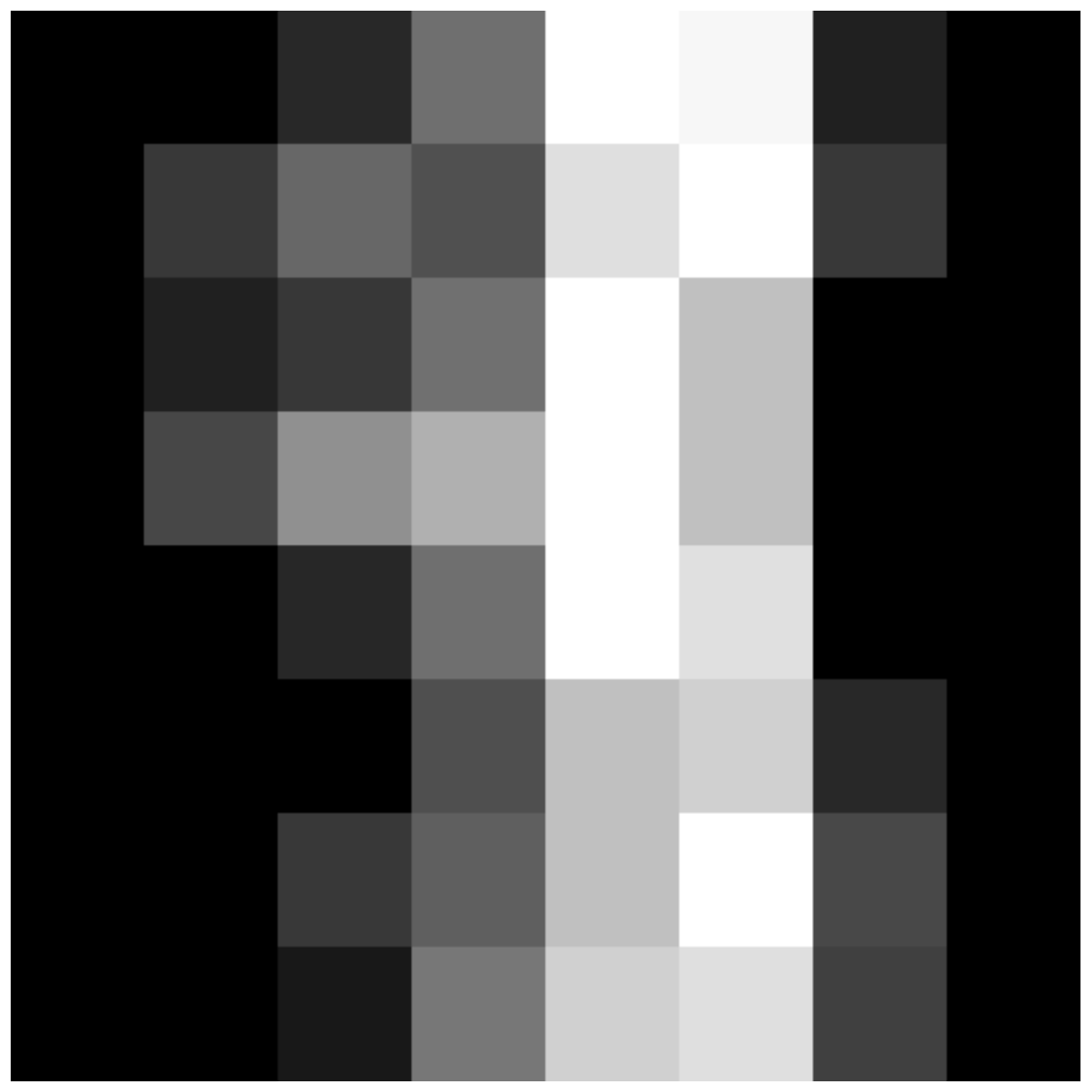}} &
\imageincS{{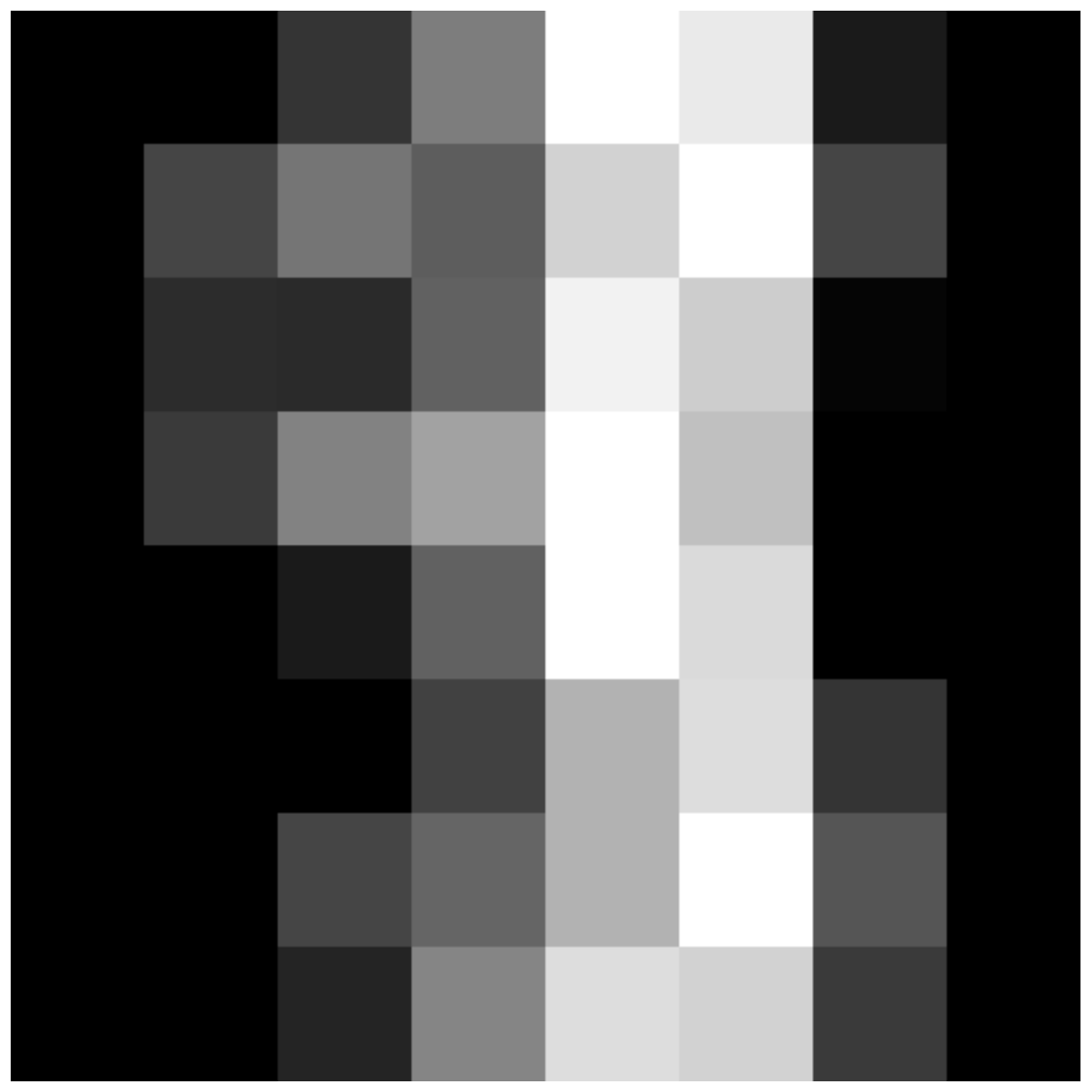}} &
\imageincS{{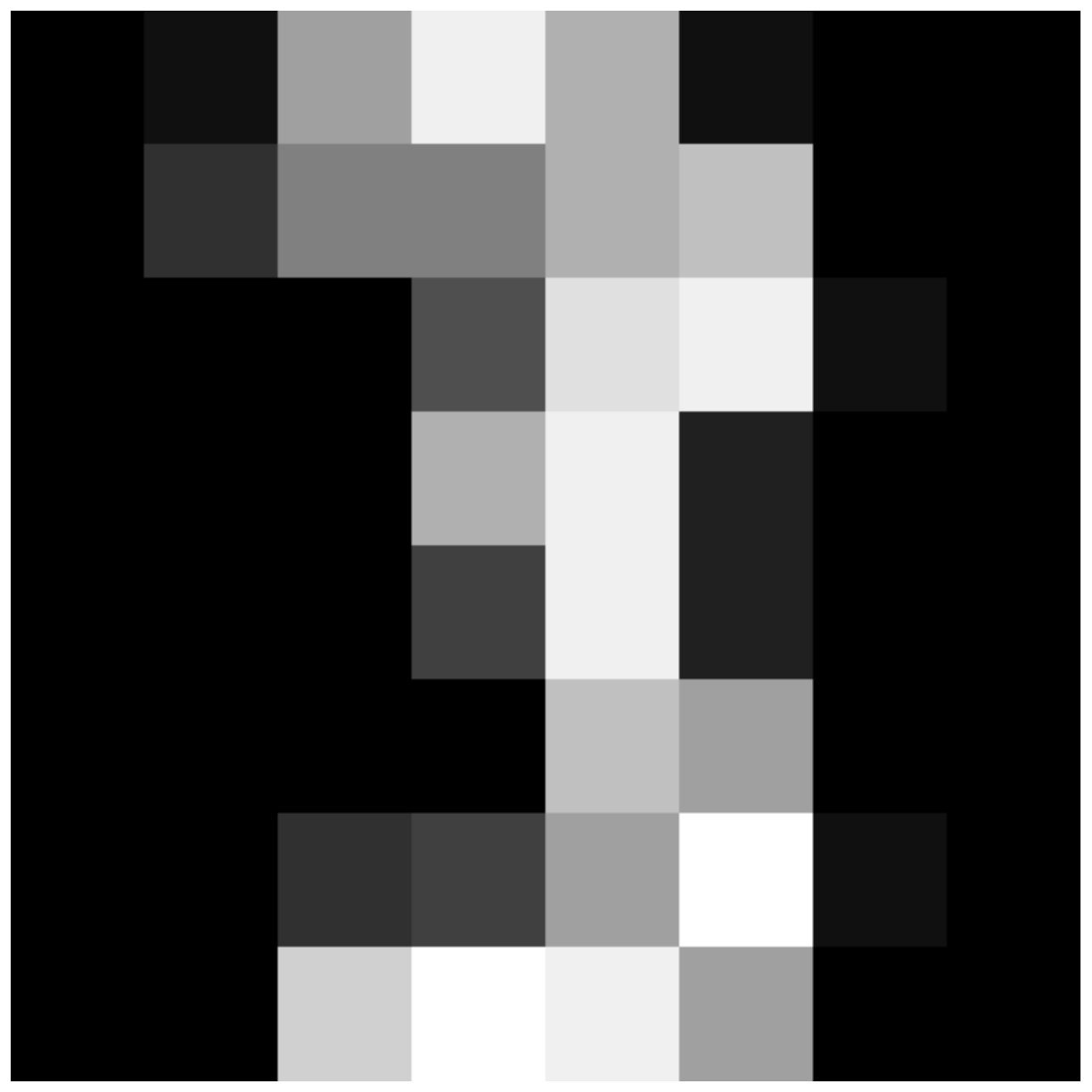}} &
\imageincS{{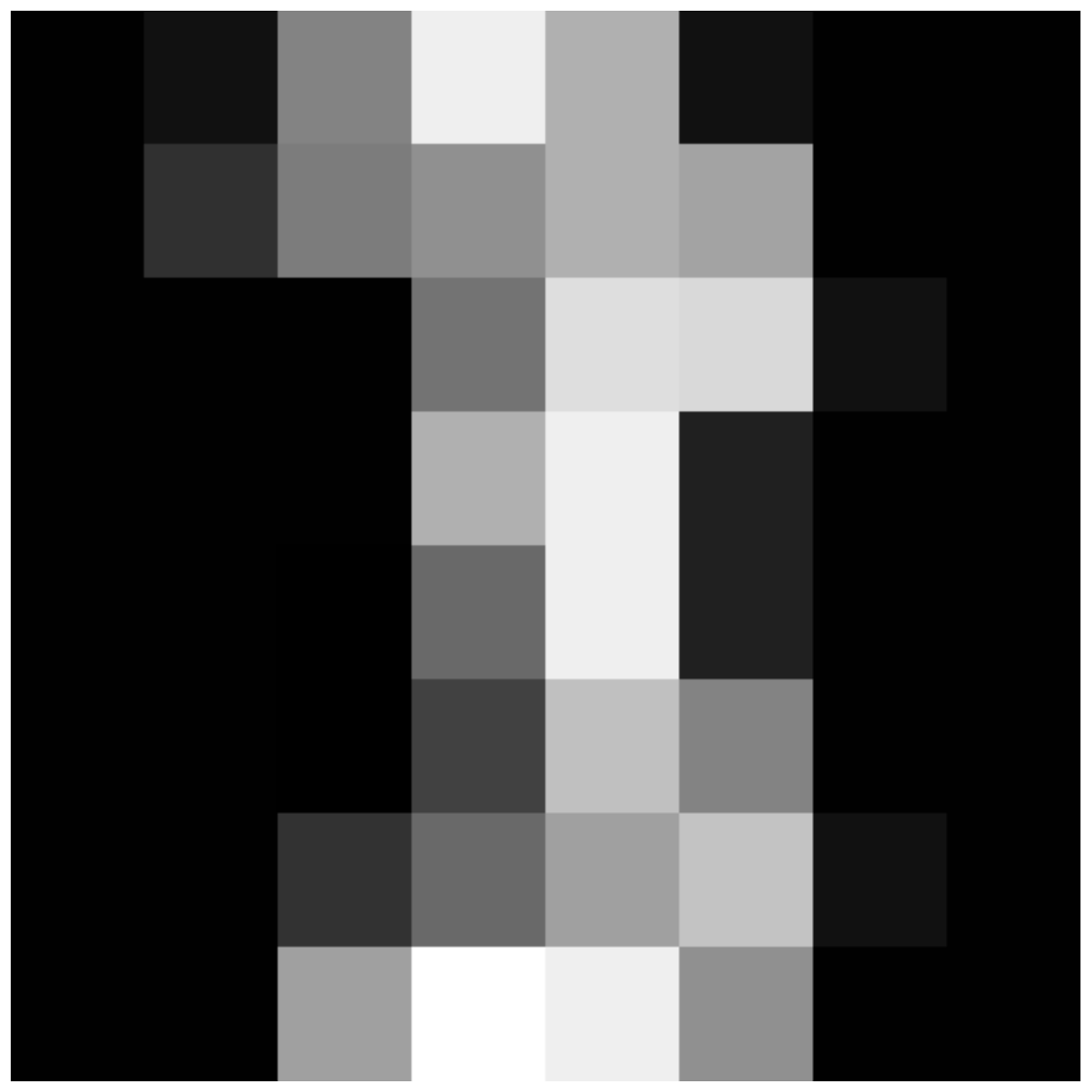}} &
\imageincS{{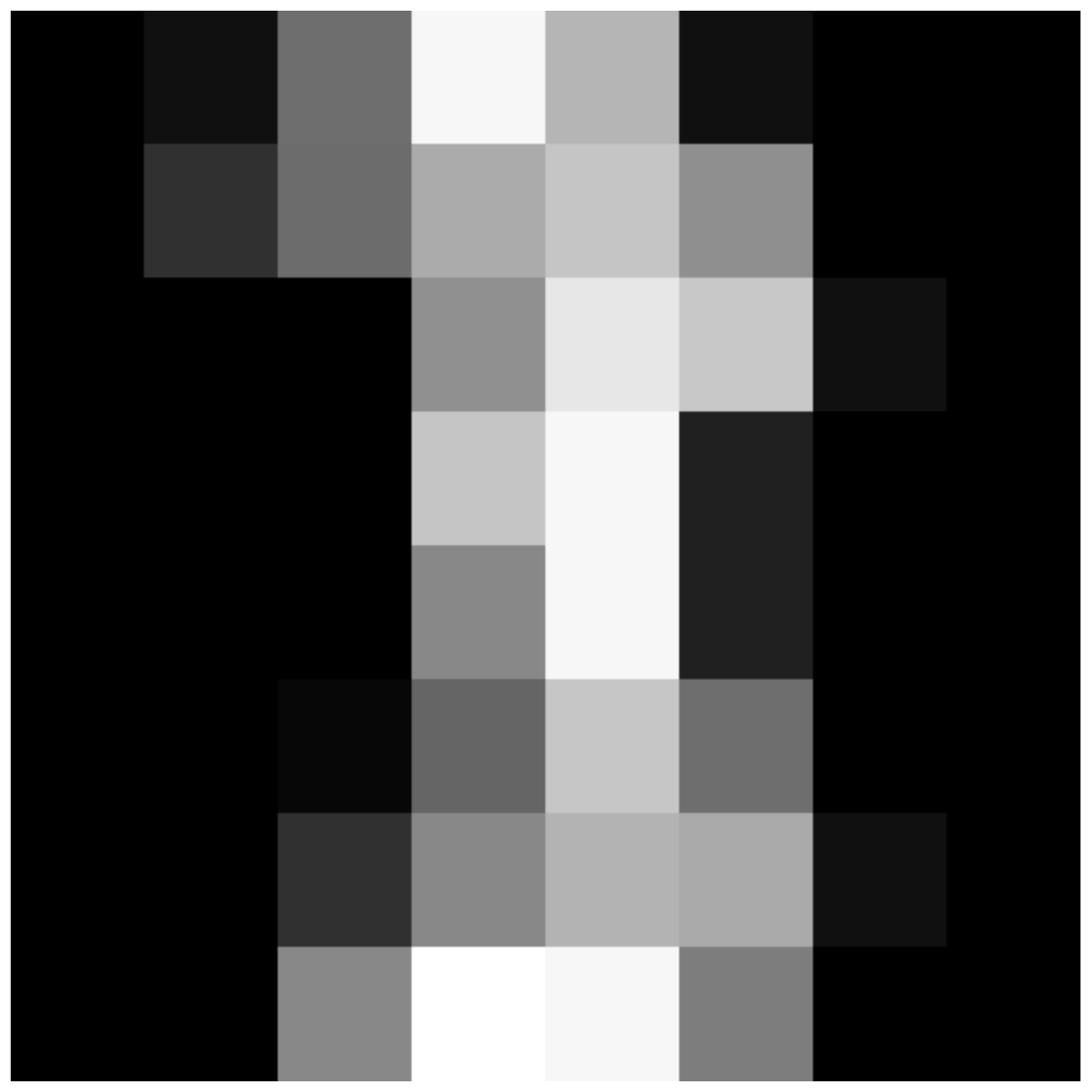}} &
\imageincS{{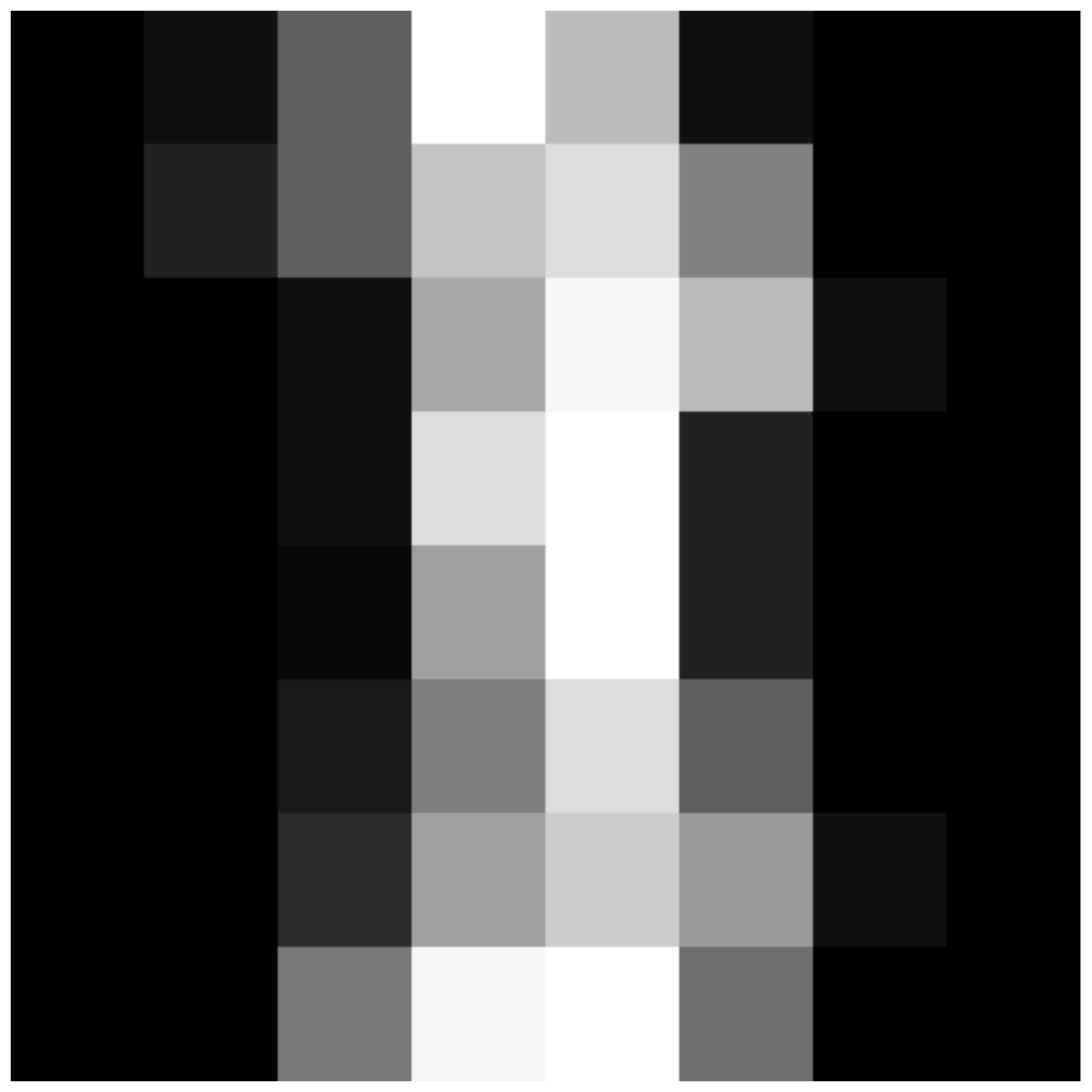}} &
\imageincS{{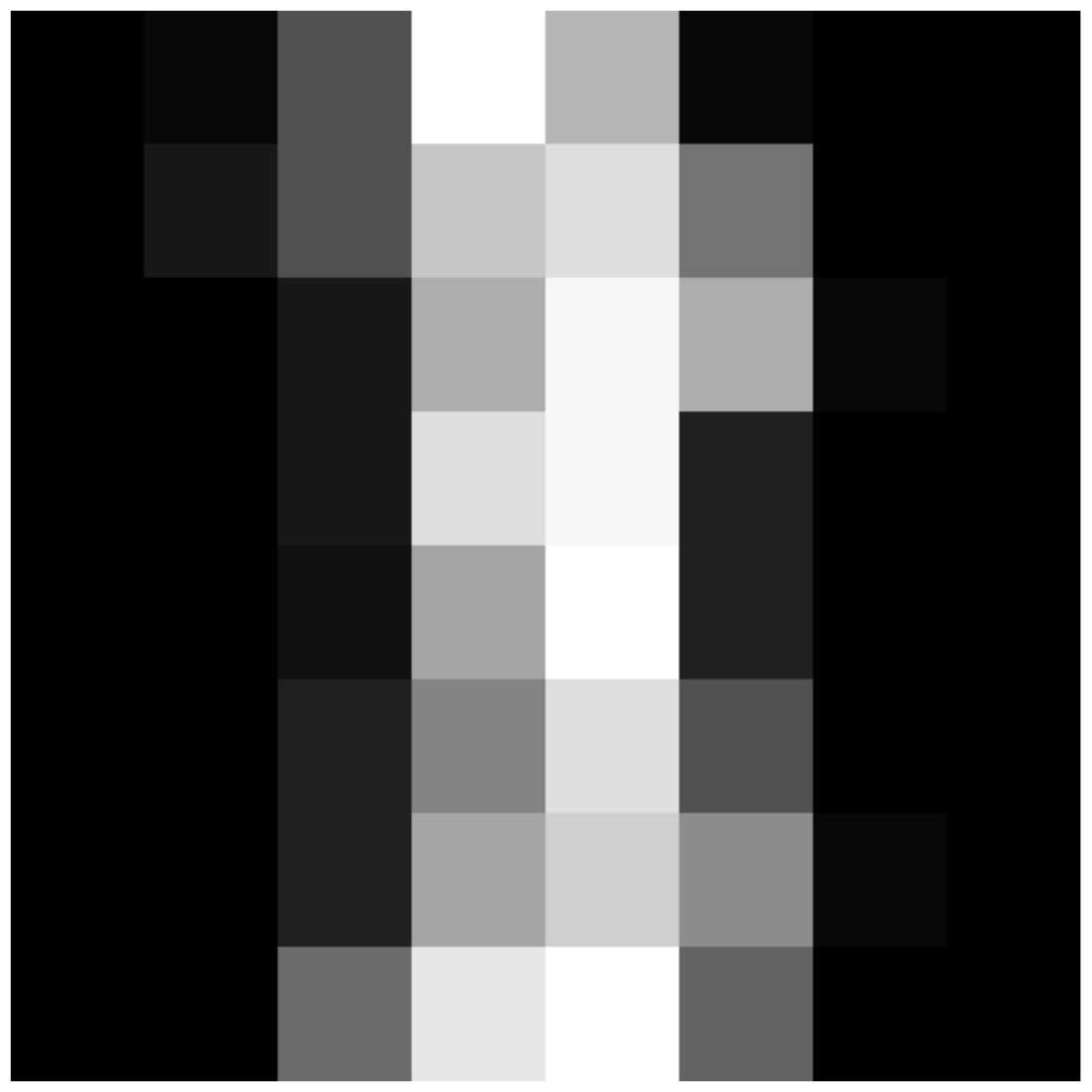}} \\
\imageincS{{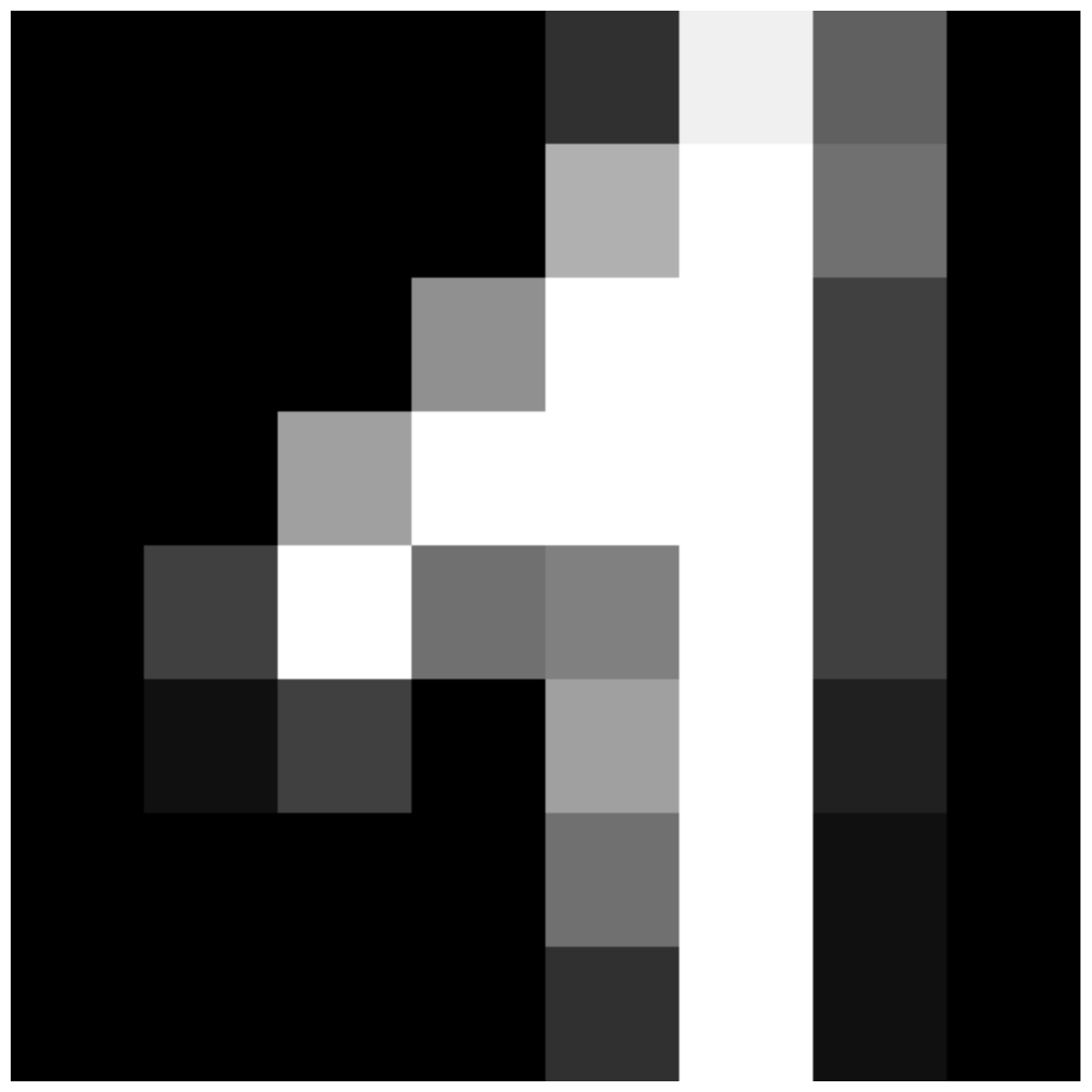}} &
\imageincS{{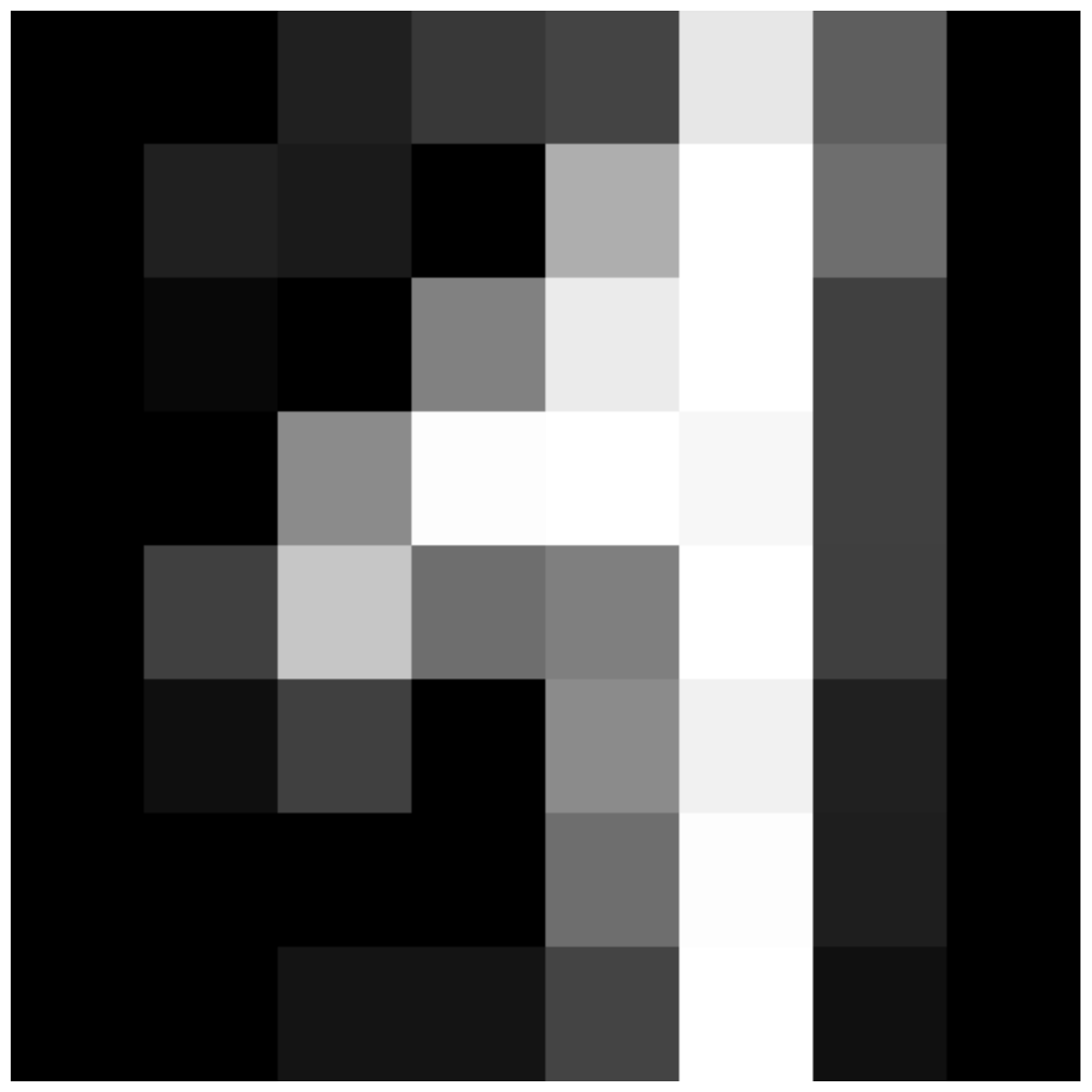}} &
\imageincS{{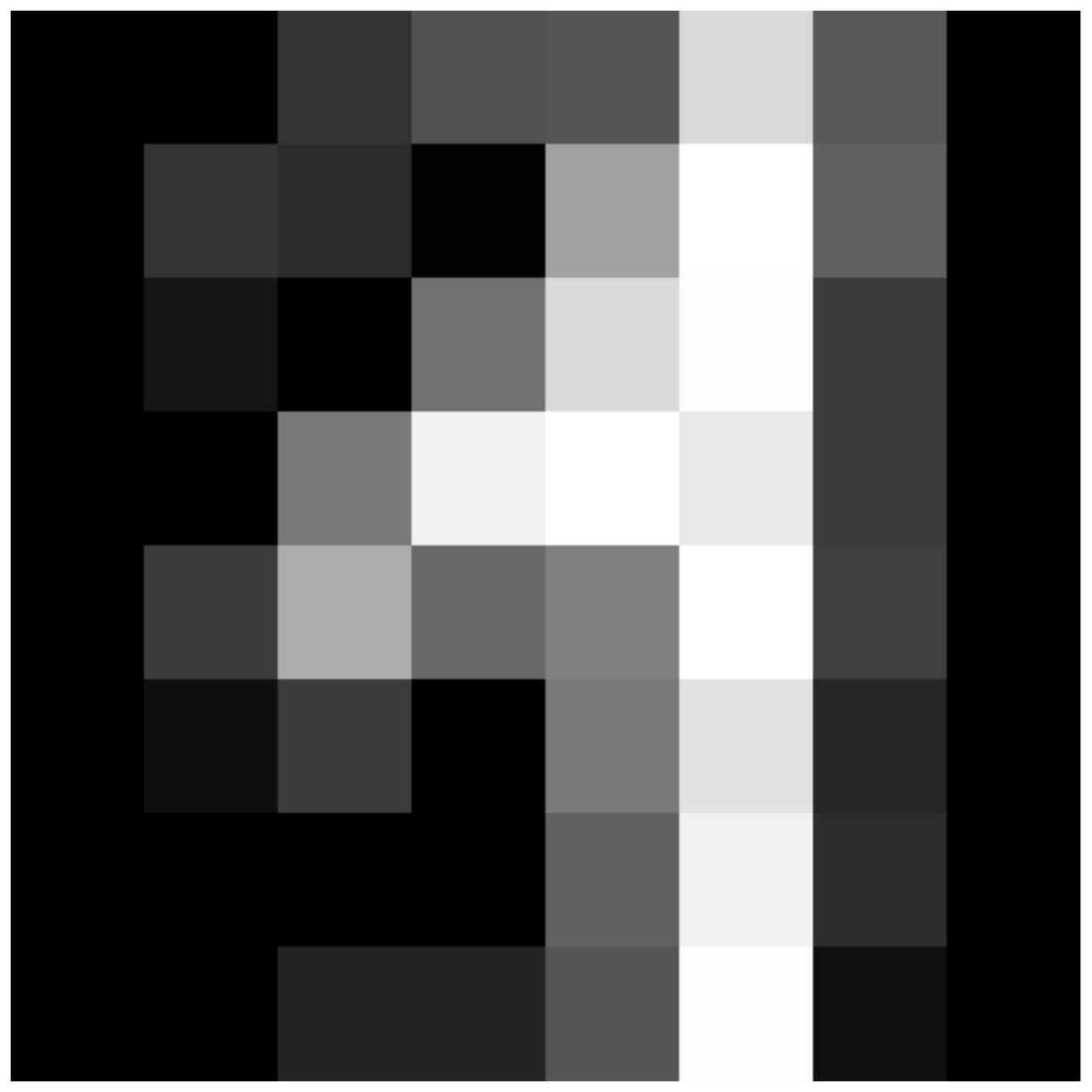}} &
\imageincS{{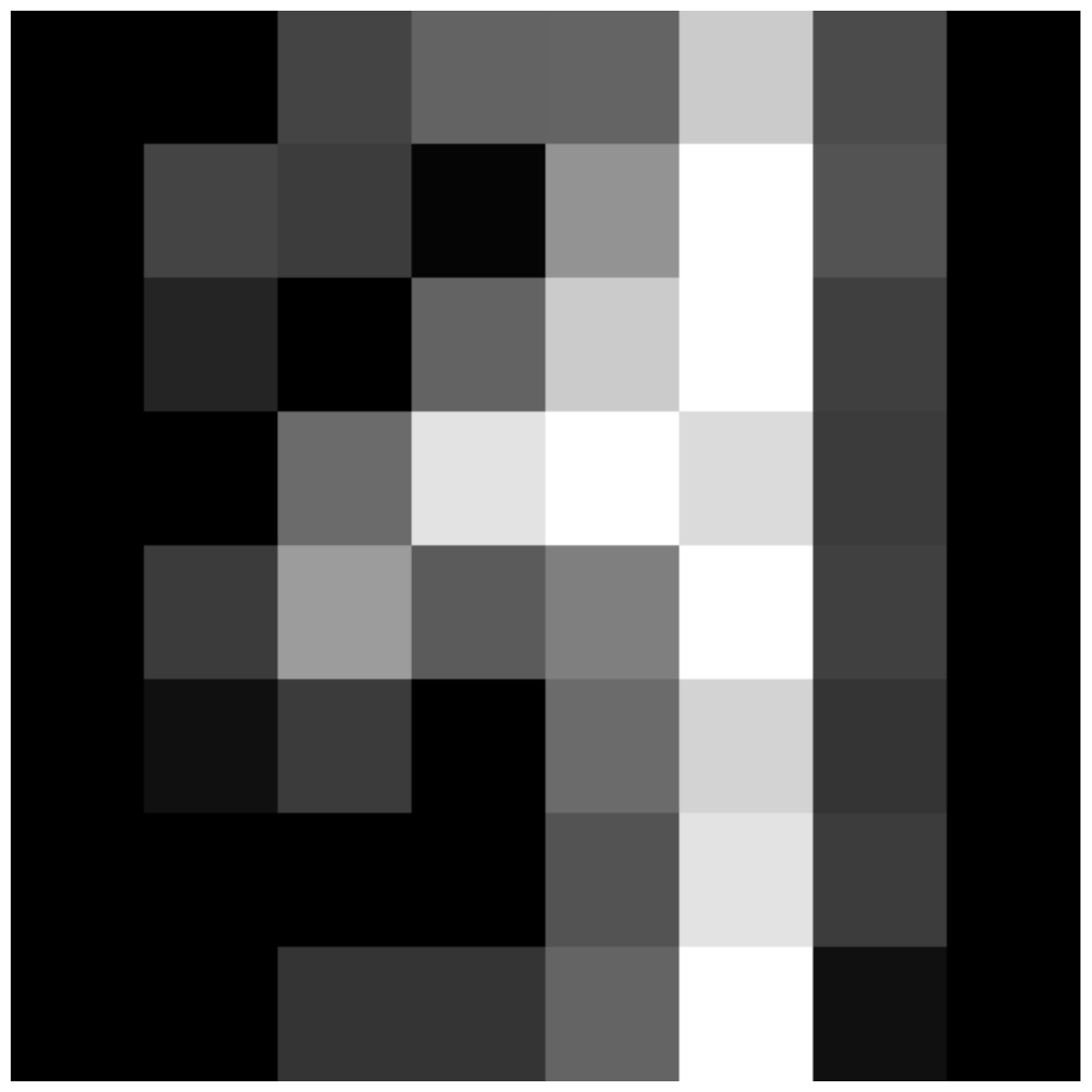}} &
\imageincS{{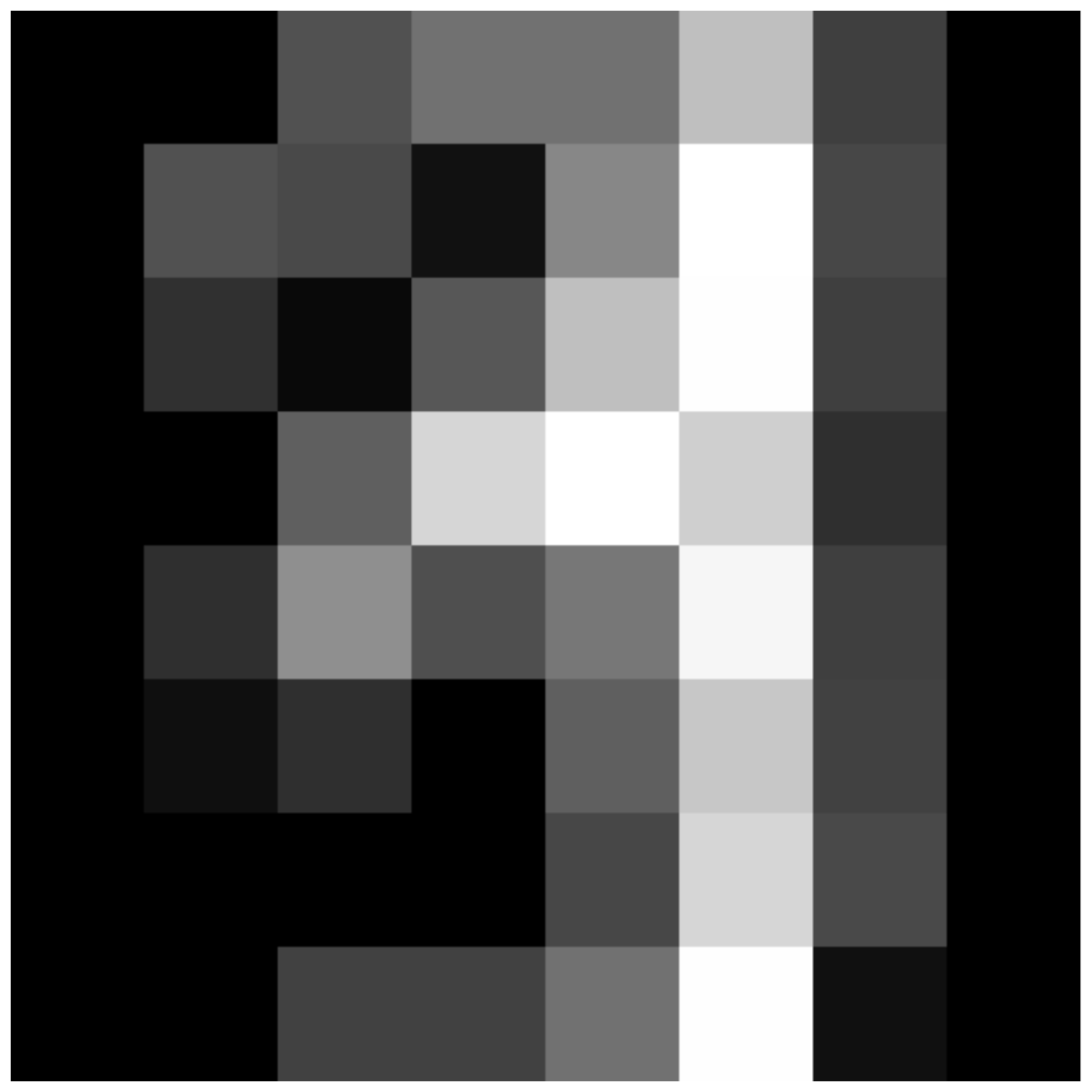}} &
\imageincS{{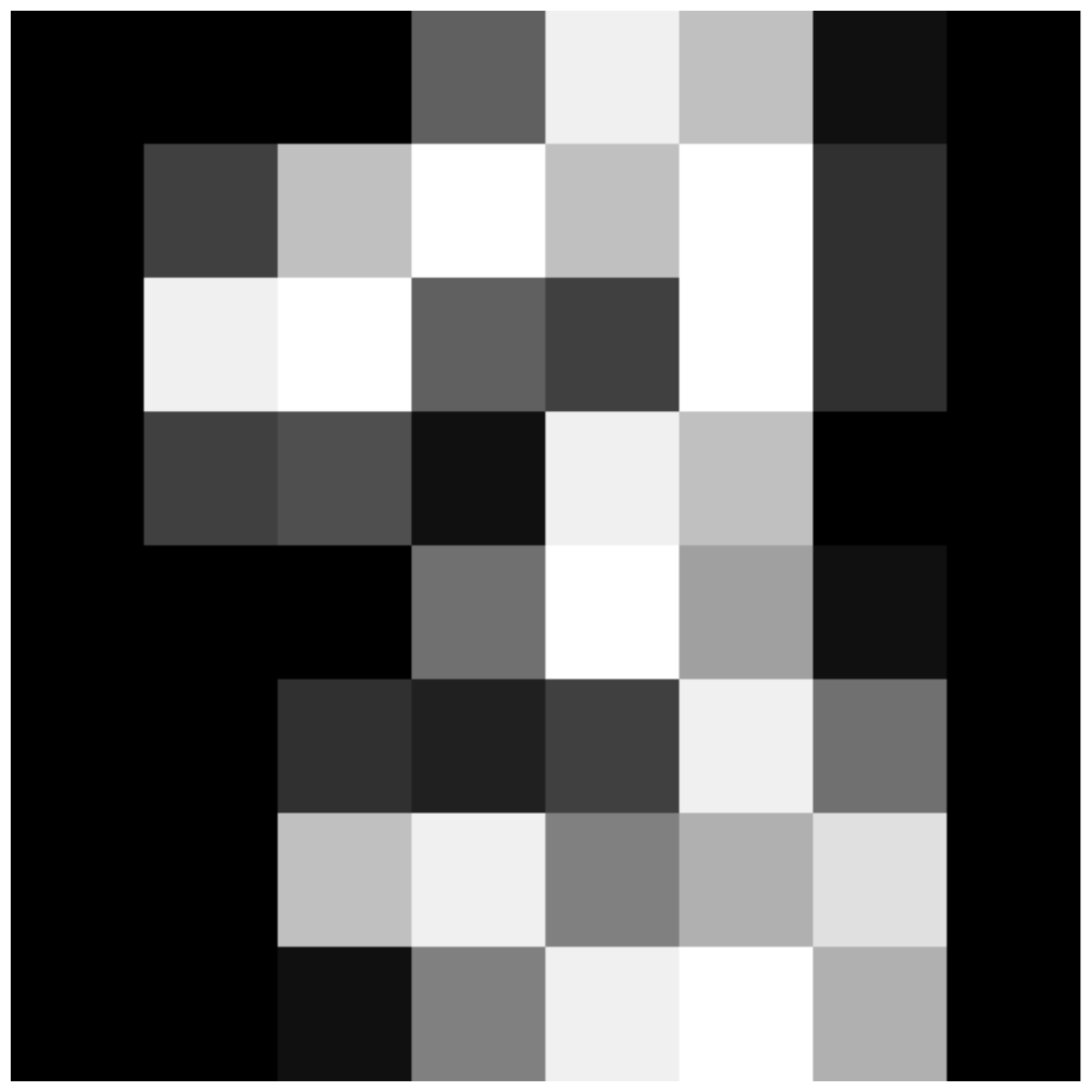}} &
\imageincS{{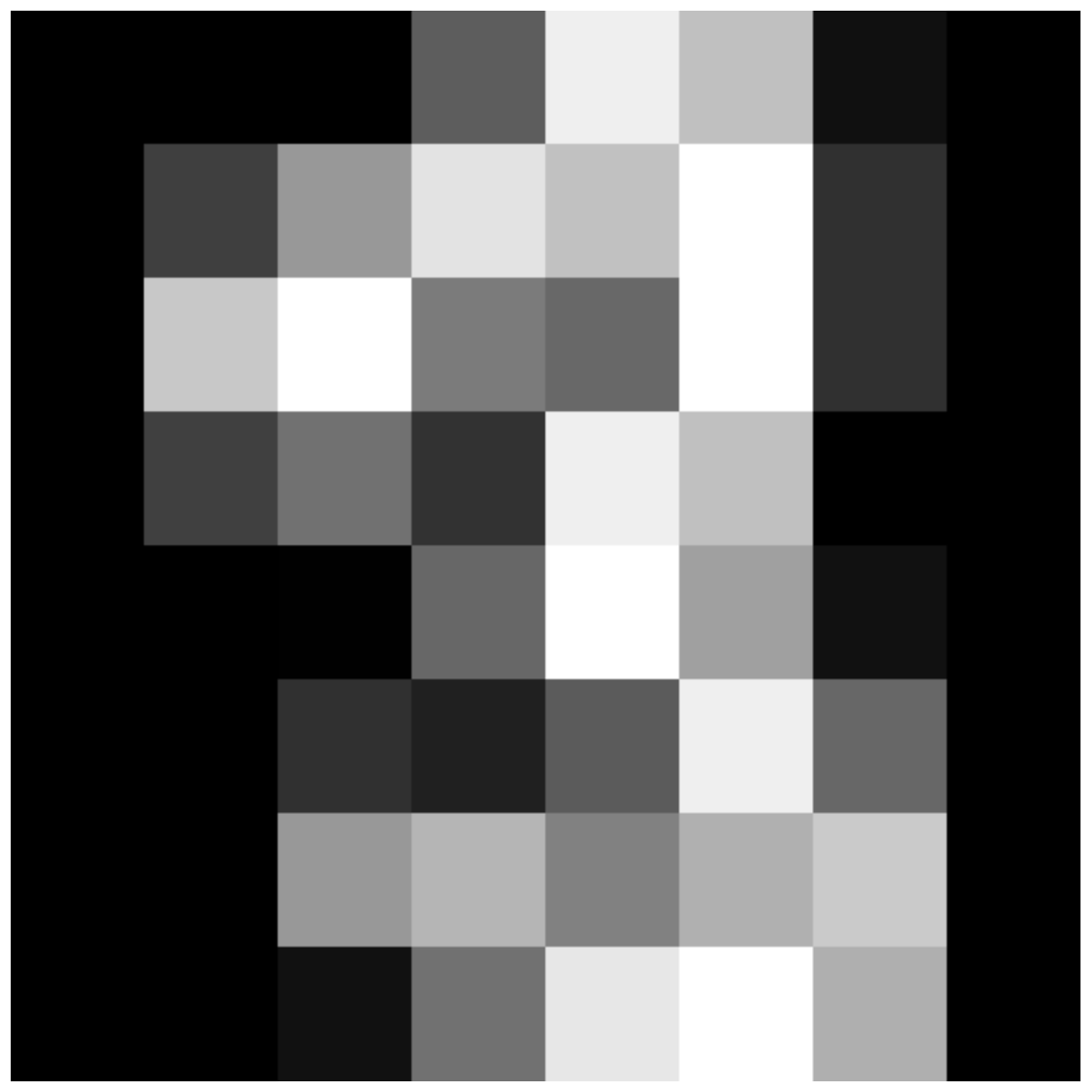}} &
\imageincS{{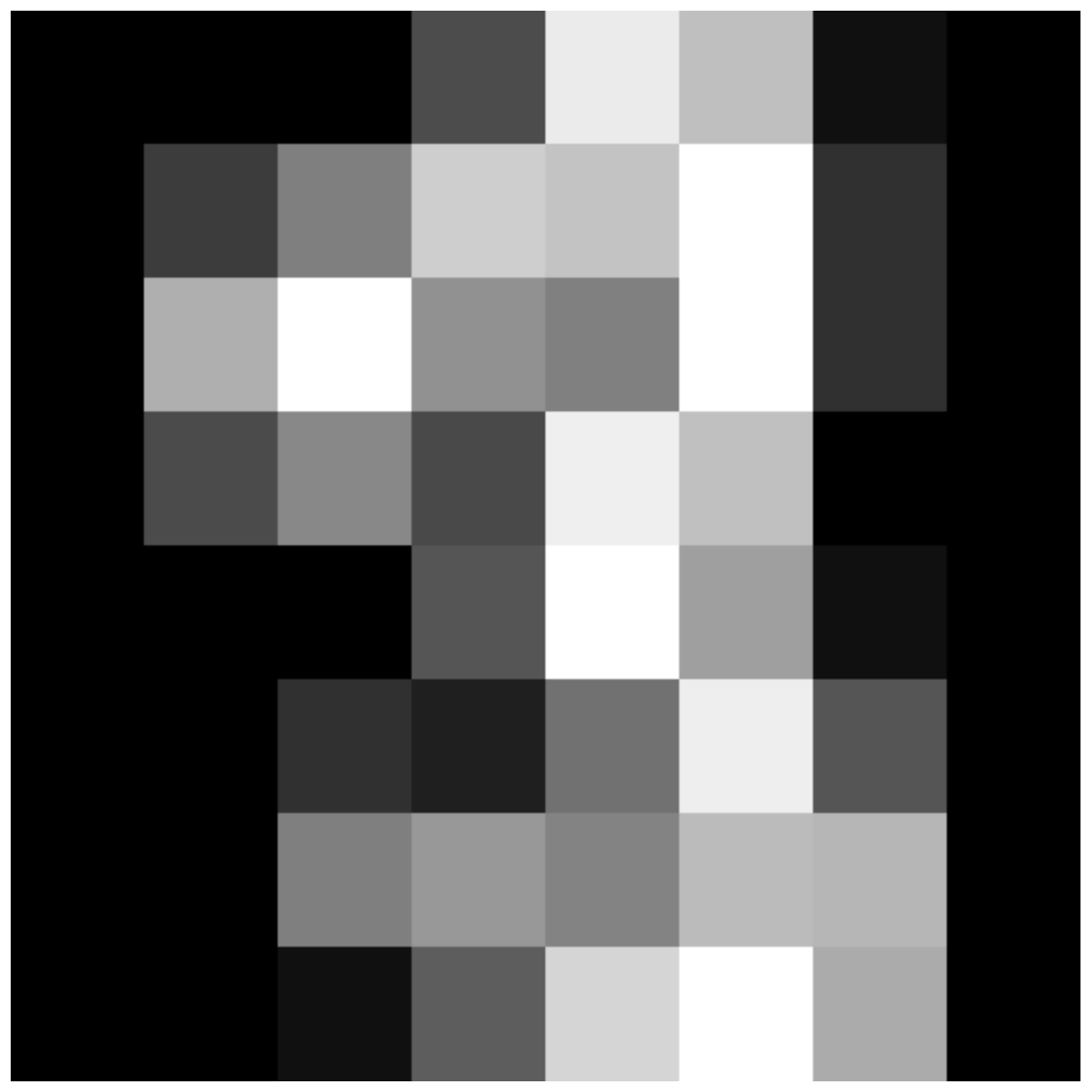}} &
\imageincS{{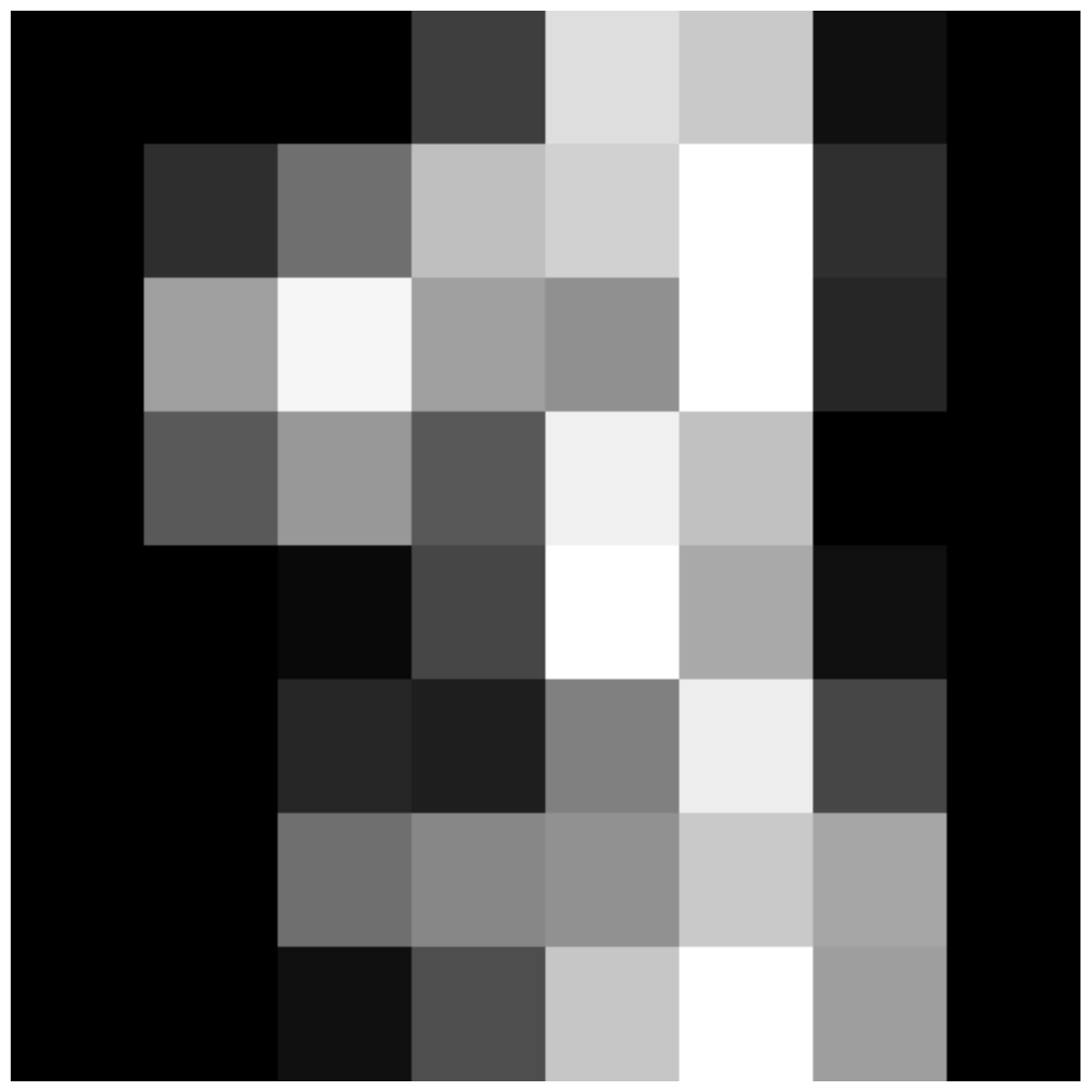}} &
\imageincS{{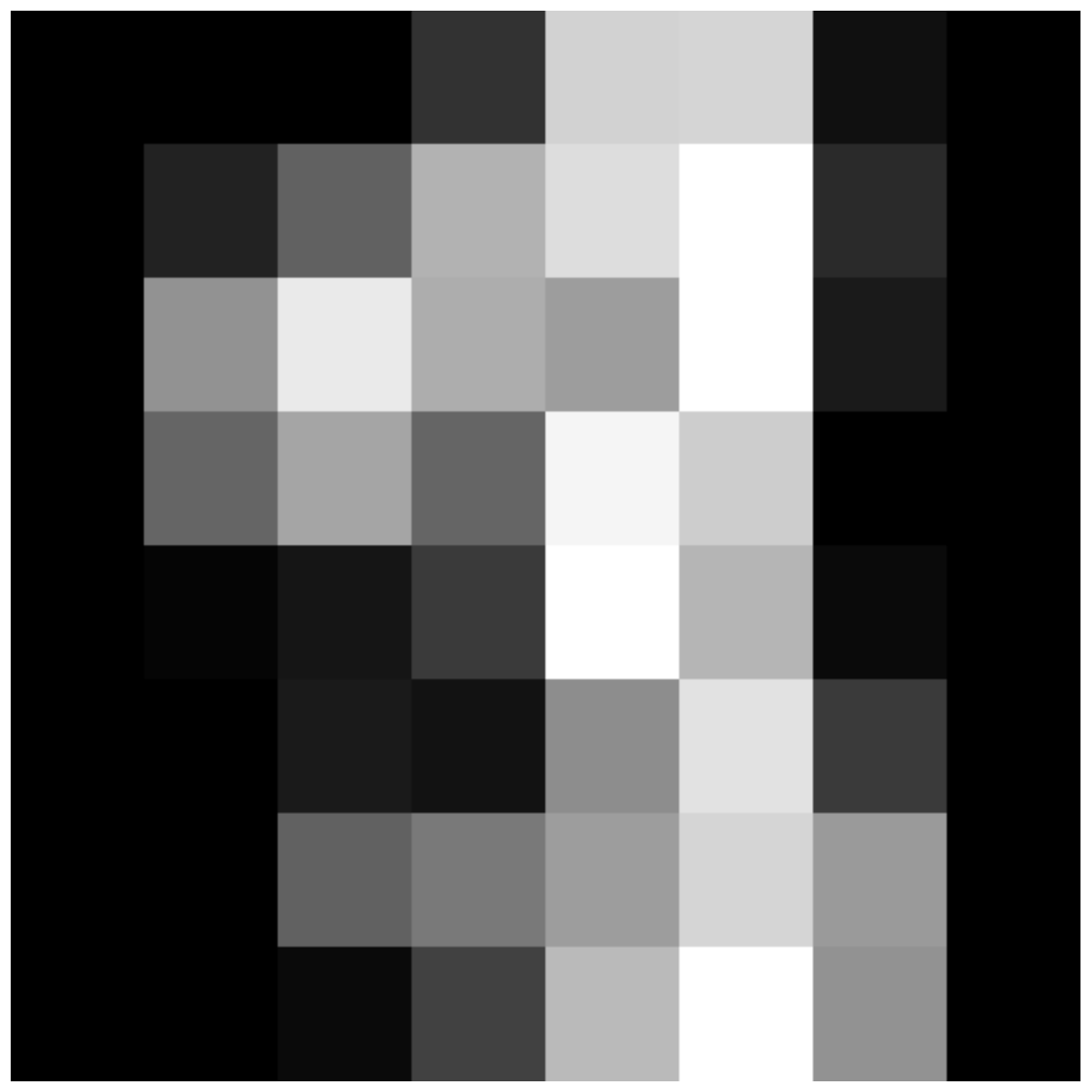}} \\ 
\end{tabular}
\end{tabular}
}
\end{center}
\caption{
\label{fig:toy} \textbf{Left}: results for the 1D toy problem as a function of $\alpha$.
\textit{Left plot:} the expected log loss for the training/testing distribution pairs a/a, a/c, and c/a, where a (respectively c) denotes the adversarial (clean) data distribution.
Hence for a/c we optimised the logistic regression classifier on the adversarial distribution, and computed the log loss on the clean distribution.
\textit{Right plot:} the optimal transport cost $\delta$ (left scale) and the norm of the logistic regression weights $\|\ve w\|_2$ (right scale). \textbf{Right}: sample results of \texttt{digits} as they are transformed by the OT adversary (convention follows Figure \ref{tab-int-e}).
}
\vspace{-0.4cm}
\end{figure*}

\begin{table}[t]
\begin{center}
\begin{tabular}{c|cccc}\hline\hline
$\alpha/d$ & c/c & c/a & a/c & a/a \\\hline
0.15 & 0.03 & 0.11 & \textbf{0.00} & \textbf{0.02} \\
0.30 & 0.03 & 0.25 & \textbf{0.00} & 0.12 \\
0.45 & 0.03 & 0.48 & \textbf{0.01} & 0.55 \\
0.60 & 0.03 & 0.74 & 0.20 & 0.96 \\\hline \hline
\end{tabular}
\vspace{-0.2cm}
\end{center}
\caption{
\label{tab:toy2} 
log loss USPS results. $\alpha/d$ is the strength of the adversary. The convention $\{\mbox{a,c}\} / \{\mbox{a,c}\}$ follows Figure \ref{fig:toy}. \textbf{Bold faces} denote results better than the c/c baseline.
}
\vspace{-0.4cm}
\end{table}

We have performed toy experiments to demonstrate our new setting.
Our objective is not to investigate the competition with respect to the wealth of results that have been recently published in the field, but rather to touch upon the interest that such a novel setting might have for further experimental investigations.
Compared to the state-of-the-art, ours is a clear two-stage setting
where we first compute the adversaries assuming relevant knowledge of
the learner (in our case, we rely on Theorem \ref{thmMEF} and
therefore assume that the adversary knows at least the cost $c$, see
below), and then we learn based on an adversarially transformed set of
examples. This process has the advantage over the direct minimization
of \eqref{eqNOVO} that it extracts the computation of the adversarial
examples from the training loop: we can generate \textit{once} the adversarial examples, then
store them and / or share / reuse them to robustly train various models
(recall that under a general Lipschitz assumptions on classifiers,
such examples can fit the adversarial training of different kinds of
models, see Theorem \ref{thOTA}). This
process is also reminiscent of the training process for invariant
support vector machines \citep{dsTI} and can also be viewed as a
particular form of vicinal risk minimization \citep{cwbvVR}. We have
performed two experiments: a 1D experiment involving a particular
Mixup adversary and a USPS experiment involving a closer proxy of the
optimal transport compression that we call Monge adversary.

\noindent $\triangleright$ \textit{1D experiment, mixup adversary}. Our example involves the unit interval $\mathcal X = [0,1]$ with $P(x)\propto\exp(-(x-0.2)^2/0.1^2)$ and $N(x)\propto\exp(-(x-0.6)^2/0.2^2)$. We let $\mathcal A$ contain a single deterministic mapping parametrised by $\alpha$ as $\mathcal{A} \defeq \{a(x)\defeq (1-\alpha) x + \alpha \E_{(\X,\Y)\sim D} \X\}$. Notice that this adversary is just the $(1-\alpha)$-mixup to the unconditional mean, following Section \ref{sec-wts}.
We further let $\mathcal H$ be the space of linear functions $h(x) = \ve w \cdot (x,1)^\top$, $\ve w \in \mathbb R^2$,
which is the RKHS with linear kernel $\kappa(\ve x, \ve y)=\ve x \cdot \ve y$ (assuming that $\ve x$ and $\ve y$ include the constant 1), and $\| h\|_\mathcal H = \| \ve w \|_2$.
The transport cost function of interest is 
% AKM: edit
$c(\ve x, \ve y)=\|\ve x-\ve y\|_2$.
%$c(\ve x, \ve y)=\sqrt{\kappa(\ve x, \ve x)+\kappa(\ve y, \ve y)-2\kappa(\ve x, \ve y)}=\|\ve x-\ve y\|_2$.
We discretize $\mathcal X$ to simplify the computation of the OT cost. Results are summarized in Figure \ref{fig:toy} (and \supplement). We theoretically achieve loss $\ell_0$ as $\alpha \rightarrow 1$.
There are several interesting observations from Figure \ref{fig:toy}: first, the mixup adversary indeed works like a Monge efficient adversary: by tuning $\alpha$, we can achieve any desired level of Monge efficiency.
The left plot completes in this simple case observations of \citet{tsetmRM,zcdlMB}: the \textit{worst} result is consistently obtained for training on clean data and testing on adversarial data, which indicates that our adversaries may be useful to get robustness using adversarial training.

\noindent $\triangleright$ \textit{USPS digits, Monge adversary}. We have
picked 100 examples of each of the "\texttt{1}" and "\texttt{3}"
classes of the 8$\times$8 pixel greyscale USPS handwritten digit
dataset. The set of \textit{Monge} adversaries is $\mathcal A \doteq \{a:\mathbb
R^{64} \mapsto \mathbb R^{64} \, | \, \left \| a(x)-x\right\|_1 \leq
\alpha \}$, in which, under the $L_1$ budget constraint, we optimize
the Wasserstein distance $W_2^2$ between the empirical class
marginals. We achieve this by combining a generic gradient-free
optimiser with a linear program solver\footnote{Code available upon
  request to CW}. We learn using logistic
regression. We demonstrate three strengths of adversary --- namely
$\alpha/d=0.15, 0.30, 0.45, 0.6$ where $d$ is $L_1$ distance between
the (clean) class conditional means. Sample transformations as
obtained by the Monge adversary are displayed in Figure \ref{tab:toy2}
(more in \supplement), and Table \ref{tab:toy2} provides log loss
values for different training / test schemes, following the scheme of
the 1D data. It clearly emerges two facts: (i) as the budget increases, the Monge adversary smoothly transforms digits in credible adversarial examples, and (ii), as previously observed, training over a tight budget adversary tends to increase generalization abilities \cite{tsetmRM,zcdlMB}.
% !TEX root=../nips18-adversarial-mf-1.tex

\section{Conclusion}
\label{sec:conclusion}

It has been observed over the past years that classifiers can be
extremely sensitive to changes in inputs that would be imperceptible to
humans. How such tightly limited \textit{resource}-constrained changes
can affect and be
so damaging to machine learning and how to find a cure has been
growing as a very intensive area of research. There is so far little
understanding on the formal side and many experimental approaches
would rely on adversarial data that, in some way, shrinks the gap 
between classes in a
controlled way.\\
In this paper, we studied the intuition that such a process can
indeed be beneficial for adversarial training. Our answer involves a simple, sufficient (and sometimes loss-independent) property for any given class of adversaries to be detrimental to learning.
This property involves a measure of ``harmfulness'',
which relates to (and generalizes)
integral
probability metrics and the maximum mean discrepancy.
We presented
a sufficient condition for this sufficient property to hold for
Lipschitz classifiers, which relies on framing it into optimal
transport theory. This brings a general way to formalize how
adversaries can indeed "shrink the gap" between classes with the objective to
be detrimental to learning.
As an example, we delivered a negative boosting result which
shows how weakly contractive adversaries for a RKHS can be combined
to build a maximally detrimental adversary. We also provided justifications that
several experimental approaches to adversarial training involve
proxies for adversaries like the ones we analyze. On the experimental side,
we provided a simple toy assessment of the ways one can compute and
then use such adversaries in a two-stage process.\\
Our experimental results, even when carried out
on a toy domain, bring additional reasons to consider such
adversaries, this time
from a generalization standpoint: our results might indeed indicate that they could at least be useful to
gain additional robustness in generalization.
% !TEX root=../nips18-adversarial-mf-1.tex

\section*{Acknowledgments}
\label{sec:ackno}

The authors warmly thank Kamalika Chaudhuri, Giorgio Patrini, Bob
Williamson, Xinhua Zhang for numerous remarks and stimulating discussions around this material.

%\bibliography{references}
%\bibliographystyle{icml2017}
\bibliographystyle{abbrvnat}
\bibliography{references,bibgen}

\clearpage

\section{Appendix}\label{sec-app}

\section{Proof of Theorem \ref{thPCL} and Corollary \ref{clink}}\label{proof_thPCL}

Our proof assumes basic knowledge about proper losses (see for example
\cite{rwCB}). From \citep[Theorem 1, Corollary 3]{rwCB} and \cite{samAP}, $\properloss$ being twice
differentiable and proper, its conditional Bayes risk $\cbr$ and
partial losses $\properloss_{1}$ and  $\properloss_{-1}$ are related by:
\begin{eqnarray}
-\cbr''(c) = \frac{\properloss'_{-1}(c)}{c} = -
\frac{\properloss'_{1}(c)}{1-c} \:\:, \forall c \in (0,1).
\end{eqnarray}
The weight function \citep[Theorem 1]{rwCB} being also $w = -\cbr''$,
we get from the integral representation of partial losses
\citep[eq. (5)]{rwCB},
\begin{eqnarray}
\properloss_{1}(c) & = & - \int_c^1 (1-u) \cbr''(u) \mathrm{d}u,
\end{eqnarray}
from which we derive by integrating by parts and then using the
Legendre conjugate of $-\cbr$,
\begin{eqnarray}
\properloss_{1}(c) + \cbr(1) & = & - \left[ (1-u) \cbr'(u) \right]_{c}^1 -
\int_c^1 \cbr'(u) \mathrm{d}u + \cbr(1)\nonumber\\
 & = & (1-c) \cbr'(c) + \cbr(c) - \cbr(1) + \cbr(1)\label{eqBreg1}\\
 & = & -(-\cbr')(c) + c \cdot (-\cbr')(c) - (-\cbr)(c)\nonumber\\
 & = & -(-\cbr')(c) + (-\cbr)^\star((-\cbr)'(c)).
\end{eqnarray}
Now, suppose that the way a real-valued prediction $v$ is fit in the
loss is through a general inverse link $\psi^{-1} : \mathbb{R}
\rightarrow (0,1)$. Let
\begin{eqnarray}
v_{\ell, \psi} & \defeq & (-\cbr') \circ \psi^{-1} (v).
\end{eqnarray}
Since $(-\cbr)'^{-1}(v_{\ell, \psi}) = \psi^{-1} (v)$, the proper
composite loss $\ell$ with link $\psi$ on prediction $v$ is the same
as the proper
composite loss $\ell$ with link $(-\cbr)'$ on prediction $v_{\ell,
  \psi}$. This last loss is in fact using its canonical link and so is
proper canonical \citep[Section 6.1]{rwCB}, \citep{bssLF}. Letting in this case $c \defeq
(-\cbr)'^{-1}(v_{\ell, \psi})$, we get that the partial loss satisfies
\begin{eqnarray}
\properloss_{1}(c) & = & -v_{\ell, \psi} + (-\cbr)^\star(v_{\ell, \psi}) - \cbr(1).\label{eql1}
\end{eqnarray}
Notice the constant appearing on the right hand side. Notice also that
if we see \eqref{eqBreg1} as a Bregman divergence, $\properloss_{1}(c)
= (-\cbr)(1) - (-\cbr)(c) -  ((1-c) (-\cbr')(c) = D_{-\cbr}(1\|c)$,
then the canonical link is the function that defines uniquely the dual
affine coordinate system of the divergence \citep{anMO} (see also
\citep[Appendix B]{rwCB}). 

We can repeat the derivations for the partial loss $\properloss_{-1}$,
which yields \citep[eq. (5)]{rwCB}:
\begin{eqnarray}
\properloss_{-1}(c) + \cbr(0)& = & - \int_0^c u \cbr''(u) \mathrm{d}u + \cbr(0)
\nonumber\\
& = & - \left[ u\cbr'(u) \right]_{0}^c +
\int_0^c \cbr'(u) \mathrm{d}u \nonumber\\
 & = & -c \cbr'(c) + \cbr(c) - \cbr(0) + \cbr(0)\label{eqBreg2}\\
 & = & c \cdot (-\cbr')(c) - (-\cbr)(c)\nonumber\\
 & = & (-\cbr)^\star((-\cbr)'(c)),
\end{eqnarray}
and using the canonical link, we get this time
\begin{eqnarray}
\properloss_{-1}(c) & = & (-\cbr)^\star(v_{\ell, \psi}) - \cbr(0).\label{eqlm1}
\end{eqnarray}
We get from \eqref{eql1} and \eqref{eqlm1} the canonical proper
composite loss
\begin{eqnarray}
\properloss(y, v) & = & (-\cbr)^\star(v_{\ell, \psi}) - \frac{y+1}{2}\cdot v_{\ell, \psi} -
\frac{1}{2}\cdot \left((1-y)\cdot \cbr(0) + (1+y)\cdot 
\cbr(1)\right).
\end{eqnarray}
Note that for the optimisation of
$\properloss(y, v)$ for $v$, we could discount the right-hand side parenthesis, which acts just
like a constant with respect to $v$. Using Fenchel-Young inequality yields the
non-negativity of $\properloss(y, v)$ as it brings $(-\cbr)^\star(v_{\ell, \psi}) -
((y+1)/2)\cdot v_{\ell, \psi} \geq \cbr((y+1)/2)$ and so
\begin{eqnarray}
\properloss(y, v) & \geq & \cbr\left(\frac{1+y}{2}\right) -
\frac{1}{2}\cdot \left((1-y)\cdot \cbr(0) + (1+y)\cdot 
\cbr(1)\right)\nonumber\\
 & & = \cbr\left(\frac{1}{2}\cdot(1-y)\cdot 0 + \frac{1}{2}\cdot(1+y)\cdot 1\right) -
\frac{1}{2}\cdot \left((1-y)\cdot \cbr(0) + (1+y)\cdot 
\cbr(1)\right)\nonumber\\
 & \geq & 0, \forall y\in
\{-1,1\}, \forall v \in \mathbb{R},\label{eq00}
\end{eqnarray}
from Jensen's inequality (the conditional Bayes risk $\cbr$ is always
concave \citep{rwCB}). Now, if we consider the alternative use of Fenchel-Young inequality,
\begin{eqnarray}
(-\cbr)^\star(v_{\ell, \psi}) - \frac{1}{2}\cdot v_{\ell, \psi} & \geq & \cbr \left(\frac{1}{2}\right),
\end{eqnarray}
then if we let
\begin{eqnarray}
\Delta(y) & \defeq & \cbr \left(\frac{1}{2}\right) -
\frac{1}{2}\cdot \left((1-y)\cdot \cbr(0) + (1+y)\cdot 
\cbr(1)\right),
\end{eqnarray}
then we get
\begin{eqnarray}
\properloss(y, v) & \geq &\Delta(y)- \frac{y}{2}\cdot v_{\ell, \psi} , \forall y\in
\{-1,1\}, \forall v \in \mathbb{R}. \label{eq01}
\end{eqnarray}
It follows from \eqref{eq00} and \eqref{eq01},
\begin{eqnarray}
\properloss(y, v) & \geq & \max\left\{0, \Delta(y)- \frac{y}{2}\cdot
  v_{\ell, \psi}\right\}, \forall y\in
\{-1,1\}, \forall v \in \mathbb{R},
\end{eqnarray}
and we get, $\forall h \in \mathbb{R}^{\mathcal{X}}, a \in
\mathcal{X}^{\mathcal{X}}$,
\begin{eqnarray}
\lefteqn{\E_{(\X, \Y) \sim D} [\properloss(y, h\circ
  a(\X))]}\nonumber\\
 & \geq & \E_{(\X, \Y) \sim D} \left[ \max\left\{0, \Delta(\Y)- \frac{\Y}{2}\cdot
  (h\circ
  a)_{\ell, \psi}(\X)\right\}\right] \nonumber\\
 & \geq & \max\left\{0, \E_{(\X, \Y) \sim D} \left[ \Delta(\Y)- \frac{\Y}{2}\cdot
  (h\circ
  a(\X))_{\ell, \psi} \right]\right\}\nonumber\\
 & & = \max\left\{0, \cbr \left(\frac{1}{2}\right) - \frac{1}{2} \cdot \E_{(\X, \Y) \sim D} \left[ \Y \cdot
  (h\circ
  a(\X))_{\ell, \psi} + (1-\Y)\cdot \cbr(0) + (1+\Y)\cdot 
\cbr(1)\right]\right\} \nonumber\\
 & = & \max\left\{0, \cbr \left(\frac{1}{2}\right) - \frac{1}{2} \cdot
   \left(
\begin{array}{c}
\E_{\X \sim P} \left[ \pi \cdot ((h\circ
  a(\X))_{\ell, \psi} + 2\cbr(1)) \right] \\
- \E_{\X \sim N} \left[ (1-\pi) \cdot ((h\circ
  a(\X))_{\ell, \psi} -2\cbr(0)) \right]
\end{array}\right)\right\}\nonumber\\
 & = & \max\left\{0, \cbr \left(\frac{1}{2}\right) - \frac{1}{2} \cdot
   \left(\phi(P, (h\circ
  a)_{\ell, \psi}, \pi, 2\cbr(1)) - \phi(N, (h\circ
  a)_{\ell, \psi}, 1-\pi, -2\cbr(0))\right)\right\},
\end{eqnarray}
with
\begin{eqnarray}
\phi(Q, f, b, c) & \defeq & \int_{\mathcal{X}} b\cdot(f(\ve{x}) + c)\mathrm{d}Q(\ve{x}),
\end{eqnarray}
and we recall 
\begin{eqnarray}
(h\circ
  a)_{\ell, \psi} & = & (-\cbr') \circ \psi^{-1} \circ h\circ
  a.
\end{eqnarray}
Hence,
\begin{eqnarray}
\lefteqn{\min_{h\in \mathcal{H}} \E_{(\X, \Y) \sim D} [ \max_{a \in \mathcal{A}} \properloss(\Y, h\circ
  a(\X))]}\nonumber\\
 & \geq & \min_{h\in \mathcal{H}} \max_{a \in \mathcal{A}} \E_{(\X, \Y) \sim D} [ \properloss(\Y, h\circ
  a(\X))] \label{ll1}\\
 & \geq & \min_{h\in \mathcal{H}} \max_{a \in \mathcal{A}} \max\left\{0, \cbr \left(\frac{1}{2}\right) - \frac{1}{2} \cdot
  \left(\phi(P, (h\circ
  a)_{\ell, \psi}, \pi, 2\cbr(1)) - \phi(N, (h\circ
  a)_{\ell, \psi}, 1-\pi, -2\cbr(0))\right)\right\}\nonumber\\
 & \geq & \max_{a \in \mathcal{A}} \min_{h\in \mathcal{H}} \max\left\{0, \cbr \left(\frac{1}{2}\right) - \frac{1}{2} \cdot
  \left(\phi(P, (h\circ
  a)_{\ell, \psi}, \pi, 2\cbr(1)) - \phi(N, (h\circ
  a)_{\ell, \psi}, 1-\pi, -2\cbr(0))\right)\right\}\nonumber\\
 & & = \max_{a \in \mathcal{A}} \max\left\{0, \min_{h\in \mathcal{H}} \left(\cbr \left(\frac{1}{2}\right) - \frac{1}{2} \cdot
  \left(\phi(P, (h\circ
  a)_{\ell, \psi}, \pi, 2\cbr(1)) - \phi(N, (h\circ
  a)_{\ell, \psi}, 1-\pi, -2\cbr(0))\right)\right)\right\}\nonumber\\
 & = & \max_{a \in \mathcal{A}} \max\left\{0, \cbr \left(\frac{1}{2}\right) - \frac{1}{2} \cdot
  \max_{h\in \mathcal{H}} \left(\phi(P, (h\circ
  a)_{\ell, \psi}, \pi, 2\cbr(1)) - \phi(N, (h\circ
  a)_{\ell, \psi}, 1-\pi, -2\cbr(0))\right)\right\}\nonumber\\
 & = & \max_{a \in \mathcal{A}} \left(\cbr \left(\frac{1}{2}\right) - \frac{1}{2} \cdot
  \max_{h\in \mathcal{H}} \left(\phi(P, (h\circ
  a)_{\ell, \psi}, \pi, 2\cbr(1)) - \phi(N, (h\circ
  a)_{\ell, \psi}, 1-\pi, -2\cbr(0))\right)\right)_+\nonumber\\
 & = & \left(\cbr \left(\frac{1}{2}\right) - \frac{1}{2} \cdot
  \min_{a \in \mathcal{A}} \max_{h\in \mathcal{H}} \left(\phi(P, (h\circ
  a)_{\ell, \psi}, \pi, 2\cbr(1)) - \phi(N, (h\circ
  a)_{\ell, \psi}, 1-\pi, -2\cbr(0))\right)\right)_+\nonumber\\
 & = & \left(\cbr \left(\frac{1}{2}\right) - \frac{1}{2} \cdot
  \min_{a \in \mathcal{A}} \upgamma^{g}_{\mathcal{H}, a} (P, N, \pi,
  2\cbr(1), 2\cbr(0))\right)_+\nonumber\\
 & = & \left(\properloss^\circ - \frac{1}{2} \cdot
  \min_{a \in \mathcal{A}} \upgamma^{g}_{\mathcal{H}, a} (P, N, \pi,
  2\cbr(1), 2\cbr(0))\right)_+\nonumber\\
 & = & \left(\properloss^\circ - \frac{1}{2} \cdot
  \min_{a \in \mathcal{A}} \beta_a\right)_+,
\end{eqnarray}
as claimed for the statement of Theorem \ref{thPCL} (we have let $g \defeq (-\cbr') \circ \psi^{-1}$). Hence, if, for some $\epsilon \in [0,1]$,
\begin{eqnarray}
\exists a \in \mathcal{A} : \upgamma^{g}_{\mathcal{H}, a} (P, N, \pi,
  2\cbr(1), 2\cbr(0)) & \leq & 2\epsilon \cdot \properloss^\circ,\label{constgamma}
\end{eqnarray}
then
\begin{eqnarray}
\min_{h\in \mathcal{H}} \E_{(\X, \Y) \sim D} [ \max_{a \in \mathcal{A}} \properloss(\Y, h\circ
  a(\X))] & \geq & \left(\properloss^\circ - \epsilon \cdot
    \properloss^\circ\right)_+\nonumber\\
 & & = (1-\epsilon) \cdot \properloss^\circ,
\end{eqnarray}
which ends the proof of Corollary \ref{clink} if $\ell$ is proper
composite with link $\psi$. If it is proper canonical, then $ (-\cbr')
\circ \psi^{-1} = \mathrm{Id}$ and so $\upgamma^{g}_{\mathcal{H}, a} =
\upgamma_{\mathcal{H}, a}$ in \eqref{constgamma}.

\section{Proof sketch of Corollary \ref{corMMD}}\label{proof_corMMD}

Recall that $\beta_a = \upgamma_{\mathcal{H}, a} \hspace{-0.1cm}\left(P, N, \frac{1}{2},
  2\cbr(1), 2\cbr(0)\right)$. We prove the following, more general result which does not assume $\pi
= 1/2$ nor $\upgammaellpi = 0$.
\begin{corollary}\label{corMMD2}
Suppose $\ell$ is canonical proper and let $\mathcal{H}$ denote the unit ball of a
reproducing kernel
Hilbert space (RKHS) of functions with reproducing kernel $\kappa$.
Denote 
\begin{eqnarray}
\mu_{a, Q} & \defeq & \int_{\mathcal{X}} \kappa( a(\ve{x}), .)
    \mathrm{d}Q(\ve{x})
\end{eqnarray}
the adversarial mean embedding of $a$ on $Q$. Then 
\begin{eqnarray*}
\lefteqn{2 \cdot \upgamma_{\mathcal{H}, a} (P, N, \pi,
  2\cbr(1), 2\cbr(0))}\nonumber\\
 & = &  \upgammaellpi + \|\pi \cdot \mu_{a, P} - (1-\pi) \cdot \mu_{a, N}\|_{\mathcal{H}}  .
\end{eqnarray*}
\end{corollary}
\begin{proof}
It comes from the
reproducing property of $\mathcal{H}$,
\begin{eqnarray}
\lefteqn{2 \cdot \upgamma_{\mathcal{H}, a} (P, N, \pi,
  2\cbr(1), 2\cbr(0))}\nonumber\\
 & = & \upgammaellpi + \max_{h \in \mathcal{H}}
  \left\{\pi \cdot \int_{\mathcal{X}} h\circ a(\ve{x})
    \mathrm{d}P(\ve{x})-(1-\pi) \cdot \int_{\mathcal{X}} h\circ a(\ve{x})
    \mathrm{d}N(\ve{x})\right\}\nonumber\\
 & = &  \upgammaellpi + \max_{h \in \mathcal{H}}
  \left\{\pi \cdot \left\langle h, \int_{\mathcal{X}} \kappa( a(\ve{x}), .)
    \mathrm{d}P(\ve{x})\right\rangle_{\mathcal{H}} -(1-\pi) \cdot \left\langle h, \int_{\mathcal{X}} \kappa( a(\ve{x}), .)
    \mathrm{d}N(\ve{x})\right\rangle_{\mathcal{H}}\right\}\nonumber\\
 & = &  \upgammaellpi + \max_{h \in \mathcal{H}}
  \left\{\left\langle h, \pi \cdot \mu_{a, P} - (1-\pi) \cdot \mu_{a, N}\right\rangle_{\mathcal{H}} \right\}\nonumber\\
 & = &  \upgammaellpi + \|\pi \cdot \mu_{a, P} - (1-\pi) \cdot \mu_{a, N}\|_{\mathcal{H}} ,
\end{eqnarray}
as claimed, where the last equality holds for the unit ball.
\end{proof}

\section{Proof of Theorem \ref{thOTA}}\label{proof_thOTA}

We first show a Lemma giving some additional properties on our
definition os Lipschitzness.
\begin{lemma}\label{lemCONSTC}
Suppose 
$\mathcal{H}$ is $(u, v, K)$-Lipschitz. If $c$ is symmetric, then 
$\{u\circ h - v\circ h\}_{h\in \mathcal{H}}$ is $2K$-Lipschitz. If $c$ satisfies the triangle
inequality, then $u-v$ is bounded. If $c$
satisfies the identity of indiscernibles, then $u\leq v$.
\end{lemma}
\begin{proof}
If $c$ is symmetric, then we just add two instances of
\eqref{defGENLIP} with $\ve{x}$ and $\ve{y}$ permuted, reorganize and get:
\begin{eqnarray*}
u\circ h(\ve{x}) - v\circ h(\ve{y}) + u\circ h(\ve{y}) - v\circ
h(\ve{x}) & \leq & K \cdot (c(\ve{x},\ve{y})+c(\ve{y},\ve{x})), \forall h \in \mathcal{H}, \forall \ve{x},
\ve{y} \in \mathcal{X}.\\
\Leftrightarrow (u\circ h-v\circ h)(\ve{x}) - (u\circ h-v\circ
h)(\ve{y}) & \leq & 2K c(\ve{x},\ve{y}) , \forall h \in \mathcal{H}, \forall \ve{x},
\ve{y} \in \mathcal{X}.
\end{eqnarray*}
and we get the statement of the Lemma. If $c$ satisfies the triangle
inequality, then we add again two instances of
\eqref{defGENLIP} but this time as follows:
\begin{eqnarray*}
u\circ h(\ve{x}) - v\circ h(\ve{y}) + u\circ h(\ve{y}) - v\circ
h(\ve{z}) & \leq & K \cdot (c(\ve{x},\ve{y})+c(\ve{y},\ve{z})), \forall h \in \mathcal{H}, \forall \ve{x},
\ve{y} ,
\ve{z} \in \mathcal{X}.\\
\Leftrightarrow u\circ h(\ve{x}) - v\circ h(\ve{z}) + \Delta(\ve{y})& \leq & K c(\ve{x},\ve{z}) , \forall h \in \mathcal{H}, \forall \ve{x},
\ve{y} ,
\ve{z} \in \mathcal{X},
\end{eqnarray*}
where $\Delta(\ve{y}) \defeq u\circ
h(\ve{y}) - v\circ h(\ve{y}) $. If $c$ is finite for at least one
couple $(\ve{x}, \ve{z})$, then we cannot have $u-v$
unbounded in $\cup_h \mathrm{Im}(h)$. Finally, if $c$ satisfies the identity of indiscernibles,
then picking $\ve{x} = \ve{y}$ in  \eqref{defGENLIP} yields $u\circ
h(\ve{x}) - v\circ h(\ve{x}) \leq 0, \forall h \in \mathcal{H}, \forall \ve{x} \in \mathcal{X}$
and so $(u-v)(\cup_h \mathrm{Im}(h)) \cap \mathbb{R}_+ \subseteq
\{0\}$, which, disregarding the images in $\mathcal{H}$ for simplicity, yields $u\leq v$.
\end{proof}
We now prove Theorem{thOTA}. In fact, we shall prove the following
more general Theorem.
\begin{theorem}\label{thOTA2}
Fix any $\epsilon > 0$ and proper loss $\ell$ with link
$\psi$. Suppose $\exists c : \mathcal{X} \times \mathcal{X}
\rightarrow \mathbb{R}$ such that:
\begin{enumerate}
\item [(1)] $\mathcal{H}$ is $(\pi \cdot g, (1-\pi) \cdot g,
  K)$-Lipschitz with respect to $c$, where $g$ is defined in \eqref{defgg};
\item [(2)] $\mathcal{A}$ is $\delta$-Monge efficient for cost $c$ on marginals
$P, N$ for
\begin{eqnarray}
\delta & \leq & 2\cdot \frac{2 \epsilon \properloss^\circ - \upgammaellpi}{K}.\label{bdelta1}
\end{eqnarray}
\end{enumerate}
Then $\mathcal{H}$ is
$\epsilon$-defeated by $\mathcal{A}$ on $\properloss$.
\end{theorem}
\begin{proof}
We have
for all $a \in \mathcal{A}$,
\begin{eqnarray}
\lefteqn{\max_{h \in \mathcal{H}} \left(\phi(P,  h\circ
  a, \pi, 2 \cbr(1)) - \phi(N,  h\circ
  a, 1-\pi, - 2 \cbr(0))\right)}\nonumber\\
 & = & \upgammaellpi + \frac{1}{2} \cdot \max_{h \in \mathcal{H}}
\left(\int_{\mathcal{X}}  \pi\cdot g \circ h\circ
  a (\ve{x}) \mathrm{d}P(\ve{x}) -\int_{\mathcal{X}}  (1-\pi)\cdot g \circ h\circ
  a (\ve{x}') \mathrm{d}N(\ve{x}')\right),\label{bound11}
\end{eqnarray}
where we recall $g \defeq (-\cbr') \circ \psi^{-1}$. Let us denote for short
\begin{eqnarray}
\Delta & \defeq & \max_{h \in \mathcal{H}}
\left(\int_{\mathcal{X}}  \pi\cdot g \circ h\circ
  a (\ve{x}) \mathrm{d}P(\ve{x}) -\int_{\mathcal{X}}  (1-\pi)\cdot g \circ h\circ
  a (\ve{x}') \mathrm{d}N(\ve{x}')\right).\label{defDELTA}
\end{eqnarray}
$\mathcal{H}$ being 
$(\pi \cdot g, (1-\pi) \cdot g,
  K)$-Lipschitz for cost $c$, since
  $$\mathcal{H}\subseteq \{h \in \mathbb{R}^{\mathcal{X}} : \pi g\circ h\circ
  a (\ve{x}) - (1-\pi) g\circ h\circ
  a (\ve{x}') \leq K c(a (\ve{x}), a (\ve{x} ')), \forall \ve{x},
  \ve{x}' \in \mathcal{X}\},$$ it comes after letting for short
  $\Psi \defeq \pi g\circ h\circ
  a, \chi  \defeq (1-\pi) g\circ h\circ
  a$,
\begin{eqnarray}
\Delta & \leq & \max_{\Psi (\ve{x}) - \chi (\ve{x}') \leq K c(a (\ve{x}), a (\ve{x} '))}
\left(\int_{\mathcal{X}} \Psi  (\ve{x}) \mathrm{d}P(\ve{x}) -\int_{\mathcal{X}} \chi (\ve{x}) \mathrm{d}N(\ve{x})\right)\nonumber\\
 & \leq & K \cdot \inf_{\muup \in \Pi(P, N)} \int
 c(a (\ve{x}), a (\ve{x}')) \mathrm{d}\muup(\ve{x}, \ve{x}').
\end{eqnarray}
See for example \citep[Section 4]{vOT} for the last inequality. Now, if some adversary $a \in \mathcal{A}$ is $\delta$-Monge efficient
for cost $c$, then
\begin{eqnarray}
K \cdot \inf_{\muup \in \Pi(P, N)} \int
 c(a (\ve{x}), a (\ve{x}')) \mathrm{d}\muup(\ve{x}, \ve{x}') & \leq & K\delta.
\end{eqnarray}
From Theorem \ref{thPCL}, if we want $\mathcal{H}$ to be
$\epsilon$-defeated by $\mathcal{A}$, then it is sufficient from \eqref{bound11} that $a$
satisfies
\begin{eqnarray}
\upgammaellpi + \frac{1}{2} \cdot K\delta & \leq & 2\epsilon \properloss^\circ,
\end{eqnarray}
resulting in 
\begin{eqnarray}
\delta & \leq & 2\cdot \frac{2 \epsilon \properloss^\circ - \upgammaellpi}{K},
\end{eqnarray}
as
claimed.
\end{proof}

\textbf{Remark 1} note that unless
$\pi = 1/2$, $c$ cannot be a distance in the general case fot Theorem \ref{thOTA}: indeed, the
identity of indiscernibles and Lemma
\ref{lemCONSTC} enforce $(1-2\pi) \cdot g\geq 0$ and so $g$ cannot take
both signs, which is impossible whenever $\ell$ is canonical
proper as $g = \mathrm{Id}$ in this case. We take it as a potential difficulty for the adversary which, we recall,
cannot act on $\pi$.\\ 

\textbf{Remark 2} In the light of
recent results \citep{cbgduPN,cksnDR,mkkySN}, there is an interesting corollary to Theorem \ref{thOTA} when $\pi =
1/2$ using a form of Lipschitz continuity of the \textit{link} of the
loss .
\begin{corollary}\label{corMonge2}
Suppose loss $\ell$ is proper with link $\psi$ and
furthermore its canonical link satisfies,
some $K_\ell > 0$:
\begin{eqnarray*}
(\cbr)'(y) - (\cbr)'(y') \hspace{-0.2cm} & \leq & \hspace{-0.2cm} K_\ell \cdot |\psi(y) - \psi(y')|, \forall y,
y' \in [0,1].
\end{eqnarray*}
Suppose furthermore that (i) $\pi = 1/2$, (ii) $\mathcal{H}$ is $K_h$-Lipschitz with respect to
some non-negative $c$ and (iii) $\mathcal{A}$ 
is $\delta$-Monge efficient for cost $c$ on marginals $P, N$ for
\begin{eqnarray}
\delta & \leq & \frac{4 \epsilon \properloss^\circ -
  2\upgammaellpi}{K_\ell K_h}.\label{bdelta2}
\end{eqnarray}
Then $\mathcal{H}$ is
$\epsilon$-defeated by $\mathcal{A}$ on $\properloss$.
\end{corollary}
\begin{proof}
The domination condition on links,
\begin{eqnarray}
(\cbr)'(y) - (\cbr)'(y') & \leq & K_\ell \cdot |\psi(y) - \psi(y')|, \forall y,
y' \in [0,1],
\end{eqnarray}
implies $g$ is Lipschitz and letting $y \defeq \psi^{-1} \circ
h \circ a (\ve{x})$, $y' \defeq \psi^{-1} \circ
h \circ a (\ve{x}')$, we obtain equivalently $g \circ h \circ a
(\ve{x}) - g \circ h \circ a
(\ve{x}) \leq  K_\ell \cdot |h \circ a (\ve{x}) - h \circ a
(\ve{x}')|, \forall \ve{x}, \ve{x}' \in \mathcal{X}$. But $\mathcal{H}$ is $K_h$-Lipschitz with respect to
some non-negative $c$, so we have $|h \circ a (\ve{x}) - h \circ a
(\ve{x}')| \leq K_h c(a (\ve{x}), a (\ve{x}'))$, and so bringing these
two inequalities together, we have 
from the proof of Theorem \ref{thOTA} that $\Delta$ now satisfies
\begin{eqnarray}
\Delta & \leq & \frac{K_\ell K_h}{2}\cdot \inf_{\muup \in \Pi(P, N)} \int
 c(a (\ve{x}), a (\ve{x}')) \mathrm{d}\muup(\ve{x}, \ve{x}'),
\end{eqnarray}
so to be $\epsilon$-defeated by $\mathcal{A}$ on $\properloss$, we now
want that $a$
satisfies
\begin{eqnarray}
\upgammaellpi + \frac{K_\ell K_h}{2} \cdot \delta & \leq & 2\epsilon \properloss^\circ,
\end{eqnarray}
resulting in the statement of the Corollary.
\end{proof}

\section{Proof of Theorem \ref{thmMEF}}\label{proof_thmMEF}

Denote $a^ J \defeq a \circ a \circ ... \circ a$ ($J$ times). We have by definition
\begin{eqnarray}
C_\Phi(a^J,P,N) & \defeq & \inf_{\muup \in \Pi(P, N)}
   \int_{\mathcal{X}} \|\Phi \circ a^J (\ve{x}) - \Phi \circ a^J
                           (\ve{x}')\|_{\mathcal{H}}
                           \mathrm{d}\muup(\ve{x}, \ve{x}') \nonumber\\
& = & \inf_{\muup \in \Pi(P, N)}
   \int_{\mathcal{X}} \|\Phi \circ a \circ a^{J-1} (\ve{x}) - \Phi \circ  a \circ a^{J-1} 
                           (\ve{x}')\|_{\mathcal{H}}
                           \mathrm{d}\muup(\ve{x}, \ve{x}') \label{eq123}\\
& \leq & (1-\eta) \cdot \inf_{\muup \in \Pi(P, N)}
   \int_{\mathcal{X}} \|\Phi \circ a^{J-1} (\ve{x}) - \Phi \circ a^{J-1} 
                           (\ve{x}')\|_{\mathcal{H}}
                           \mathrm{d}\muup(\ve{x}, \ve{x}') \nonumber\\
& \vdots & \nonumber\\
& \leq & (1-\eta)^{J} \cdot \inf_{\muup \in \Pi(P, N)}
   \int_{\mathcal{X}} \|\Phi (\ve{x}) - \Phi 
                           (\ve{x}')\|_{\mathcal{H}}
                           \mathrm{d}\muup(\ve{x}, \ve{x}') \nonumber\\
& & = (1-\eta)^{J} \cdot W_1^\Phi,\label{mefEQ}
\end{eqnarray}
where we have used the assumption that $a$ is $\eta$-contractive and
the definition of $W_1^\Phi$. There remains to bound the last line by
$\delta$ and solve for $J$ to get the statement of the Theorem. We can
also stop at \eqref{eq123} to conclude that $\mathcal{A}$ is
$\delta$-Monge efficient for $\delta = (1-\eta)\cdot
W_1^\Phi$. The number of iterations for $\mathcal{A}^J$ to be
$\delta$-Monge efficient is obtained from \eqref{mefEQ} as
\begin{eqnarray}
J & \geq & \frac{1}{\log \left(\frac{1}{1-\eta}\right)} \cdot  \log \frac{W_1^\Phi}{\delta},
\end{eqnarray}
which gives the statement of the Theorem once we remark that
$\log(1/(1-\eta)) \leq 1/\eta$.

\section{Proof of Lemma \ref{lemMIX1}}\label{proof_lemMIX1}
The proof follows from the observation that for any $\ve{x}, \ve{x}'$
in $\mathcal{S}$,
\begin{eqnarray}
\|a(\ve{x}) - a(\ve{x}')\| & = & \lambda \| \ve{x} -
\ve{x}' \|,
\end{eqnarray}
where $\|.\|$ is the metric of $\mathcal{X}$.
Thus, letting $a$ denote a mixup to $\ve{x}^*$ adversary for some
$\lambda \in [0,1]$, we have
$C(a,P,N) = \lambda \cdot W_1(\mathrm{d}P, \mathrm{d}N)$, where $W_1(\mathrm{d}P, \mathrm{d}N)$ denotes the
Wasserstein distance of order 1 between the class marginals. $\delta > 0$ being fixed, all mixups
to $\ve{x}^*$ adversaries in $\mathcal{A}$ that are also $\delta$-Monge efficient are those for
which:
\begin{eqnarray}
\lambda & \leq & \frac{\delta}{W_1(\mathrm{d}P, \mathrm{d}N)},
\end{eqnarray}
and we get the statement of the Lemma.

% !TEX root=../nips18-adversarial-sm-1.tex

\section{Experiments}
\label{sec:toy:sup}

Figure \ref{fig:sm:toy} includes detailed plots for the $\alpha=0.5$ case of the numerical toy example.

\begin{figure*}[t]
\begin{center}
\subcaptionbox{Clean class distributions.}{
    \includegraphics[height=0.275\textwidth,page=1]{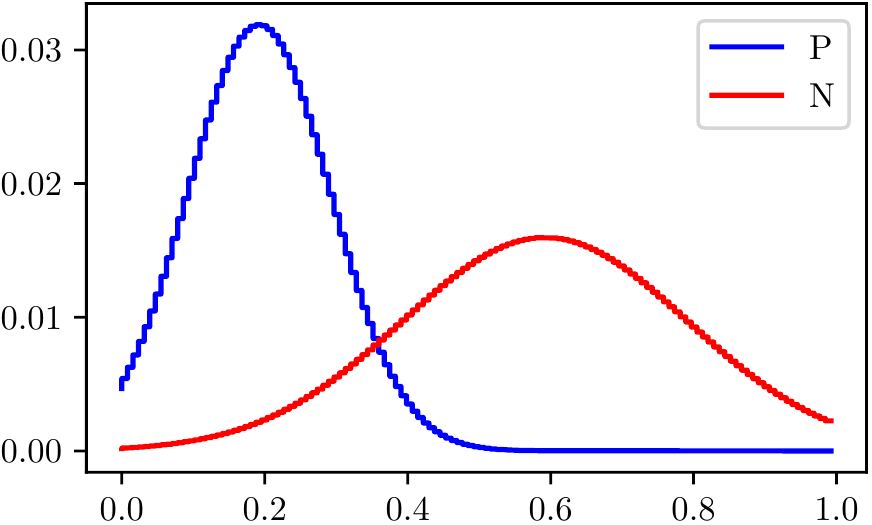}
}
~
\subcaptionbox{Adversarial class distributions.}{
    \includegraphics[height=0.275\textwidth,page=3]{\toyfn}
}
\\ ~ \\ ~ \\
\subcaptionbox{Transport cost.}{
    \includegraphics[height=0.4\textwidth,page=6]{\toyfn}
}
~
\subcaptionbox{Optimal transport plan.}{
    \includegraphics[height=0.4\textwidth,page=8]{\toyfn}
}
\end{center}
\caption{
\label{fig:sm:toy}
Visualising the toy example for the case $\alpha=0.5$. Clockwise from top left: (a) the clean class conditional distributions, (b) the class distributions mapped by the adversary $a$, (c) the transport cost $c$ under the adversarial mapping $a$, (d) the corresponding optimal transport $\mu$.
}
\end{figure*}

\clearpage
\newpage
\begin{center}
\includegraphics[width=1.01\textwidth]{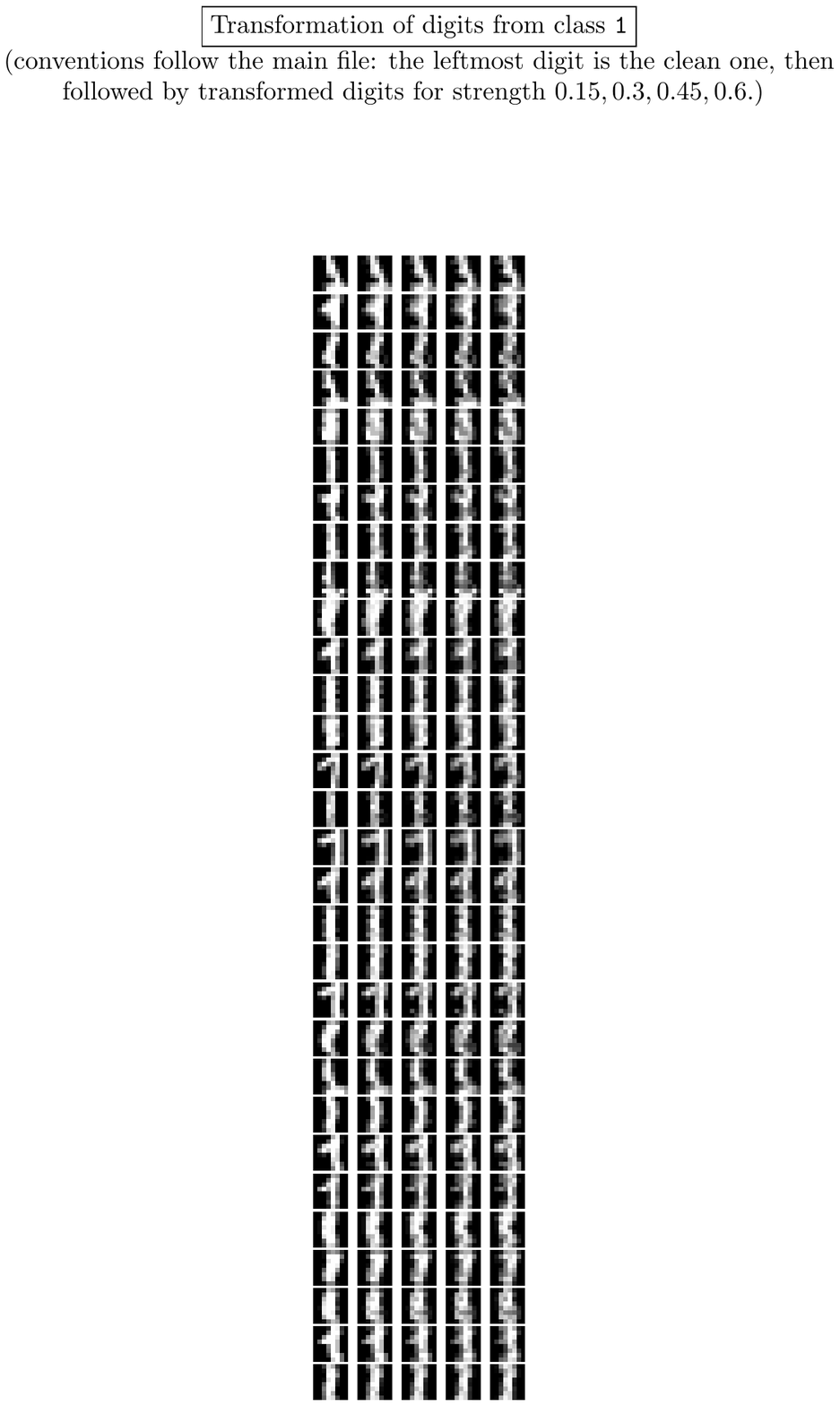}
\end{center}
\newpage
\begin{center}
\includegraphics[width=1.01\textwidth]{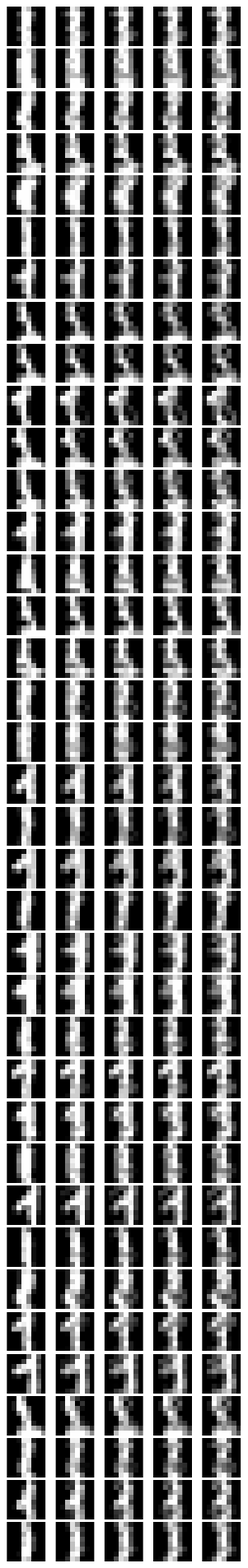}
\end{center}
\newpage
\begin{center}
\includegraphics[width=1.01\textwidth]{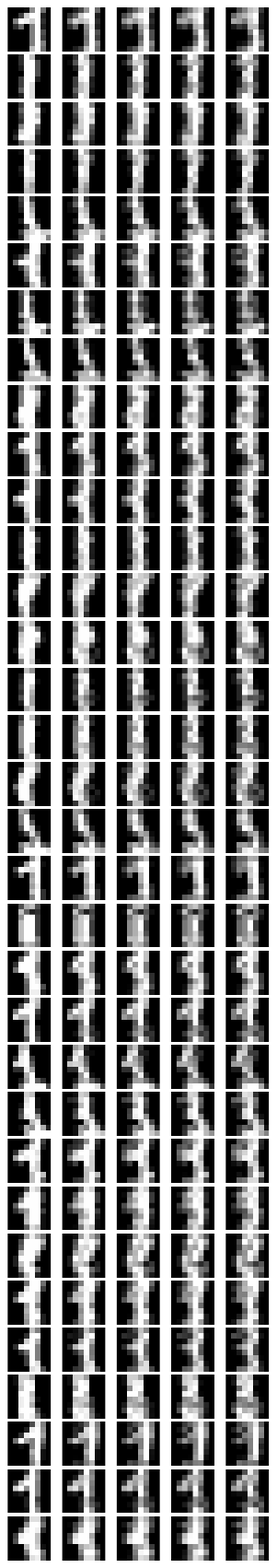}
\end{center}
\newpage
\begin{center}
\includegraphics[width=1.01\textwidth]{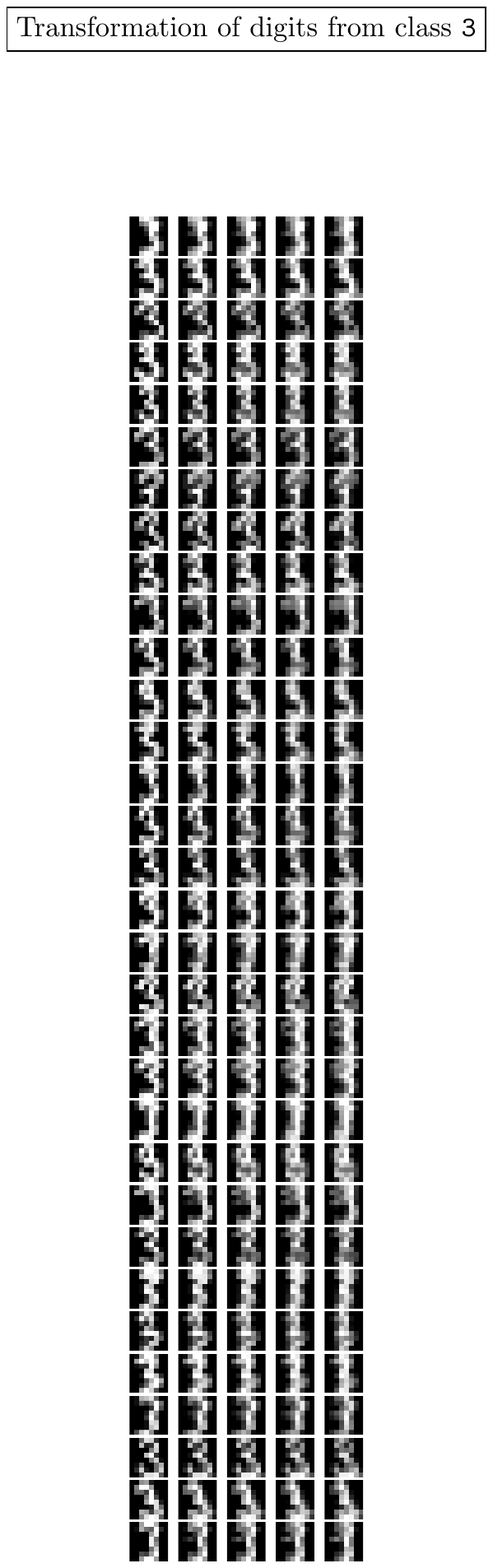}
\end{center}
\newpage
\begin{center}
\includegraphics[width=1.01\textwidth]{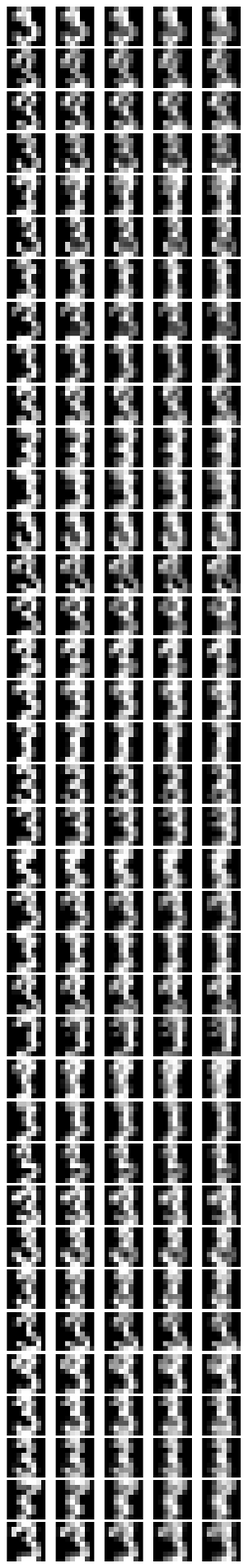}
\end{center}

\end{document}